\documentclass[table]{ai2style/ai2}

\usepackage{microtype}
\usepackage{hyperref}
\usepackage{url}
\usepackage{booktabs, tabularx} 
\usepackage{graphicx}
\usepackage{lineno}
\usepackage{enumitem}
\usepackage{listings} 
\usepackage{svg}
\usepackage{float}

\usepackage{tikz}
\usetikzlibrary{decorations.pathreplacing}
\usetikzlibrary{calc}

\definecolor{ darkblue}{rgb}{0, 0, 0.5}
\hypersetup{colorlinks=true, citecolor=darkblue, linkcolor=darkblue, urlcolor=darkblue}

\usepackage{amssymb}
\usepackage{multirow}
\usepackage{bigdelim}
\usepackage{longtable}
\usepackage{tabularray}
\usepackage{wrapfig}
\usepackage[most]{tcolorbox}
\usepackage{url}
\usepackage{xspace}
\usepackage{svg}
\usepackage[absolute]{textpos} 

\usepackage{array}
\usepackage{ragged2e} 

\newcolumntype{L}[1]{>{\RaggedRight\arraybackslash}p{#1}}
\newcolumntype{C}[1]{>{\Centering\arraybackslash}p{#1}}

\usepackage{fdsymbol}   

\usepackage[utf8]{inputenc} 
\usepackage[T1]{fontenc}    
\usepackage{url}            
\usepackage{booktabs}       
\usepackage{amsfonts}       
\usepackage{nicefrac}       
\usepackage{microtype}      
\usepackage[table,xcdraw]{xcolor}         
\usepackage{amsmath}
\usepackage[most]{tcolorbox}
\usepackage{csquotes}

\usepackage{siunitx}
\usepackage{graphicx}
\usepackage{arydshln}
\usepackage{wrapfig}
\usepackage{enumitem}
\usepackage{soul} 

\usepackage{multirow}
\usepackage{xspace}
\usepackage{adjustbox}
\usepackage{pifont}
\usepackage{caption}
\usepackage{makecell}
\usepackage{subcaption}
\usepackage{bold-extra}
\usepackage{url}

\usepackage{amsthm}
\usepackage{thmtools}
\usepackage{thm-restate} 

\usepackage{pgf-pie}

\usepackage{algorithm}
\usepackage{algorithmic}

\usepackage{hyperref}
\definecolor{linkcolor}{RGB}{0, 0, 128}
\hypersetup{
     colorlinks   = true,
     citecolor    = linkcolor,
     linkcolor    = linkcolor,
     urlcolor     = linkcolor,
}
\usepackage{pifont}
\usepackage{listings}

\setlist[itemize]{leftmargin=*,itemsep=0em,parsep=0.3em,topsep=0.3em}

\DeclareUnicodeCharacter{2212}{\ensuremath{-}}

\definecolor{maroon}{HTML}{F26035}
\definecolor{yellow}{HTML}{FDBC42}
\definecolor{lavender}{HTML}{734f96}
\definecolor{darkergrey}{HTML}{444444}
\definecolor{midgrey}{HTML}{e6eded}
\definecolor{ai2pink}{HTML}{f0529c}
\definecolor{ai2midpink}{HTML}{fad3e5}
\definecolor{ai2lightpink}{HTML}{fbecf3}
\definecolor{ai2midwhite}{HTML}{f2e5d9}
\definecolor{ai2offwhite}{HTML}{fbf4ee}
\definecolor{ai2green}{HTML}{0fcb8c}
\definecolor{ai2lightgreen}{HTML}{e7f9f3}
\definecolor{ai2darkgreen}{HTML}{105257}
\definecolor{ai2purple}{HTML}{B932EB}
\definecolor{ai2lightpurple}{HTML}{f7e8fc}
\definecolor{neutralEight}{HTML}{343434}
\definecolor{neutralFive}{HTML}{838383}
\definecolor{neutralThree}{HTML}{bebebe}
\definecolor{neutralOne}{HTML}{dedede}
\definecolor{lightgrey}{HTML}{fafcfc}
\definecolor{plum}{rgb}{0.56,0.27,0.52}
\definecolor{ai2lightteal}{HTML}{C9D9DA}

\usepackage{tikz}

\definecolor{maroon}{HTML}{F26035}
\definecolor{yellow}{HTML}{FDBC42}
\definecolor{darkred}{RGB}{156, 39, 33}
\definecolor{darkblue}{RGB}{31, 90, 153}
\definecolor{forestgreen}{rgb}{0.13, 0.55, 0.13}
\definecolor{brickred}{rgb}{0.8, 0.25, 0.33}
\definecolor{olmoDarkBlue}{HTML}{012e59}
\definecolor{olmoBlue}{HTML}{265ed4}
\definecolor{olmoLightBlue}{HTML}{012e59}
\definecolor{olmoTeal}{HTML}{00d5ff}
\definecolor{olmoYellow}{HTML}{ffbb00}
\definecolor{olmoOrange}{HTML}{ff9100}

\usepackage{setspace}

\usepackage{nicematrix}
\newcolumntype{L}[1]{>{\raggedright\let\newline\\\arraybackslash\hspace{0pt}}m{#1}}
\newcolumntype{C}[1]{>{\centering\let\newline\\\arraybackslash\hspace{0pt}}m{#1}}
\newcolumntype{R}[1]{>{\raggedleft\let\newline\\\arraybackslash\hspace{0pt}}m{#1}}
\newcolumntype{P}[1]{>{\centering\let\newline\\\arraybackslash\columncolor{ai2lightpink}}m{#1}}
\newcommand{\allenAiAff}{\raisebox{.28em}{\hspace{.02em}\scalebox{0.7}{\textbf{1}}}}

\newcommand{\uwAff}{\raisebox{.28em}{\hspace{.02em}\scalebox{0.7}{\textbf{3}}}}
\newcommand{\stanAff}{\raisebox{.28em}{\hspace{.02em}\scalebox{0.7}{\textbf{2}}}}
\newcommand{\commaAff}{\raisebox{.28em}{\hspace{.02em}\scalebox{0.7}{\textbf{,}\hspace{0.1em}}}}

\newcommand{\huggingface}{\raisebox{-1.5pt}{\includegraphics[height=1.05em]{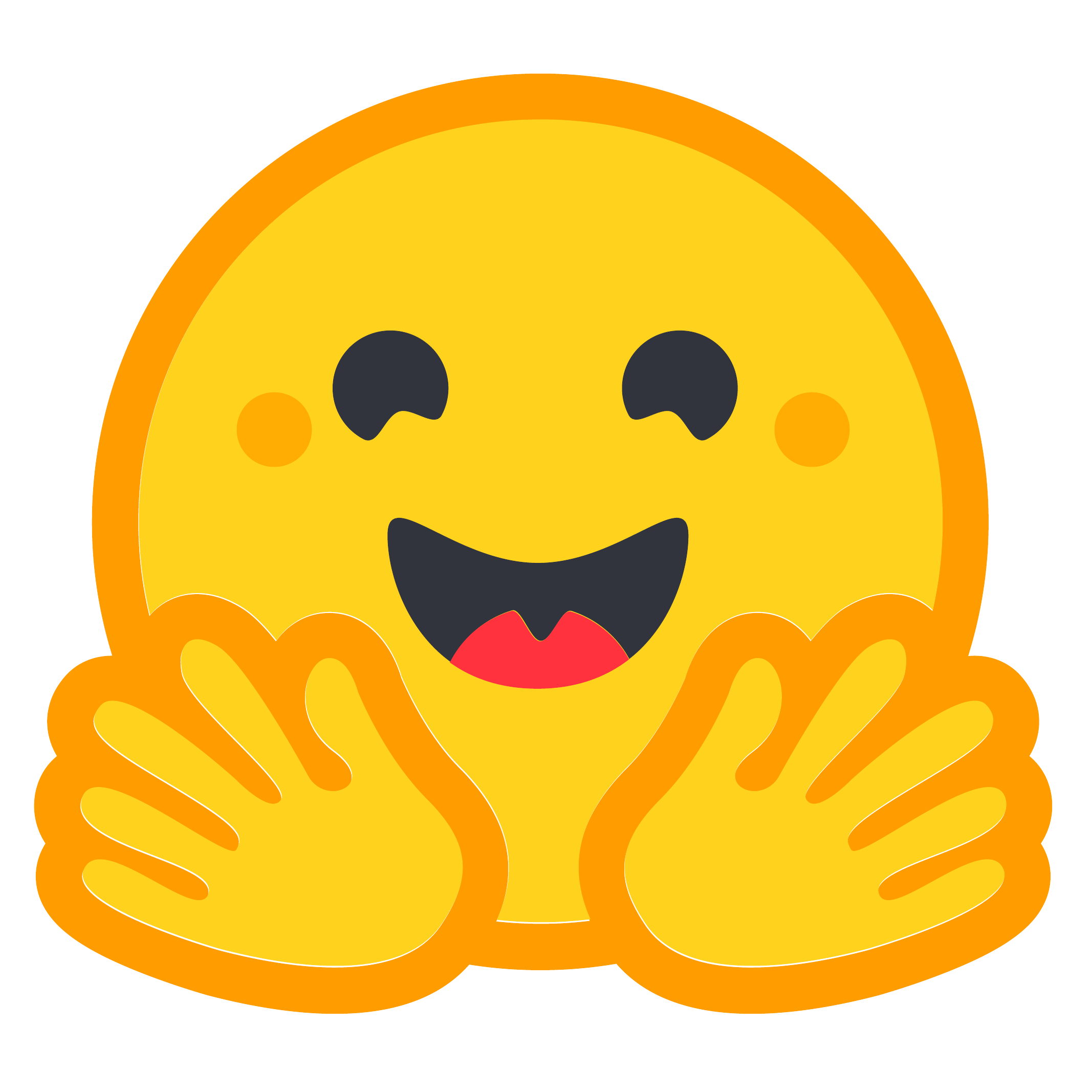}}\xspace}

\newcommand{\emailLogo}{\raisebox{-1.5pt}{\includegraphics[height=1.05em]{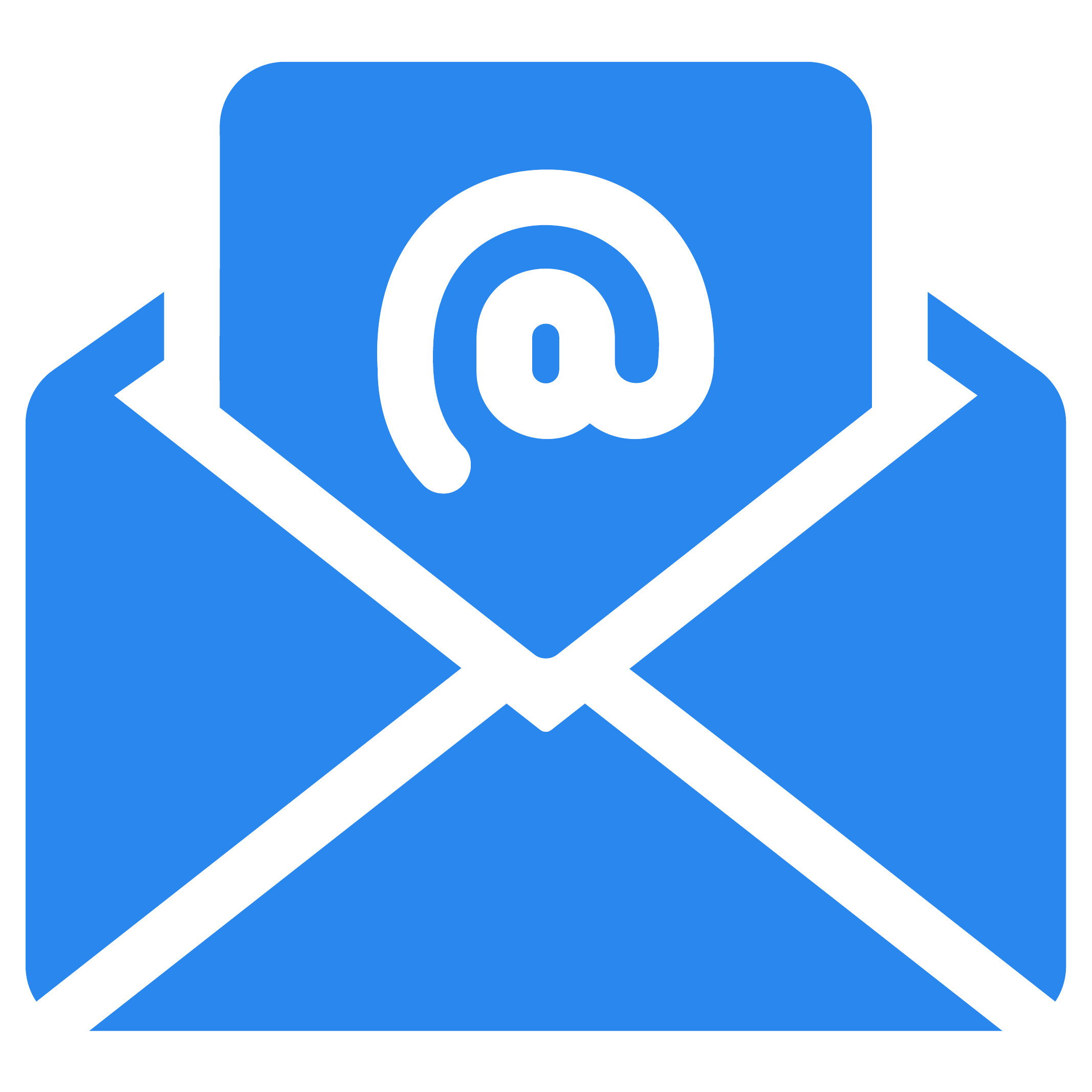}}\xspace}
\newcommand{\github}{\raisebox{-1.5pt}{\includegraphics[height=1.05em]{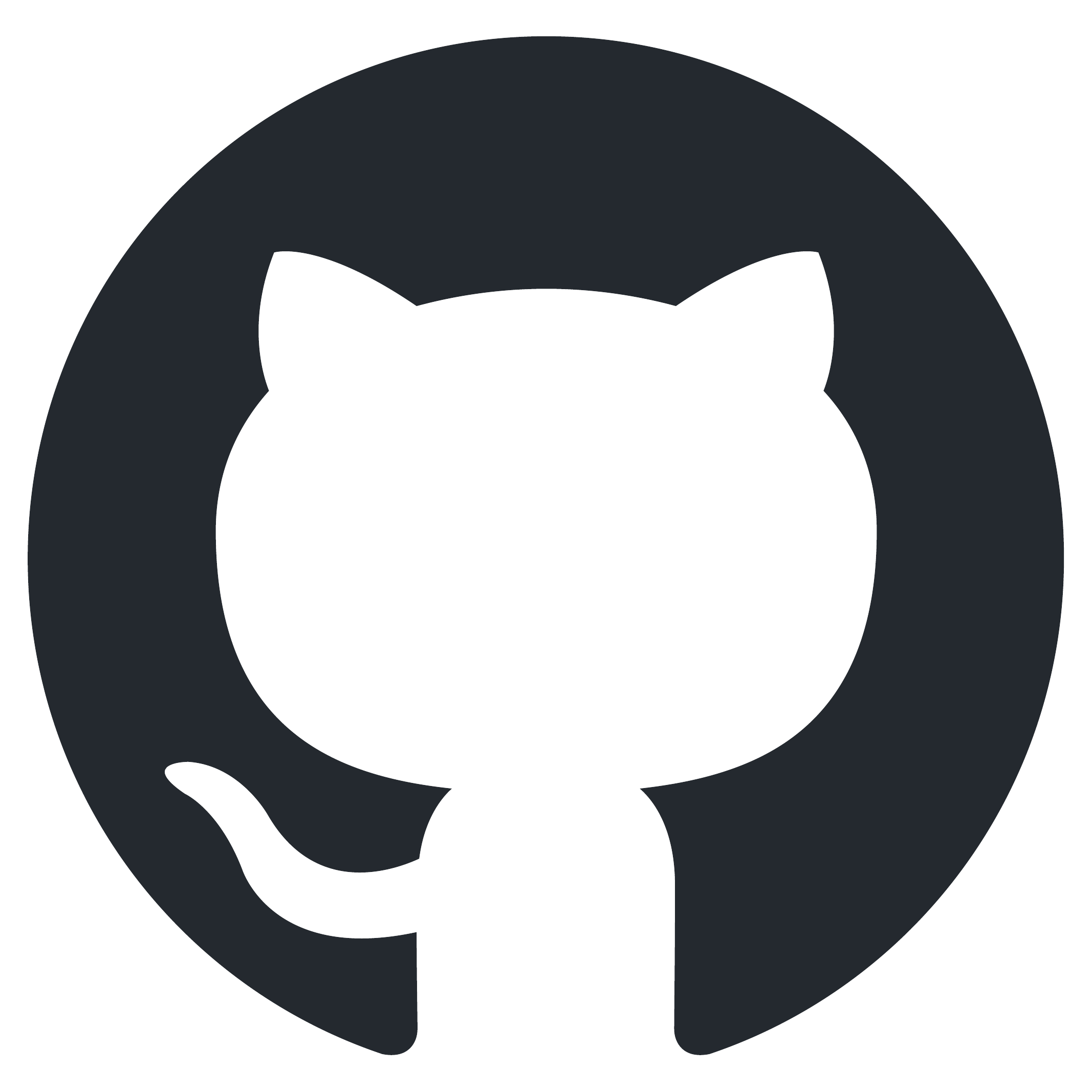}}\xspace}

\newcommand{\sbf}[1]{{\footnotesize\bfseries #1}}

\title{Olmix: A Framework for Data Mixing Throughout LM Development} 

\authorOne[\allenAiAff\commaAff\stanAff]{Mayee F. Chen}
\authorOne[\allenAiAff]{Tyler Murray}
\authorOne[\allenAiAff]{David Heineman}
\authorOne[\allenAiAff]{Matt Jordan}
\authorOne[\allenAiAff\commaAff\uwAff]{Hannaneh Hajishirzi}
\authorOne[\stanAff]{Christopher R\'e}
\authorOne[\allenAiAff]{Luca Soldaini}
\authorOne[\allenAiAff\commaAff\uwAff]{Kyle Lo}

\affiliation[\allenAiAff]{Allen Institute for AI}
\affiliation[\stanAff]{Stanford University}
\affiliation[\uwAff]{University of Washington}

\abstract{
Data mixing---determining the ratios of data from different domains---is a first-order concern for training language models (LMs). While existing mixing methods show promise, they fall short when applied during real-world LM development. We present \olmix, a framework that addresses two such challenges. First, the configuration space for developing a mixing method is not well understood---design choices across existing methods lack justification or consensus and overlook practical issues like data constraints. We conduct a comprehensive empirical study of this space, identifying which design choices lead to a strong mixing method. Second, in practice, the domain set evolves throughout LM development as datasets are added, removed, partitioned, and revised---a problem setting largely unaddressed by existing works, which assume fixed domains. We study how to efficiently recompute the mixture after the domain set is updated, leveraging information from past mixtures. We introduce mixture reuse, a mechanism that reuses existing ratios and recomputes ratios only for domains affected by the update. Over a sequence of five domain-set updates mirroring real-world LM development, mixture reuse matches the performance of fully recomputing the mix after each update with 74\% less compute and improves over training without mixing by 11.6\% on downstream tasks.
}

\metadata[\github Code:]{\href{https://github.com/allenai/olmix}{\texttt{Olmix}}}
\metadata[\huggingface Data:]{\href{https://huggingface.co/datasets/allenai/olmix}{\texttt{Olmix}}}
\metadata[\emailLogo Contact:]{\color{ai2accent}{\texttt{mfchen@cs.stanford.edu \quad \{lucas,kylel\}@allenai.org}}}

\newtheorem{theorem}{Theorem}
\newtheorem{lemma}{Lemma}

\newtheorem{assumption}{Assumption}

\usepackage{xspace}

\newcommand{\D}[0]{\mathcal{D}}

\newcommand{\R}[0]{\mathbb{R}}

\newif\ifsinglecolumn

\newcommand{\F}[0]{\mathcal{F}}

\definecolor{darkgreen}{RGB}{0, 100, 0}
\definecolor{darkred}{RGB}{139, 0, 0}

\newcommand\SmallMatrix[1]{{%
\tiny\arraycolsep=0.3\arraycolsep\ensuremath{\begin{bmatrix}#1\end{bmatrix}}}}

\newcommand{\olmix}{\textsc{Olmix}\xspace}

\newcommand{\LM}[0]{\text{LM}}
\newcommand{\Ssmall}[0]{S_{\text{small}}}
\newcommand{\Rsmall}[0]{R_{\text{small}}}

\newcommand{\kl}[0]{D_{\text{KL}}}

\newcommand{\Dfix}[0]{\mathcal{D}_{\text{fix}}}
\newcommand{\Dcomp}[0]{\mathcal{D}_{\text{comp}}}

\newcommand{\alphafix}[0]{\alpha_{\text{fix}}}
\newcommand{\alphacomp}[0]{\alpha_{\text{comp}}}

\newcommand{\Dpartial}[0]{\mathcal{D}_{\text{partial}}}

\newcommand{\base}[0]{\textsc{OlmixBase}\xspace}

\newcommand{\fix}[0]{\text{fix}}
\newcommand{\comp}[0]{\text{comp}}

\newcommand{\fullmixreuse}[0]{\textsc{FullMixtureReuse}\xspace}
\newcommand{\partialmixreuse}[0]{\textsc{PartialMixtureReuse}\xspace}

\newcommand{\Add}[0]{\texttt{Add}\xspace}
\newcommand{\Remove}[0]{\texttt{Remove}\xspace}
\newcommand{\Partition}[0]{\texttt{Partition}\xspace}
\newcommand{\Revise}[0]{\texttt{Revise}\xspace}

\newcommand{\tv}[0]{\text{TV}}

\begin{document}

\maketitle

\section{Introduction}

\begin{figure}[h]
    \centering
    \includegraphics[width=1.0\linewidth, trim=0 0 0 1cm, clip]{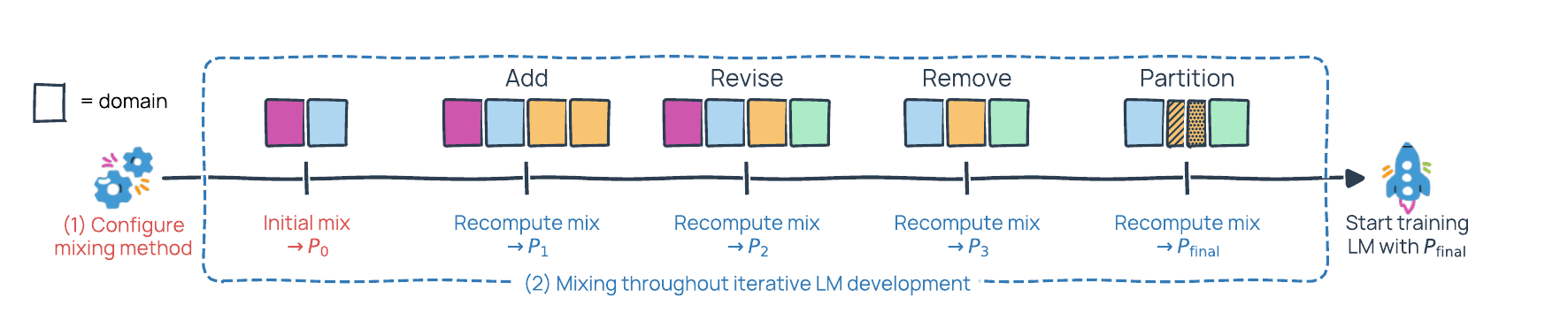}
    \caption{Two problems with data mixing encountered during LM development: (1) How to best configure your mixing method? (2) How to efficiently mix under evolving domain sets?}
    \label{fig:banner}
\end{figure}

Modern language models (LMs) are trained on datasets composed of many domains, such as web text, code, and PDFs. The composition of these domains is crucial for strong downstream performance, making data mixing a first-order component of LM development~\citep{dubey2024llama, qwen2.5, olmo2025olmo3}. 
However, finding a good mix is non-trivial: practitioners often resort to manual weight tuning or exhaustive search, which can require many training runs---possibly thousands of GPU hours---to assess performance.
This has resulted in a growing literature on data mixing methods that aim to find strong mixtures systematically with less compute~\citep{xie2023doremioptimizingdatamixtures, fan2024dogedomainreweightinggeneralization, chen2025aioliunifiedoptimizationframework}. 
Many mixing methods that achieve promising results follow a common \textit{offline mixing schema}~\citep{liu2025rethinkingdatamixturelarge, liu2025regmixdatamixtureregression, ye2025datamixinglawsoptimizing} that consists of three steps: 1. train a set of smaller proxy models on different mixtures (a ``swarm''), 2. fit a regression model on this swarm to predict performance from mixture weights, and 3. propose a final mix that optimizes predicted performance. 

However, data mixing remains far from solved: translating these methods to real-world LM development surfaces challenges that the existing literature does not address.
In this work, we present \olmix, a mixing framework developed for pre-training Olmo~3~\citep{olmo2025olmo3}, a family of fully open, state-of-the-art language and reasoning models at the 7B and 32B parameter scales.
We identify two challenges in applying data mixing methods both at the \emph{start of} LM development--- configuring a mixing method on an initial domain set---and \emph{throughout} development---recomputing mixtures efficiently as domains evolve (Figure~\ref{fig:banner}). \olmix makes progress on addressing both.

\textbf{P1: The mixing configuration problem (\S\ref{sec:config-choices}).} At the start of LM development, we must configure a mixing method on an initial set of domains. However, the configuration space of existing methods is poorly understood; design choices are often unjustified, contradict across methods, or fail to address practical issues like data constraints (Table~\ref{tab:mixing_design_choices}). 
As a result, practitioners have conflicting or little guidance on which choices lead to a strong mixing method.

The first component of \olmix addresses P1 through a comprehensive study of the configuration space of the offline mixing schema. 
We identify seven key design choices needed to instantiate a mixing method from the schema, and we empirically study each design choice in the context of pre-training Olmo 3. Below, we highlight several of these design choices and our findings:
\begin{enumerate}[itemsep=0pt, topsep=2pt, leftmargin=12pt]
    \item \emph{What is the minimum number of proxy runs (swarm size) needed to learn an effective mix for a given domain set?} 
    We find that the required number of runs scales linearly with the number of domains (Figure~\ref{fig:swarm_size}).

    \item \emph{Is there an optimal family of regression models for predicting mix performance?}
    We find that different regression model families excel at different swarm sizes---a key confounding factor that explains disagreement in prior work---but that the log-linear model adapted from~\citet{ye2025datamixinglawsoptimizing} achieves the best overall downstream performance (Figure~\ref{fig:regression_models}).   
        
    \item \emph{How can we perform data mixing under finite data constraints?} Existing methods assume unlimited data, but in practice, certain mixes require sampling more data than is available, leading to excessive sample repetition that degrades performance~\citep{muennighoff2025scalingdataconstrainedlanguagemodels}. We show that incorporating constraints into the mixture optimization problem controls sample repetition while maintaining performance, and that these constraints significantly shape the proposed mix (Figure~\ref{fig:vary_rep_top_7}).    
\end{enumerate}

These findings inform our base mixing method, \base (Algorithm~\ref{alg:base_method}), which we use throughout Olmo~3 development and as a building block for addressing P2.

\textbf{P2: The evolving domain problem (\S\ref{sec:conditional_mixing}).} 
Existing mixing literature assumes a fixed domain set~\citep{albalak2024surveydataselectionlanguage}, but real-world LM development is iterative. 
In our experience building Olmo 1--3~\citep{olmo2025olmo3}, we repeatedly added, removed, revised, and partitioned datasets---a pattern also observed in other iterative development efforts, such as SmolLM 1--3~\citep{bakouch2025smollm3}.
As a result, mixtures must repeatedly be recomputed in practice as the domain set evolves over time. 
However, recomputing from scratch (e.g., using \base) after each domain update becomes increasingly expensive as practitioners make more updates to their datasets. 
Instead, we propose to leverage information from historical mixtures, motivating a new problem: \emph{given an existing mix, how do we efficiently recompute the mix to maintain strong performance after a domain update?}

The second component of \olmix addresses P2 by introducing a spectrum of recomputation strategies that trade off compute cost against performance. At one end is full recomputation---mixing from scratch after each domain update---which achieves high performance but incurs high computational cost. We make the following contributions developing \textit{mixture reuse} alternatives along this performance-cost spectrum (Figure~\ref{fig:mixreuse}):
\begin{enumerate}[itemsep=0.00pt,topsep=0pt,leftmargin=12pt]
\item We propose \fullmixreuse, a mechanism that reuses the existing mix over the domains that are unaffected by the update. In particular, this mechanism freezes the relative weights among the unaffected domains and recomputes only their total weight alongside the weights of the affected domains. This restricts recomputation (via \base) to a lower-dimensional subspace, reducing computational cost.
\item We analyze when \fullmixreuse performs as well as full recomputation. The performance gap is governed by how much the optimal ratios change after the update and the coupling between reused and recomputed domains, measured by how much they impact the same downstream tasks. 
When both effects are small, \fullmixreuse roughly matches full recomputation while being less costly.
\item We propose \partialmixreuse, an extension of \fullmixreuse that provides a middle ground between full reuse and full recomputation which selectively recomputes the mix on some unaffected domains in addition to the affected ones. This approach can reduce coupling effects and narrow the performance gap to full recomputation at a slightly higher cost than \fullmixreuse.
\end{enumerate}
 
Empirically, we evaluate \olmix---combining our configuration from P1 (\base) with our reuse mechanisms from P2 (\fullmixreuse and \partialmixreuse)---in a real-world LM development setting where the domain set evolves through 5 updates culminating in 64 domains.
When training 1B parameter models on 100B tokens, \fullmixreuse improves over the natural distribution (a baseline with ratios proportional to domain sizes) by +11.6\% and captures 95\% of full recomputation's gains while requiring 74\% fewer proxy runs.
\partialmixreuse reaches 98\% (+12.0\% improvement) while using 67\% fewer proxy runs.
In addition, our best mix obtained via mixture reuse is $3.05\times$ more data-efficient than the natural distribution, reaching the same final performance in one-third as many training steps.

\section{Related Work}

\textbf{Mixing Methods.} The offline mixing schema, also described as ``function fitting-based offline methods'' in~\citet{liu2025rethinkingdatamixturelarge} has been extensively used in existing mixing methods. In Table~\ref{tab:mixing_design_choices}, we provide an overview of several offline mixing methods and how they address the key design choices we study. Some methods use explicit parametric regression models~\citep{que2024dcptlawdomainspecificcontinual}, while others use nonparametric approaches like LightGBM and Gaussian Processes~\citep{liu2025regmixdatamixtureregression, chen2025admirebayesoptaccelerateddatamixture}. Some methods explicitly model the role of proxy model size or training budget~\citep{ge2025bimixbivariatedatamixing, ye2025datamixinglawsoptimizing, kang2025autoscalescaleawaredatamixing} while others assume direct generalization from proxy models to target models~\citep{held2025optimizingpretrainingdatamixtures}. There are also several methods that build on top of RegMix~\citep{liu2025regmixdatamixtureregression}, exploring how it can be augmented with better domains~\citep{wettig2025organizewebconstructingdomains}, iterative swarms~\citep{diao2025climbclusteringbasediterativedata}, or more features~\citep{belenki2025optimizingpretrainingdatamixtures}. Our work examines key design choices for the offline schema that underlies all these methods.

Aside from the offline schema, DoReMi~\citep{xie2023doremioptimizingdatamixtures} and DoGE~\citep{fan2024dogedomainreweightinggeneralization} determine a mixture by using one or two proxy runs that dynamically explore the mixture weight space; this is in contrast to using many proxy runs with static mixes. While these approaches have proven effective in some settings, previous analysis~\citep{chen2025aioliunifiedoptimizationframework} has suggested some suboptimality in how the dynamic update rules are constructed, and the offline schema is generally considered simpler to implement.

Online mixing methods~\citep{chen2023skillitdatadrivenskillsframework, jiang2024adaptivedataoptimizationdynamic, albalak2023efficientonlinedatamixing} adjust the mixture weights throughout the final training run. Rather than exploring the mixture space and learning from it offline, these methods do this on the fly during training. Note that our notion of dynamics, which is over the domain set and during LM development (i.e., before the final training run), is different from the dynamic aspect of online mixing methods.

Existing mixture literature largely assumes fixed domain sets. The recent work Chameleon~\citep{xie2025chameleonflexibledatamixingframework} addresses adaptability to domain changes by computing domain weights from learned embeddings using kernel ridge leverage scores, which allows direct transfer to new data without proxy retraining. Chameleon focuses on adding new domains and computes weights directly from domain embeddings without the swarm-based regression approach of offline methods. In contrast, our mixture reuse approach explicitly handles various domain update operators (add, remove, partition, and revise) and is designed to work with any method that follows the offline mixing schema. Moreover, our approach provides a theoretical and empirical analysis of when and why existing mixes can be reused.

\textbf{Data-constrained settings.} Several works have studied how to train language models under data constraints. UniMax~\citep{chung2023unimaxfairereffectivelanguage} proposes an allocation algorithm for multilingual pretraining that distributes a fixed token budget to maximize uniform coverage across languages while capping repetitions to avoid overfitting on low-resource languages. \citet{muennighoff2025scalingdataconstrainedlanguagemodels} find that up to 4 epochs of repetition yields negligible degradation and propose modified scaling laws that account for diminishing returns of repeated tokens. Our work explicitly integrates repetition constraints directly into data mixing, enabling principled allocation of limited data budgets while producing a mix that yields strong downstream performance.

\textbf{LM Data Development.} Real-world LM development involves iterative refinement of training data. Works like DCLM~\citep{dclm}, Dolma~\citep{soldaini2024dolma}, and FineWeb~\citep{penedo2024fineweb} extensively document the curation processes involved in creating high-quality pretraining corpora, including quality filtering, deduplication strategies, and careful domain selection. Continuous projects like OLMo 1-3~\citep{Groeneveld2024OLMoAT, olmo20242olmo2furious, olmo2025olmo3} and SmolLM 1-3~\citep{allal2024smollm, allal2025smollm2smolgoesbig, bakouch2025smollm3} showcase how training data evolves over time. Motivated by these works, \olmix views data mixing as an ongoing process throughout LM development.
\section{The Mixing Configuration Problem}
\label{sec:config-choices}

We present the first component of \olmix: a comprehensive empirical study of the configuration space of mixing methods. We begin by formalizing the data mixing problem and describing the offline mixing schema (\S\ref{sec:background}). Then, we enumerate the key design choices that must be made to configure an offline mixing method (\S\ref{sec:design_choices}).
Our main contribution is an empirical study of these design choices in \S\ref{sec:study}. 
These findings inform \base (\S\ref{sec:base_method}), the mixing method we used throughout Olmo 3 development.

\subsection{Background}\label{sec:background}

\subsubsection{The Data Mixing Problem}\label{sec:setup_1}

Our goal is to determine ratios over training data domains that result in strong downstream performance.

\textbf{Domain set and mixes.} Let $\D = \{D_1, D_2, \dots, D_m \}$ be a set of $m$ domains, where each domain $D_i$ has a training dataset of size $N_i$ tokens. A domain is a group of data, ranging from coarse-grained \textbf{sources} (defined by data provenance) to fine-grained \textbf{topics} (semantically coherent partitions of a source)~\citep{wettig2025organizewebconstructingdomains}.
We specify a data mixture via a probability vector $p \in \triangle^{m-1}$, such that training on $R$ total tokens uses a dataset with $p_i \cdot R$ tokens from $D_i$ for each domain $i$.

\textbf{Model and evaluation.} We train a target LM of $S$ parameters for $R$ tokens on $p$, denoted as $\LM(S, R, p)$.
We then evaluate this model on a suite of $n$ downstream tasks. 
We measure the performance $f_i(\LM(S, R, p))$ on each task $i$ in terms of bits-per-byte (BPB), the negative log likelihood of the correct answer normalized by answer length in UTF-8 bytes. \citet{heineman2025signalnoiseframeworkreducing} showed that BPB can be used for decision-making at small model scales and \citet{huang2024compression} showed that BPB sets correlate with downstream performance across capabilities and model families.

\textbf{Goal.} Given a domain set $\D$ and target model configuration of $S$ parameters and $R$ tokens, we aim to find a mixture $p^\star$ that minimizes the average BPB across all downstream tasks---that is, minimizing $\frac{1}{n} \sum_{i = 1}^n f_i(\LM(S, R, p))$.

\subsubsection{The Offline Mixing Schema}\label{sec:schema}

Many existing methods follow an offline mixing schema to propose a mix; these methods are also described as ``function fitting-based offline methods'' in~\citet{liu2025rethinkingdatamixturelarge}'s survey. We describe the three steps of this schema below and present them in Figure~\ref{fig:schema} as well. 

\textbf{Step 1: Swarm construction.} We sample the space of possible mixes by training a ``swarm'' of $K$ small proxy models of size $\Ssmall$ on $\Rsmall$ tokens, each with different mixture weights $p^1, p^2, \dots, p^K \in \triangle^{m-1}$ sampled from some distribution $\mathcal{P}$. We evaluate each proxy model on the downstream task suite to obtain $y_{ij} := f_i(\LM(\Ssmall, \Rsmall, p^j))$, the performance of the $j$th proxy model on the $i$th task, for all tasks over the entire swarm. Altogether, this creates a dataset $\{(p^j, \{y_{ij}\}_{i=1}^n)\}_{j=1}^K$ of mixture weights paired with their performance across all tasks.

\textbf{Step 2: Regression model.} We fit regression models using the above dataset to capture the relationship between the mixture weights and downstream performance. 
We learn a function $\hat{f} \in \F$, where $\hat{f}(p) \approx \frac{1}{n}\sum_{i = 1}^n f_i(\LM(\Ssmall, \Rsmall, p))$. This enables us to predict the average downstream BPB of a proxy model trained on any candidate mix $p$.

\textbf{Step 3: Mixture optimization.} We propose a mix by solving an optimization problem that uses the regression model $\hat{f}$ as a surrogate for true performance. This optimization problem takes the form $\underset{p \in \mathcal{S}}{\text{minimize}} \hat{f}(p)$, where $\mathcal{S} \subseteq \triangle^m$ denotes the feasible set, which may be the full probability simplex or a restricted subset.

\begin{figure}
    \centering
    \includegraphics[width=\linewidth]{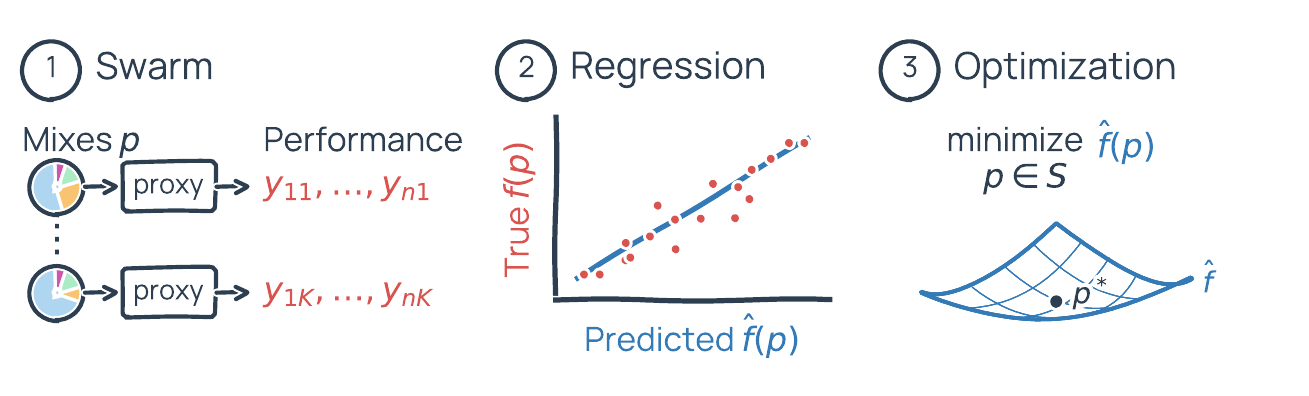}
    \caption{The offline mixing schema used by many existing methods (\S\ref{sec:schema}). We study the design choices needed to configure the schema to develop a strong mixing method (\S\ref{sec:design_choices}-\S\ref{sec:study}).}
    \label{fig:schema}
\end{figure}

\subsection{\olmix Design Choices}\label{sec:design_choices}

Key design choices of the offline mixing schema are not explicitly identified in existing literature. We articulate these design choices as a configuration checklist of research questions (RQs) that practitioners must consider when developing their own mixing methods.

\begin{tcolorbox}[colback=ai2offwhite,colframe=black,boxrule=0.5pt,arc=3pt,left=6pt,right=6pt,top=6pt,bottom=6pt]
\textbf{Swarm construction:}
\begin{itemize}[itemsep=0.00pt,topsep=0pt,leftmargin=30pt]
    \item[\emph{RQ1}] What is the smallest proxy model size (the number of parameters, $\Ssmall$) such that decision-making generalizes to larger target models?
    \item[\emph{RQ2}] How many proxy runs $K$ do we need to learn a good mix on $m$ domains?
    \item[\emph{RQ3}] How should we specify the distribution $\mathcal{P}$ to sample the mixes for the proxy runs?
\end{itemize}

\textbf{Regression model:}
\begin{itemize}[itemsep=0.00pt,topsep=0pt,leftmargin=30pt]
    \item[\emph{RQ4}] Is there an optimal family of regression models ($\mathcal{F}$) for predicting mix performance?
    \item[\emph{RQ5}] At what granularity should we fit the regression models in order to construct $\hat{f}(p)$?
\end{itemize}

\textbf{Mix optimization:}
\begin{itemize}[itemsep=0.00pt,topsep=0pt,leftmargin=30pt]
    \item[\emph{RQ6}] How do we mix under finite data constraints?
    \item[\emph{RQ7}] How do we solve the optimization problem?
\end{itemize}
\end{tcolorbox}

Table~\ref{tab:mixing_design_choices} shows how several existing methods address these design choices we identify. Shaded cells indicate design choices with reported justification while white cells represent those that lack sufficient explanation or exploration of alternatives. There are three takeaways:
\begin{itemize}[itemsep=0.00pt,topsep=0pt,leftmargin=12pt]
\item Many design choices lack guidance beyond the specific
instantiation; for instance, decisions on the swarm size
and distribution are rarely explained.
\item Even design choices that have justification lack consensus. For example, each existing method proposes a different regression model and reports strong regression fit to support it.
\item Critical aspects of LM development remain unaddressed, such as mixing under data constraints.
\end{itemize}

These gaps motivate a systematic investigation of the design choices within the offline mixing schema. 

\begin{table}[t]
    \centering
    \caption{Design choices across offline mixing methods. For each design choice and mixing method, we identify the method's configuration. We shade cells in pink if we found empirical justification for their configuration. In the final column, we present the configuration for \base.}
    \begin{footnotesize}
    \setlength{\tabcolsep}{3pt}
    \begin{tabular}
{L{2.0cm}C{1.9cm}C{1.9cm}C{1.9cm}C{1.9cm}C{1.9cm}C{1.9cm}C{1.9cm}}
        \toprule
        \sbf{Design Choice} & \sbf{RegMix}~\citep{liu2025regmixdatamixtureregression} & \sbf{DML}~\citep{ye2025datamixinglawsoptimizing} & \sbf{AutoScale}~\citep{kang2025autoscalescaleawaredatamixing} & \sbf{BiMix}~\citep{ge2025bimixbivariatedatamixing} & \sbf{ADMIRE-BayesOpt} \newline \citep{chen2025admirebayesoptaccelerateddatamixture} & \sbf{CLIMB}~\citep{diao2025climbclusteringbasediterativedata} & \sbf{\base (Algorithm~\ref{alg:base_method})} \\
        \midrule
        \rowcolor{midgrey}\multicolumn{8}{l}{\sbf{Swarm Construction}} \\
        \midrule
        Proxy model size & \cellcolor{ai2lightpink} 1M
 & \cellcolor{ai2lightpink} 70, 160, 305, 410M & Target & 280M & 1M, 60M & \cellcolor{ai2lightpink} 350M & \cellcolor{ai2lightpink} 30M \\
        Swarm size (vs $m$ domains) & 512 ($m=17$) & \cellcolor{ai2lightpink} 20 ($m=7$) & $2m+1$ & 4 & 101 ($m=17$) & \cellcolor{ai2lightpink} 112 ($m=21$) & \cellcolor{ai2lightpink} $3(m+1)$ \\
        Swarm distribution & Dirichlet with natural prior & Exponential grid & Exponential grid & Entropy-weighted & Dynamic & \cellcolor{ai2lightpink} Dirichlet with natural prior & \cellcolor{ai2lightpink} Dirichlet with natural prior (sparse for topics, dense for sources) \\
        \midrule
        \rowcolor{midgrey}\multicolumn{8}{l}{\sbf{Regression Model}} \\
        \midrule
        Regression model family & \cellcolor{ai2lightpink} LightGBM & \cellcolor{ai2lightpink} Log-Linear & \cellcolor{ai2lightpink} Power Law & \cellcolor{ai2lightpink} Power Law &  Gaussian Process & \cellcolor{ai2lightpink} LightGBM & \cellcolor{ai2lightpink} Log-Linear \\
        Regression granularity & Aggregated & Aggregated & Per-Task & Per-Task & Aggregated & Aggregated & \cellcolor{ai2lightpink} Per-Task \\
        \midrule
        \rowcolor{midgrey}\multicolumn{8}{l}{\sbf{Mixture Optimization}} \\
        \midrule
        Data repetition constraints & No & No & No & No & No & No & \cellcolor{ai2lightpink} Yes \\
        Optimization solver & Search & Search & Gradient Descent & Exact Solver & Search & Search & \cellcolor{ai2lightpink} Exact Solver with KL reg. \\
        \bottomrule
    \end{tabular}
    \end{footnotesize}
    \label{tab:mixing_design_choices}
\end{table}

\subsection{\olmix Study: Configuring a Mixing Method}\label{sec:study}

We conduct a comprehensive study of each design choice. 
We describe our experimental setup (\S~\ref{sec:study_setting}) and then investigate the design choices pertaining to each component of the offline mixing schema (\S~\ref{sec:swarm}-\ref{sec:optimization}).

\subsubsection{Experimental Setup}\label{sec:study_setting}

\textbf{Data.} We use DCLM~\citep{dclm} partitioned into 24 topic-based domains using WebOrganizer~\citep{wettig2025organizewebconstructingdomains}. See Table~\ref{tab:domain_sizes} for domain names and token counts. In Appendix~\ref{supp:study_details}, we also provide results for mixing over the final data sources of Table~\ref{tab:domain-updates-full}. 

\textbf{Model.} We train 1B parameter decoder-only transformer models using Olmo 2 architecture~\citep{olmo20242olmo2furious} to 100B tokens (5x Chinchilla). We use a batch size of 512, sequence length of 4096, and max learning rate of $0.0018$. See Appendix~\ref{supp:exp_details} for more details. 

\textbf{Evaluation.} We measure BPB (bits-per-byte) over gold responses on $52$ downstream tasks spanning math, code, and commonsense QA. 
For QA tasks, the gold continuation is the answer text; for math and code tasks it is a human-written response. 
We treat subtasks (e.g. MMLU or coding language subsets) as standalone tasks when taking an macro-average of BPB scores. See Table \ref{tab:eval_suite} for the entire list of tasks.

Unless otherwise specified, all experiments use DCLM topics as the set of domains and \base (\S~\ref{sec:base_method}, Algorithm~\ref{alg:base_method}) as the default configuration, varying only the component under study. Appendix~\ref{supp:study_details} provides full experiment details.

\subsubsection{Swarm Construction Study}\label{sec:swarm}

\textbf{RQ1: Proxy model size.} 
Proxy models must provide reliable signal that transfers to the target model's scale while ideally being small so that computational costs are low. Existing methods use sizes ranging from 1M to 400M parameters, making it unclear what size is sufficient.

We train many proxy–target model pairs, where each pair consists of a proxy model and a 1B target model trained on the same mixture ratios.
We compute the Spearman rank correlation between proxy and target model performance (average BPB) to quantify transfer; a high correlation indicates that rankings at the proxy scale predict rankings at the target scale, ensuring reliability of proxy swarm decisions. We consider proxy models of 1M, 15M, 30M, and 60M parameters with 5x Chinchilla multiplier.

Figure~\ref{fig:proxy_size} shows that proxy models with $\ge 15$M parameters achieve strong rank correlation ($\rho > 89$).
However, 1M models show significantly degraded correlation ($r = 73$), making them unreliable for guiding mixture decisions. This is in contrast to the recommendation of RegMix~\citep{liu2025regmixdatamixtureregression}\footnote{Investigating the public RegMix code, we found their 1M implementation is closer to 15M, which our results suggest is a good proxy size.}.
\textbf{For our setting, we use 30M parameter proxy models with 3B tokens (5x Chinchilla}), which achieves a Spearman correlation of 89.6 with 1B target models.

\begin{figure}
    \centering
    \includegraphics[width=0.5\linewidth]{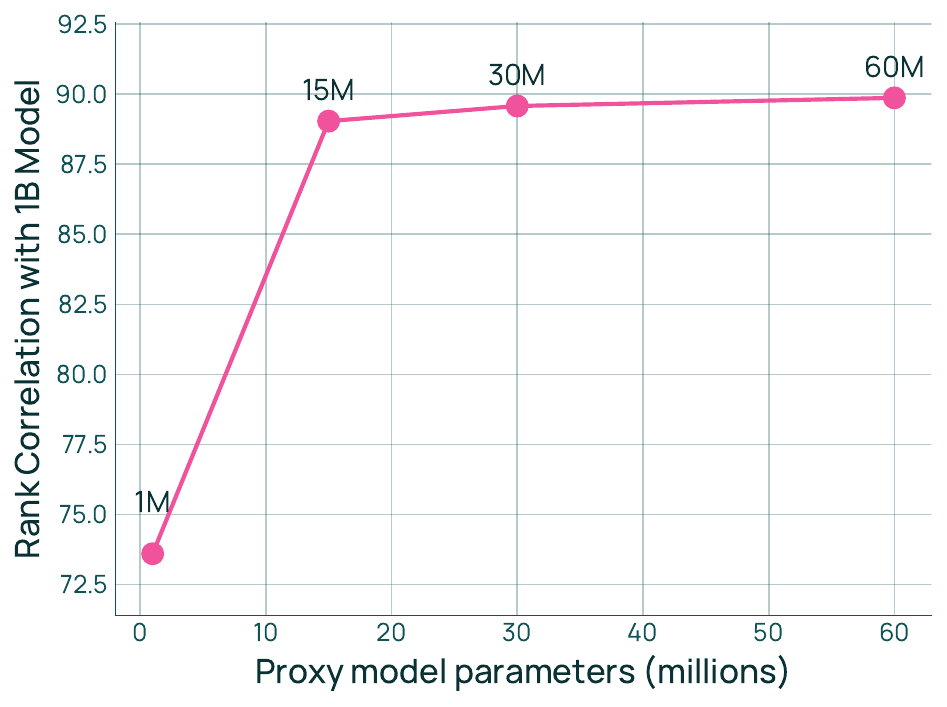}
    \caption{Correlation in performances between pairs of proxy models and 1B target models trained on the same mixture. Proxy models with over 15M parameters achieve strong rank correlation.}
    \label{fig:proxy_size}
\end{figure}

\textbf{RQ2: Swarm size in terms of number of domains.} 
The swarm size directly controls the cost-performance tradeoff: more proxy runs increase computational cost but improve mix quality. However, it is unclear what the tradeoff rate is when mixing over $m$ domains, since existing methods use fixed swarm sizes ranging from $20$ to over $500$, and few relate these choices to the number of domains $m$.

We examine the sample complexity of the log-linear regression model, which we evaluate against other regression models in RQ4 and use in \base.  Specifically, given a set of $m$ domains, the sample complexity measures how many proxy runs $K$ are needed to produce a mix 
whose downstream BPB is within a certain margin of the best mix we've identified.
For a fixed number of domains $m$, we run a sweep over swarm sizes $K=c(m+1)$ for linear multiplier $c=1, 2, 3, 4, 5$ (since the minimum number of runs needed for a unique solution to log-linear regression is $m + 1$).
We perform this sweep for $m=6,12,18,24$ domains and $3$ seeds.
For each $m$, we define error as the difference between the BPB of a proposed mix and the best BPB achieved achieved by any mix found in the corresponding sweep for that $m$.
We used 30M proxy models for all runs here but verify that the proposed mixes from these different swarm sizes also exhibit the same trends at the 1B scale (Figure~\ref{fig:swarm_size_1B}).

Figure~\ref{fig:swarm_size} shows that sample complexity is linear in $m$: the error curves across $m$ collapse when plotted against the linear multiplier $c$.
This means that for any target error, there exists a fixed $c$ that achieves it, regardless of $m$.
Therefore, $\mathcal{O}(m)$ runs are sufficient for strong performance---a finding that contradicts prior assumptions of quadratic scaling~\citep{ye2025datamixinglawsoptimizing, ge2025bimixbivariatedatamixing} and provides practitioners with a concrete prescription for allocating compute. 
\textbf{We recommend using at least $K \ge 3(m+1)$ proxy runs with the log-linear regression model}, since $c = 3, 4, 5$ all have close to $0$ error. 

\begin{figure}
    \centering
    \includegraphics[width=0.5\linewidth]{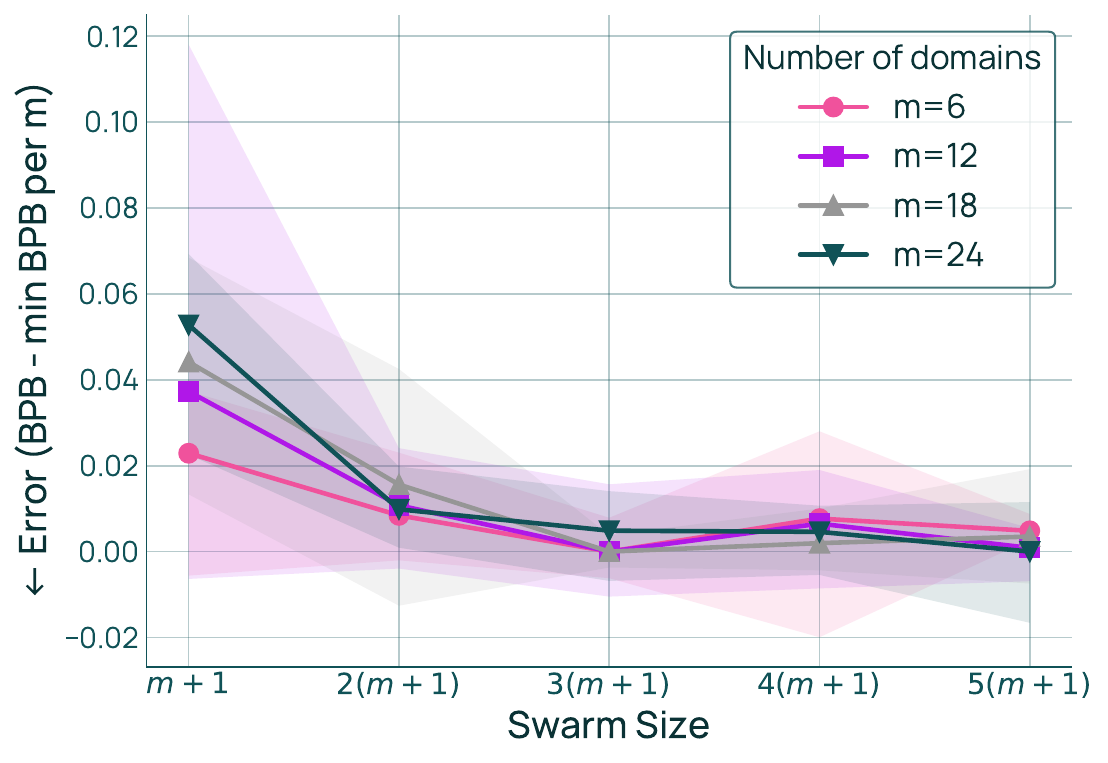}
    \caption{Error vs. Swarm Size. Curves collapse across different $m$, indicating $\mathcal{O}(m)$ runs are needed. Results are averaged across 3 random seeds, and the intervals indicate min and max results.}
    \label{fig:swarm_size}
\end{figure}

\textbf{RQ3: Swarm distribution.} We study the distribution $\mathcal{P}$ from which swarm mixtures should be sampled. The distribution determines how effectively the swarm explores the mixture space. However, existing works rarely study the impact of the swarm distribution, with only CLIMB~\citep{diao2025climbclusteringbasediterativedata} comparing a Dirichlet versus a random uniform distribution.

We investigate two aspects of the swarm distribution. To evaluate each aspect, we measure the average BPB of 1B target models trained on proposed mixes. As an intermediate metric, we also report the regression fit (Pearson correlation between predicted and true per-task BPB) on a held-out set of mixtures. 
\begin{itemize}
    \item \textit{Should each proxy run train on data from all domains (dense), or should some runs focus on subsets of domains (sparse)?} Sparse distributions enable discovering that certain domains should be excluded, while dense distributions ensure all domains are represented in all proxy runs. We generate sparse and dense swarms at the \textbf{topic level} (on DCLM topics) and \textbf{source level} (on the final sources from Table~\ref{tab:domain-updates-full}) to capture both fine and coarse-grained domains.
    
    \item \textit{Does centering the swarm around a promising mixture region improve results?} We test three Dirichlet priors: 1) the natural distribution (based on domain token counts), 2) a strong prior, and 3) a weak prior (see Appendix~\ref{supp:swarm_distribution} for exact constructions).
    
\end{itemize}

Figure~\ref{fig:swarm_distribution} shows that sparse swarms outperform dense at the topic level and vice versa at the source level—both for downstream BPB (left) and regression fit (right).
Neither is universally better; the choice depends on the domain set.
One hypothesis is that the best topic-level mixes exclude certain low-signal topics (e.g., adult content in DCLM), while the best source-level mixes utilize all sources. This aligns with how these domains were constructed: DCLM was partitioned into topics that are potentially uninformative, while sources are intentionally curated. \textbf{The choice of swarm distribution depends on the domain set; we recommend using sparse distributions for topic-level mixing and dense distributions for source-level mixing.}

Table~\ref{tab:dirichlet_priors} shows that centering on the natural distribution and the strong prior results in roughly similar downstream performance while using the weak prior significantly degrades performance. 
\textbf{We recommend that practitioners center the swarm around strong priors when available, but if prior knowledge is uncertain, the natural distribution remains a reasonable fallback.}

\begin{figure}
    \centering
    \includegraphics[width=0.47\linewidth]{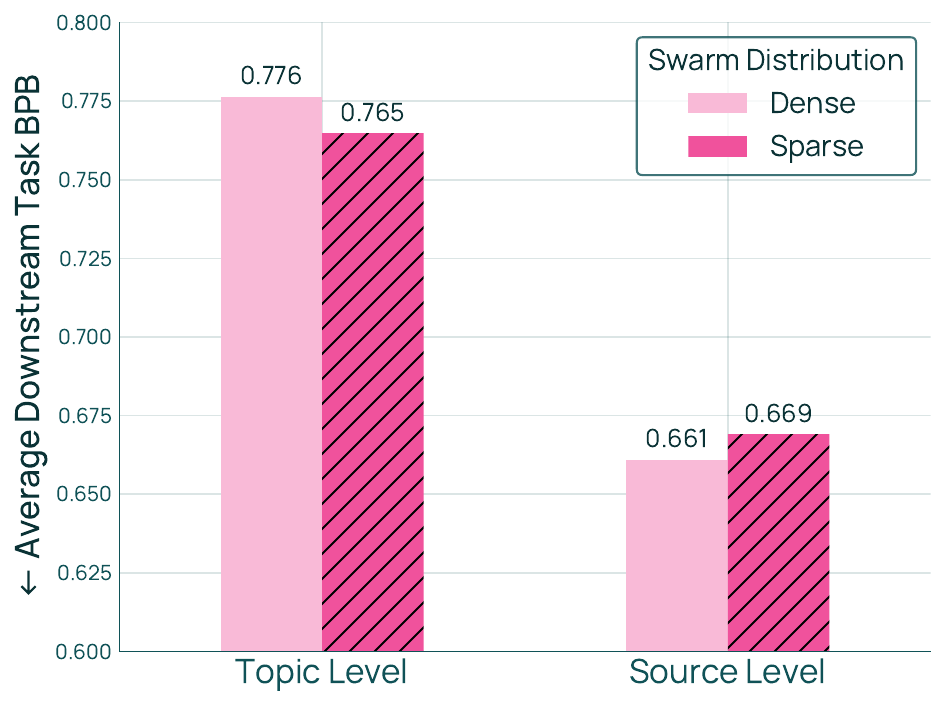}
    \includegraphics[width=0.5\linewidth]{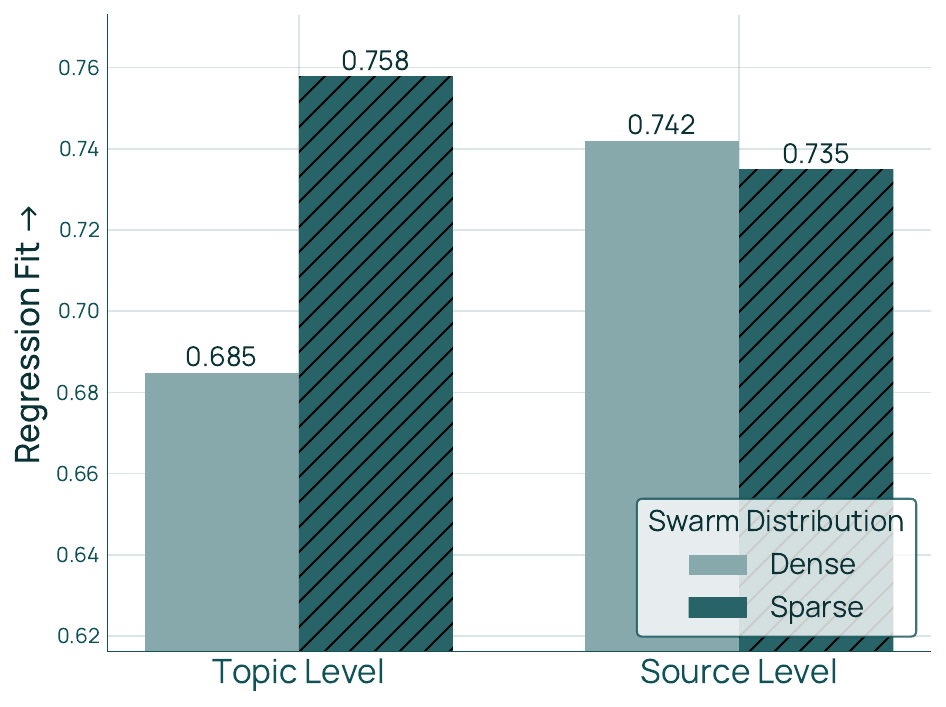}
    \caption{Performance (left) and regression fit (right) of dense versus sparse swarms. While sparse swarms perform better at the topic level, dense swarms perform better at the source level, suggesting that behavior is data-dependent.}
    \label{fig:swarm_distribution}
\end{figure}

\begin{table}[t]
\centering
\caption{Natural and strong priors achieve similar downstream performance, but the weak prior does considerably worse, suggesting that this design choice is fairly robust and that the natural distribution is a reasonable configuration.}
\begin{tabular}{lcc}
\toprule
Dirichlet Prior & Avg Downstream Task BPB $\downarrow$ & Regression Fit $\uparrow$ \\
\midrule
Natural & 0.765 & 0.748 \\
Strong  & 0.763 & 0.835 \\
Weak    & 0.797 & 0.661 \\
\bottomrule
\end{tabular}
\label{tab:dirichlet_priors}
\end{table}

\subsubsection{Regression Model Study}\label{sec:regression}

\textbf{RQ4: Regression model family.}
The regression model is crucial: mixes are optimized directly using its predictions, meaning that regression fit directly controls the quality of the proposed mix.
However, each existing method proposes a different model and reports strong fit, making it unclear which to adopt.

We compare several regression model families derived from existing methods (see Appendix~\ref{supp:regression} for adaptations):
\begin{itemize}[itemsep=0.00pt,topsep=0pt,leftmargin=10pt]
    \item \textbf{Search}: selects the best mixture from the swarm.
    \item \textbf{LightGBM}~\citep{lightgbm}: a gradient boosting regression framework using decision trees. It has been used to learn the regression model in several mixing papers, including RegMix~\citep{liu2025regmixdatamixtureregression, wettig2025organizewebconstructingdomains, diao2025climbclusteringbasediterativedata}.
    \item \textbf{Gaussian Process}: based on~\cite{chen2025admirebayesoptaccelerateddatamixture}, we consider Gaussian Process regression with an RBF kernel scaled by a signal variance term and augmented with independent Gaussian observation noise.
    \item \textbf{BiMix}~\citep{ge2025bimixbivariatedatamixing}: we consider an adaptation of BiMix to our setting that has the following parametric form: $\hat{f}_i(p) = \sum_{j = 1}^m A_{ij} p_j^{-\alpha{ij}}$, where $A_{ij}, \alpha_{ij} \in \R^+$ for all $i, j$. 
    \item \textbf{AutoScale}~\citep{kang2025autoscalescaleawaredatamixing}: we consider an adaptation of Autoscale to our setting that has the following parametric form: $\hat{f}_i(p) = c_i + \sum_{j = 1}^m (R(A_{ij} + p_j))^{-\alpha_{ij}}$, where $c_i, \alpha_{ij} \in \R^+$, $A_{ij} \in [0, 1]$ and $R$ is the number of requested tokens for all $i, j$. 
    \item \textbf{Log-linear}: we consider an adaptation of Data Mixing Laws~\citep{ye2025datamixinglawsoptimizing} to our setting that has the following parametric form: $\hat{f}_i(p) = c_i + \exp(A_i^\top p_i)$, where $c_i \in \R^+$ and $A_i \in \R^m$ for all $i$. $m+1$ proxy models must be trained to obtain a unique solution.
\end{itemize}

First, we measure the regression fit, the Pearson correlation between predicted and true per-task BPB on held-out mixtures, across regression models and swarm sizes $K= 25, 50, 75, 100, 118$. We repeat this across 3 random seeds, subsampling each swarm from a pool of $118$ proxy runs (drawn from a larger swarm of $K=128$ with $10$ held-out mixes).

Figure~\ref{fig:regression_models} (left) shows that swarm size is a key confounding factor in regression fit: different models excel at different swarm sizes, potentially explaining the lack of consensus. For instance, BiMix~\citep{ge2025bimixbivariatedatamixing} performs the best for small swarms ($K=25$), while LightGBM~\citep{liu2025regmixdatamixtureregression} requires more than 118 proxy runs for sufficient fit. Notably, log-linear models~\citep{ye2025datamixinglawsoptimizing} achieve the best regression fit overall ($\rho=80$ at $K=118$) and outperforms other models consistently when  $K \ge 75$ (i.e. $K \ge 3(m+1)$).

Next, we examine the downstream performance of these regression models when $K=128$. Figure~\ref{fig:regression_models} (right) shows that the log-linear regression model's proposed mix achieves the best downstream performance. Combined with the strong regression fit, \textbf{we recommend using log-linear models}, given swarm size $K \ge 3(m+1)$, as recommended in RQ2.

Appendix~\ref{supp:regression} provides additional results at the source level (Figure~\ref{fig:regression_models_olmo3}) and how regression fit varies with swarm size and number of domains (Figure~\ref{fig:k_vs_m_dclm}).

\begin{figure}
    \centering
    \includegraphics[width=0.48\linewidth]{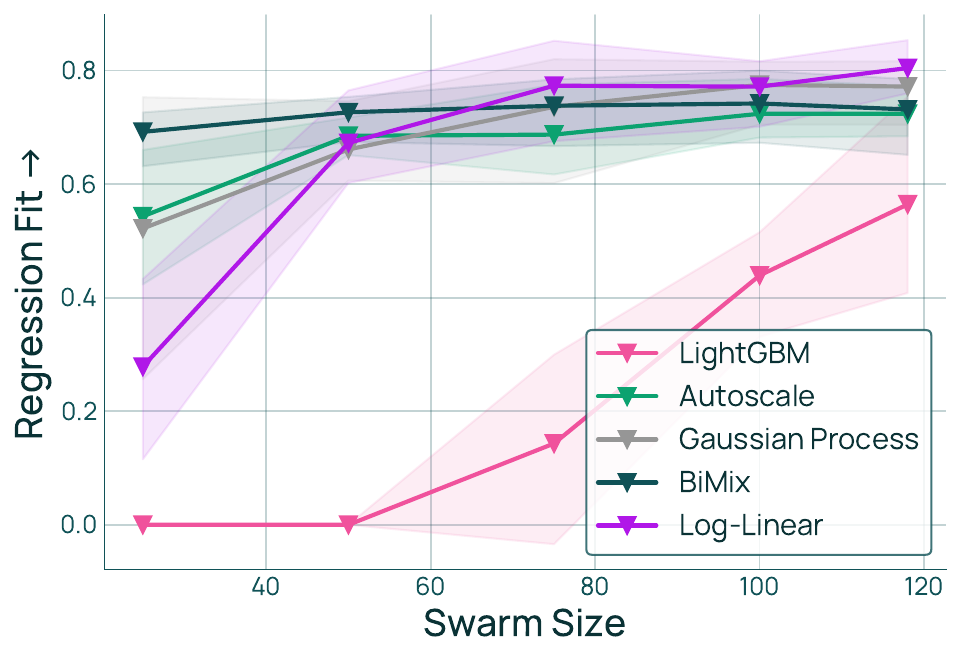}
    \includegraphics[width=0.48\linewidth]{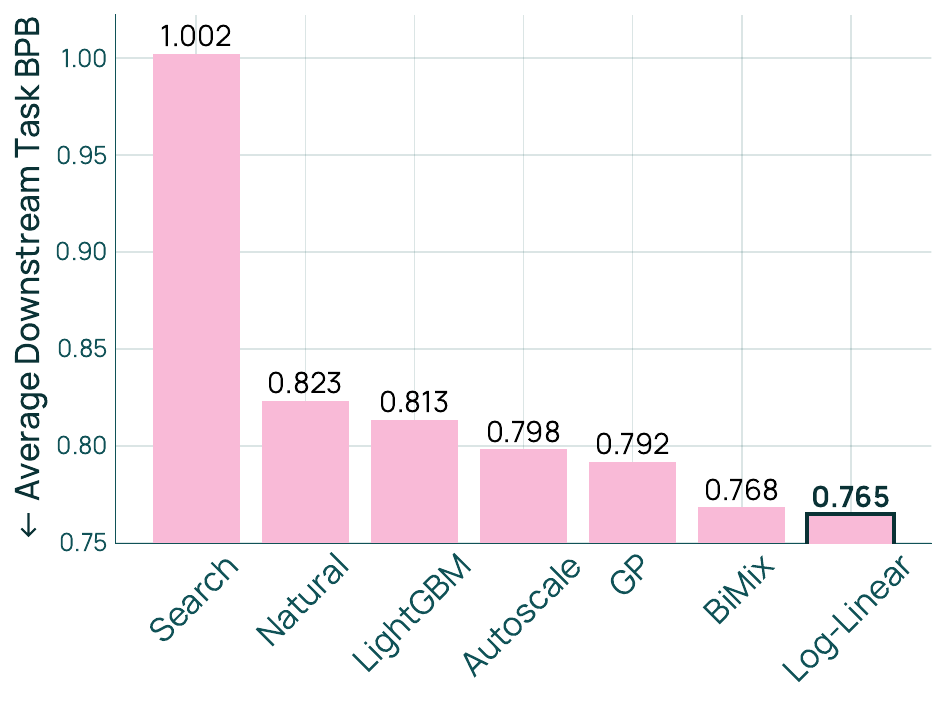}
    \caption{Left: Fit versus swarm size across regression models. Different models require different amounts of data, and the log-linear model achieves the best regression fit. Results are averaged across 3 random seeds, and the intervals indicate min and max results. Right: Downstream performance of regression models for $K=128$. The log-linear model achieves the best average downstream task BPB.}
    \label{fig:regression_models}
\end{figure}

\textbf{RQ5:Regression granularity.} We study the granularity at which to fit regression models in constructing $\hat{f}$: one model per task, one model per task family (e.g., math, code, and QA), or one model for the entire average BPB.
This involves a tradeoff between expressivity and noise. Per-task and per-family regression models can capture distinct relationships between mixtures and various capabilities, increasing expressivity of the overall function approximating average BPB, $\hat{f}$. 
However, they are fit to noisier individual measurements, while directly modeling average BPB may involve less noisy targets.

We compare three approaches. In the \textbf{per-task} approach, we fit individual functions $\hat{f}_i(p)$ using $\{(p^j, y_{ij})\}_{j=1}^K$ for each task $i \in [n]$, and then minimize $\frac{1}{n}\sum_{i = 1}^n \hat{f}_i(p)$. 
In the \textbf{per-family} approach, we group tasks into three families $F_1, F_2, F_3$: math, code, and QA. We fit one model $\bar{f}_i(p)$ to each family's average BPB using $\{(p^j, \frac{1}{|F_i|} \sum_{k \in F_i} y_{kj}) \}_{j = 1}^K$ for $i = 1, 2, 3$, and then minimize $\sum_{i = 1}^3 \frac{|F_i|}{n} \bar{f}_i(p)$.
In the \textbf{aggregated} approach, we fit a single function $\hat{f}_{\text{avg}}(p)$ directly to the average BPB across tasks using $\{(p^j, \frac{1}{n}\sum_{i = 1}^n y_{ij})\}_{j = 1}^K$ and minimize $\hat{f}_{\text{avg}}(p)$.  
We evaluate downstream performance (average BPB of 1B target models) and regression fit (Pearson correlation on held-out mixtures).

Table~\ref{tab:granularity} shows that per-task yields the best downstream performance and regression fit.
Regression fit degrades monotonically as granularity decreases, consistent with reduced expressivity.
While per-family has better regression fit than aggregated, their downstream performance is comparable; this is likely due to noise in transferring from proxy to target models since the performance on 30M models (Figure~\ref{fig:regression_granularity_30m}) exhibits trends that are consistent with the regression fit.
Nevertheless, per-task provides clear benefits on both metrics. \textbf{We recommend fitting separate regression functions per task.}

\begin{table}[t]
\centering
\caption{Performance and regression fit across regression granularities. As granularity of the regression target increases, the regression fit improves, and the downstream performance also trends towards improvement.}
\begin{tabular}{lcc}
\toprule
Method & Avg Downstream Task BPB $\downarrow$ & Regression Fit $\uparrow$ \\
\midrule
\rowcolor{ai2lightpink} Per-Task      & 0.765 & 0.983 \\
Per-Family    & 0.777 & 0.958 \\
Aggregated    & 0.774 & 0.866 \\
\bottomrule
\end{tabular}
\label{tab:granularity}
\end{table}

\subsubsection{Mix Optimization Study}\label{sec:optimization}

\textbf{RQ6: Data repetition constraints.}
In practice, LM training is often data-constrained: the target model’s requested training tokens $R$ exceed the available data. 
In this regime, some mixtures cause excessive sample repetition, degrading performance~\citep{muennighoff2025scalingdataconstrainedlanguagemodels}.
For example, a mixture that proposes 40\% code when code comprises only 5\% of available data would oversample code 8 times. Existing mixing methods assume compute-constrained settings with effectively infinite data and provide no way to control repetition.

We consider repetition constraints of the form $p_j \le \frac{k N_j}{R}$ for all $j \in [m]$, where $k$ is a repetition factor that limits how many times any domain can be sampled given $R$ requested tokens and $N_j$ available tokens in domain $j$.
We study two questions around enforcing these constraints and their impact on proposed mixes.
\begin{itemize}
    \item \textit{How should repetition constraints be enforced while ensuring strong downstream performance?} 
    We consider two places where constraints can be enforced. First, we can constrain the swarm by ensuring that all swarm mixes satisfy the constraint---the regression models then learn what mixes are feasible. Second, we can constrain the optimization problem directly. This creates three approaches: 1) constrained swarm with unconstrained optimization, 2) unconstrained swarm with constrained optimization, and 3) constrained swarm with constrained optimization.     
    We test with $k=4$, verifying whether the proposed mixture satisfies the constraint and reporting downstream performance.
    
    \item \textit{How does the repetition factor affect the proposed mix?} Using the best-performing strategy, we vary $k \in \{2, 3, 4, 5, \infty \}$ and visualize how the proposed mix changes.

\end{itemize}

Table~\ref{tab:enforce_rep} shows that approach 2 (unconstrained swarm, constrained optimization) results in a proposed mix that satisfies the repetition constraints while maintaining performance. Approach 1 fails to satisfy the constraints, repeating samples from one domain 5 times and exceeding $k=4$. Approach 3 satisfies constraints but yields worse downstream performance than approach 2, likely because the constrained swarm provides less coverage of the mixture space. \textbf{We recommend enforcing repetition constraints in the mixture optimization step.}

\begin{table}[t]
\centering
\caption{Performance and repetition constraint satisfaction. Using a constrained swarm with unconstrained optimization does not satisfy a repetition constraint with $k=4$; the proposed mix repeats samples of a domain 5 times. Between approaches 2 and 3, the former approach achieves the best downstream performance.}
\begin{tabular}{lcc}
\toprule
Approach &
Satisfies rep.\ constraint? ($k=4$) &
Avg.\ Task BPB $\downarrow$ \\
\midrule
1) Constrained Swarm, Unconstrained Opt   & No (5)  & 0.774694 \\
\rowcolor{ai2lightpink}
2) Unconstrained Swarm, Constrained Opt  & Yes (4) & 0.764718 \\
3) Constrained Swarm, Constrained Opt    & Yes (4) & 0.785517 \\
\bottomrule
\end{tabular}
\label{tab:enforce_rep}
\end{table}

Figure~\ref{fig:vary_rep_top_7} shows how the proposed mix changes with varying $k$ when we add a repetition constraint to the optimization problem.. As the constraint relaxes from $2$ to $\infty$ (the infinite data setting), the mixture weights on high-utility domains, such as software development, monotonically increase. In contrast, domains such as literature smoothly decrease as $k$ increases, suggesting that their larger allocations under tight constraints primarily compensate for limited availability of higher utility domains. This demonstrates that repetition constraints significantly shape the proposed mix.

\begin{figure}
    \centering
    \includegraphics[width=0.5\linewidth]{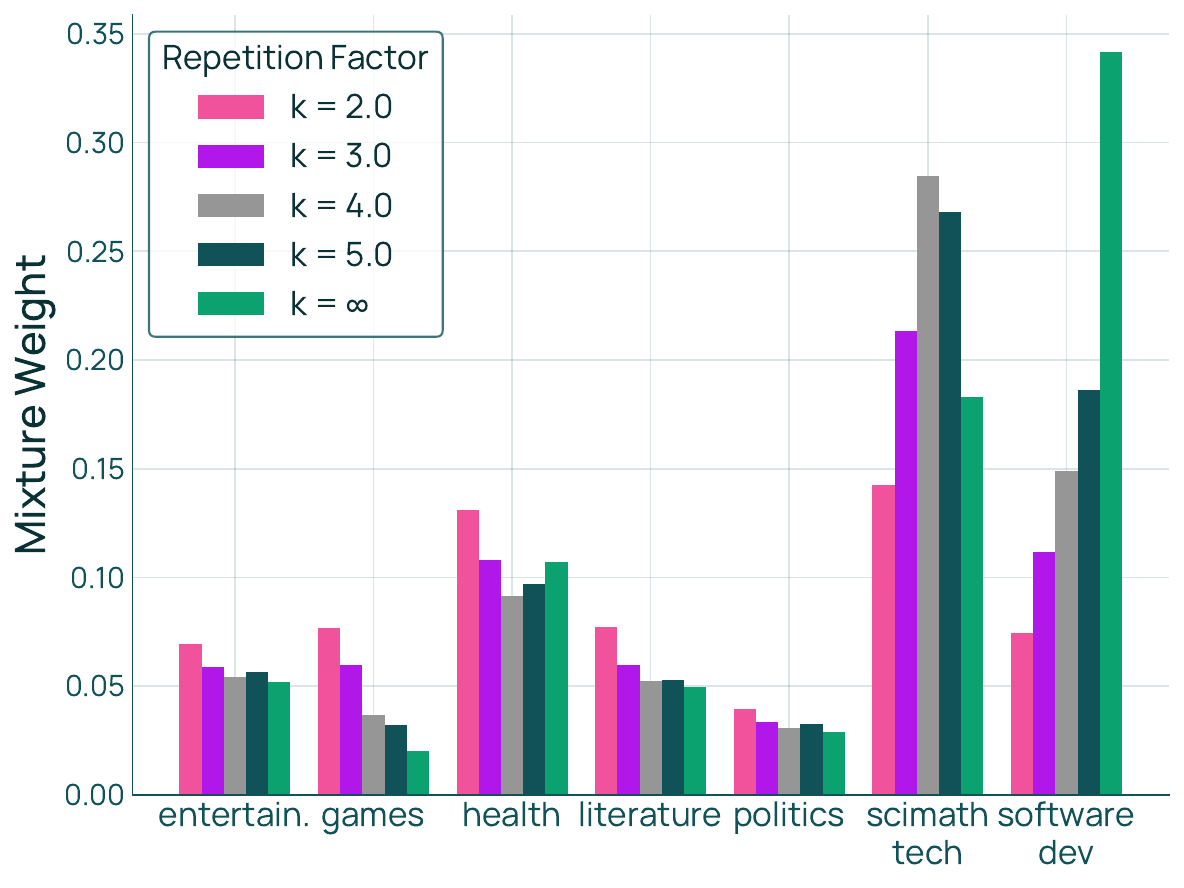}
    \caption{Proposed Mixture Weights versus Repetition Factor. Enforcing varying caps on repetition $k$ leads to significantly different proposed mixes. See Figure~\ref{fig:vary_repetition_factor} for mixture weights on the full set of domains.}
    \label{fig:vary_rep_top_7}
\end{figure}

\textbf{RQ7: Optimization solver.} We study how to solve for the mixture that minimizes the average predicted BPB $\hat{f}(p)$. Under log-linear regression, $\hat{f}(p)$ is convex, enabling exact solvers. However, since regression functions imperfectly predict performance (Figure~\ref{fig:regression_models} left), it is unclear whether exact optimization of the surrogate objective yields the mix with the best downstream performance or if incorporating some regularization may be better.

We consider three solving strategies. The \textbf{exact solver} uses CVXPY to compute the optimal mix. The \textbf{search}-based solver, used in RegMix~\citep{liu2025regmixdatamixtureregression}, samples candidate mixes from a Dirichlet distribution with a natural prior (i.e., a mix proportional to the domain sizes). The \textbf{exact solver with KL regularization} modifies the objective to $\text{minimize}_{p \in \triangle^{m - 1}} \hat{f}(p) + \lambda \kl(p || p_0)$, where $\lambda > 0$ controls the strength of regularization towards the natural distribution $p_0$; we consider $\lambda \in \{0.01, 0.05\}$. We measure the average BPB of 1B target models trained on the proposed mixes. We also examine the optimal objective value, the predicted average BPB at the 30M proxy model scale, as a sanity check that the exact solver should obtain the lowest predicted BPB.

Figure~\ref{fig:solvers} (left) shows that the exact solver with $\lambda=0.05$ achieves the best performance. In contrast, the right panel confirms that the exact solver obtains the lowest predicted BPB, while adding a KL penalty degrades it. Taken together, these results indicate that although the average predicted BPB is optimized most effectively by the exact solver, noise from regression fitting and proxy-to-target model transfer makes moderate regularization beneficial. The search baseline achieves poor performance, consistent with it being a heuristic. \textbf{We recommend using an exact solver with a KL regularization of 0.05 to solve the mixture optimization problem}. See Appendix~\ref{supp:solver}, Figure~\ref{fig:solvers_olmo3} for additional source-level results.

\begin{figure}
    \centering
    \includegraphics[width=0.48\linewidth]{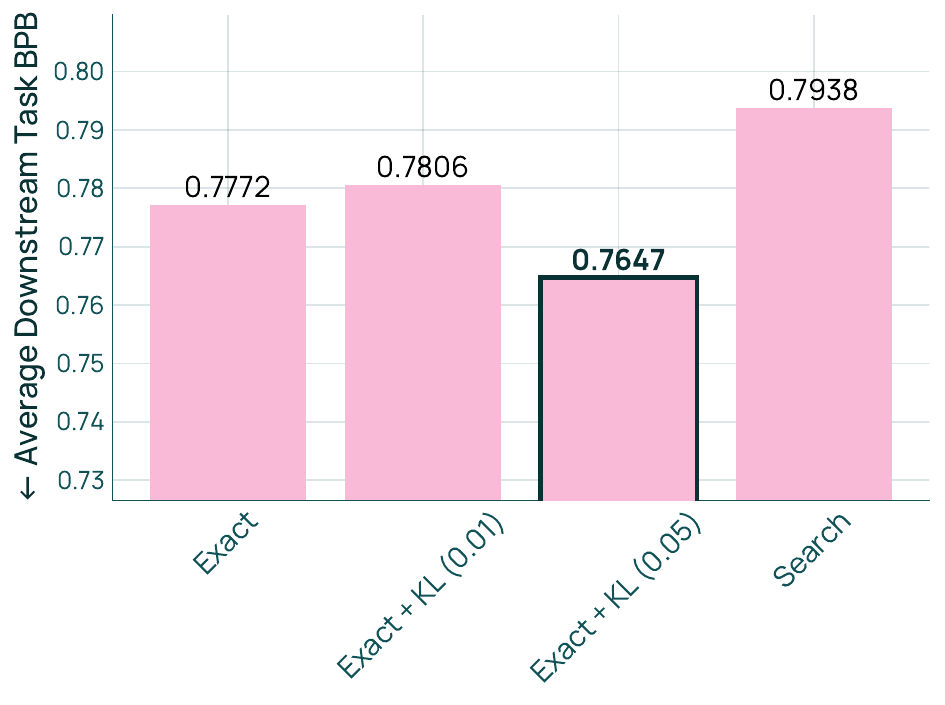}
    \includegraphics[width=0.5\linewidth]{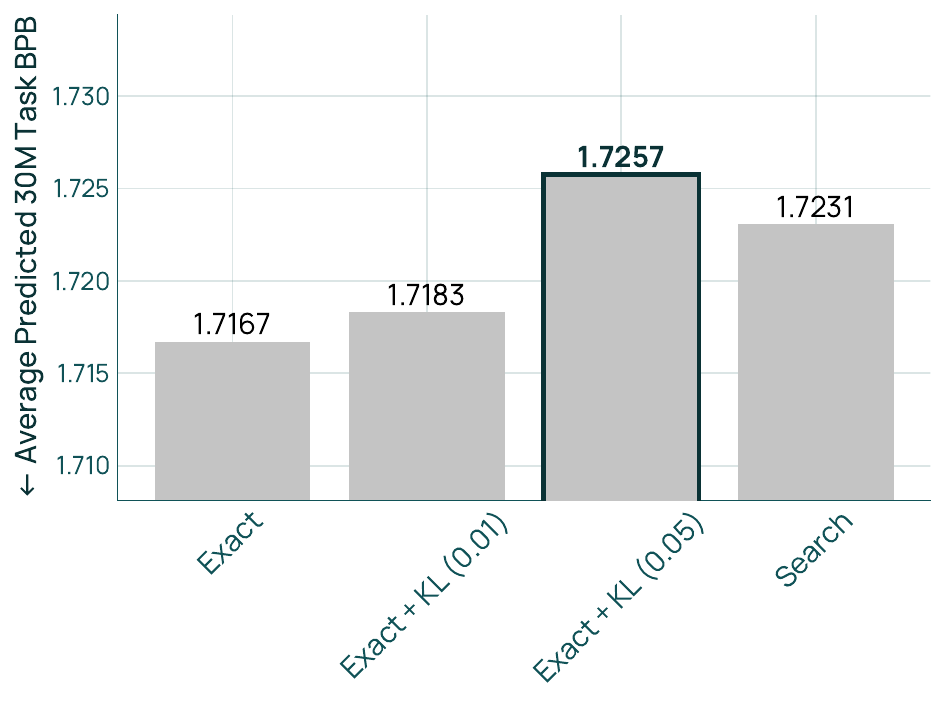}
    \caption{Downstream performance (left) and 30M predicted performance (right) across optimization solvers. The predicted performances of various solvers confirms that the exact solver minimizes the predicted BPB, as expected, followed by KL penalties of $0.01$ and $0.05$. However, the exact solver does not obtain the best downstream performance, and instead some KL regularization helps.}
    \label{fig:solvers}
\end{figure}

\subsection{Configuring \base}\label{sec:base_method}

Our results in \S\ref{sec:study} let us define \base, a concrete configuration of the offline mixing schema that we use throughout Olmo 3 development. The rightmost column of Table~\ref{tab:mixing_design_choices} summarizes the full configuration, and Algorithm~\ref{alg:base_method} formalizes the procedure (\textcolor{olmoBlue}{blue} highlights steps informed by our findings).

\begin{center}
\begin{tikzpicture}
\node[inner sep=0pt] (alg) {
\begin{minipage}{0.9\textwidth}
\begin{algorithm}[H]
   \caption{\base}
   \label{alg:base_method}
\begin{algorithmic}[1]
   \STATE {\bfseries Input:} Domains $\D = \{D_1, \dots, D_m\}$ of sizes $\{N_1, \dots, N_m\}$, swarm size \textcolor{olmoBlue}{$K = \mathcal{O}(m)$}, repetition factor $k$, requested tokens $R$, KL penalty $\lambda$, natural distribution $p_0 \propto \{N_1, \dots, N_m \}$.  
    \STATE Sample mixes $p^1, \dots, p^K \in \triangle^{m-1}$ on $\D$.
    \STATE Train proxy models \textcolor{olmoBlue}{$\Ssmall\ge 15$M parameters} on these mixes, and evaluate on downstream tasks to get a dataset of mixes and performance, $\{(p^j, \{y_{ij}\}_{i = 1}^n \}_{j =1}^K$, where $y_{ij}:= f_i(\LM(\Ssmall, \Rsmall, p^j))$.
    \FOR{$i \in [n]$}
        \STATE Use $\{(p^j, y_{ij})\}_{j = 1}^K$ to fit the \textcolor{olmoBlue}{log-linear model} $\hat{f}_i(p) \texttt{=} c_i + \exp(A_i^\top p)$, where $c_i \in \R^+$ and $A_i \in \R^{m}$.
    \ENDFOR
    \STATE Solve the optimization problem to get $p^\star$:
    \begin{align*}
        &\text{minimize}_{p \in \triangle^{m-1}} \frac{1}{n} \sum_{i=1}^n \hat{f}_i(p) + \textcolor{olmoBlue}{\lambda \kl(p || p_0)} \qquad \text{subject to} \quad \textcolor{olmoBlue}{p_j \leq \frac{k N_j}{R} \; \forall j \in [m]} 
    \end{align*}
   \STATE \textbf{Return} $p^\star$.
\end{algorithmic}
\end{algorithm}
\end{minipage}
};

\draw [decorate,decoration={brace,amplitude=5pt}]
    ([yshift=-2.0cm,xshift=5pt]alg.north east) -- ([yshift=-3.6cm,xshift=5pt]alg.north east)
    node[midway,right=10pt,rotate=-90,anchor=south] {\footnotesize Swarm};

\draw [decorate,decoration={brace,amplitude=5pt}]
    ([yshift=-3.7cm,xshift=5pt]alg.north east) -- ([yshift=-4.9cm,xshift=5pt]alg.north east)
    node[midway,right=8pt,rotate=-90,anchor=south] {\footnotesize Regression};

\draw [decorate,decoration={brace,amplitude=5pt}]
    ([yshift=-5.0cm,xshift=5pt]alg.north east) -- ([yshift=-7.0cm,xshift=5pt]alg.north east)
    node[midway,right=8pt,rotate=-90,anchor=south] {\footnotesize Optimization};
\end{tikzpicture}
\end{center}

\section{The Evolving Domain Problem}\label{sec:conditional_mixing}

From \S\ref{sec:config-choices}, \base assumes a fixed domain set $\D$. Here, we study the \emph{evolving domain problem} encountered throughout LM development.
We first formalize the problem of recomputing mixes after the domain set is updated (\S\ref{sec:problem-setup-evolving-domain}). We then present \fullmixreuse, our primary mechanism that reuses ratios on all the domains unaffected by the update (\S\ref{sec:mix_reuse_42}). We theoretically analyze when \fullmixreuse performs well (\S\ref{sec:analysis}). Finally, we introduce \partialmixreuse as an extension that reuses the ratios over some unaffected domains (\S\ref{sec:partial-reuse}). An overview of these strategies is provided in Figure~\ref{fig:mixreuse}.

\subsection{Problem Setup}
\label{sec:problem-setup-evolving-domain}

During LM development, the domain set evolves from $\D$ to an updated set $\D' = \{D_1', \dots, D_{m'}'\}$ of size $m'$, where each domain $D_i'$ now has $N_i'$ tokens. We use $q \in \triangle^{m'-1}$ to express data mixtures over $\D'$.
In practice, domain updates are typically localized, affecting only a subset of domains while leaving others unchanged.
To capture this structure, we partition the domain sets as $\D = [\D_1, \D_2]$ and $\D' = [\D_1, \D_2']$, where $\D_1$ is the \textbf{unaffected domain set}, and $\D_2$ is the \textbf{affected domain set} that is transformed into $\D_2'$.

In developing Olmo 1--3~\citep{olmo2025olmo3} and observing similar efforts like SmolLM 1--3~\citep{bakouch2025smollm3}, we identified four patterns of domain updates, which we call \textbf{domain update operators}, that describe how $\D_2$
transforms into $\D_2'$ (Figure~\ref{fig:banner}):
\begin{itemize}
    \item \Add: $\D_2 = \emptyset$ and $\D_2'$ contains the newly added domains. This occurs when new datasets are created.
    \item \Remove: $\D_2$ contains one or more domains and $\D_2' = \emptyset$. This can occur, for instance, when domains are discarded due to low utility.
    \item \Partition: $\D_2$ consists of one domain that is split into multiple subdomains in $\D_2'$ such that $\D_2 = \bigcup_{D_i' \in \D_2'} D_i'$. Partitioning is commonly used to obtain finer-grained mixtures, which can improve performance~\citep{wettig2025organizewebconstructingdomains, diao2025climbclusteringbasediterativedata,ge2025rbdomainregroupingdata, peng2025topicsourcekeyeffective}.
    \item \Revise: $\D_2$ is one domain that is modified to produce the corresponding domain in $\D_2'$. This occurs when contents of the dataset are modified, such as samples being reformatted or rewritten~\citep{kimiteam2025kimik2openagentic, maini2024rephrasingwebrecipecompute}.
\end{itemize}

\textbf{Goal.} We have an existing mixture $\tilde{p} \in \triangle^{m-1}$ over the current domain set $\D$ proposed through a procedure like \base. 
But now, the domain set has evolved from $\D$ to $\D'$, and our goal is to use $\tilde{p}$ to compute an updated optimal $q^\star \in \triangle^{m' - 1}$ on $\D'$ that solves
\begin{align}
    &\text{minimize}_{q \in \triangle^{m'-1}}  \frac{1}{n} \sum_{i = 1}^n f_i(\LM(S, R, q))
    \label{eq:mixing_problem} \\
    &\text{s.t.} \quad  q_j \le \frac{k N_j'}{R} \quad \forall j \in [m'] \nonumber
\end{align}

This objective is similar to the one in \S\ref{sec:setup_1}, except that it is defined over $q$, and we explicitly enforce the repetition constraint from \S\ref{sec:optimization} to ensure that samples are not repeated more than $k$ times.

\textbf{Baseline: full recomputation.} One way to solve~\eqref{eq:mixing_problem} is to apply \base (Algorithm~\ref{alg:base_method}) to $\D'$ after each update, which requires a swarm of $\mathcal{O}(m')$ proxy runs (\S\ref{sec:swarm}). The costs of full recomputation thus accumulate rapidly with the number of updates. We next study whether mixture reuse can produce mixes at a lower cost.

\subsection{Mixture Reuse}\label{sec:mix_reuse_42}

The core idea in mixture reuse is to freeze the relative ratios among the unaffected domains according to $\tilde{p}$, aggregate them into a single \emph{virtual domain}, and only recompute its total weight along with the weights of the affected domains $\mathcal{D}_2'$. This reduces the dimensionality of the optimization from $m'$ to $1 + |\mathcal{D}_2'|$. This mechanism is presented in Figure~\ref{fig:mixreuse} (right).

For example, suppose $\mathcal{D}$ contains $m=3$ domains with mixture
$\tilde{p} = [0.25,\,0.25,\,0.5]$, and $\mathcal{D}'$ is formed by adding
one new domain.
Instead of learning a $4$-dimensional mixture, we learn a $2$-dimensional
mixture over the virtual domain and the new domain.
If that mixture is $[0.4,\,0.6]$,  we can expand it using $\tilde{p}$ to induce the final mixture over $m'=4$ domains:
\[
q=\underbrace{0.4}_{\substack{\text{recomputed}\\\text{virtual domain}\\\text{ratio}}} \cdot \underbrace{[0.25,\,0.25,\,0.5]}_{\substack{\text{reused}\\\text{ratios}}} \;\cup\; \underbrace{[0.6]}_{\substack{\text{recomputed}\\\text{affected domain}\\\text{ratio}}}\;=\; [0.1,\,0.1,\,0.2,\,0.6].
\]
See Appendix~\ref{supp:examples} for examples of \Remove, \Partition, and \Revise.

\begin{figure}
    \centering
    \includegraphics[width=0.8\linewidth]{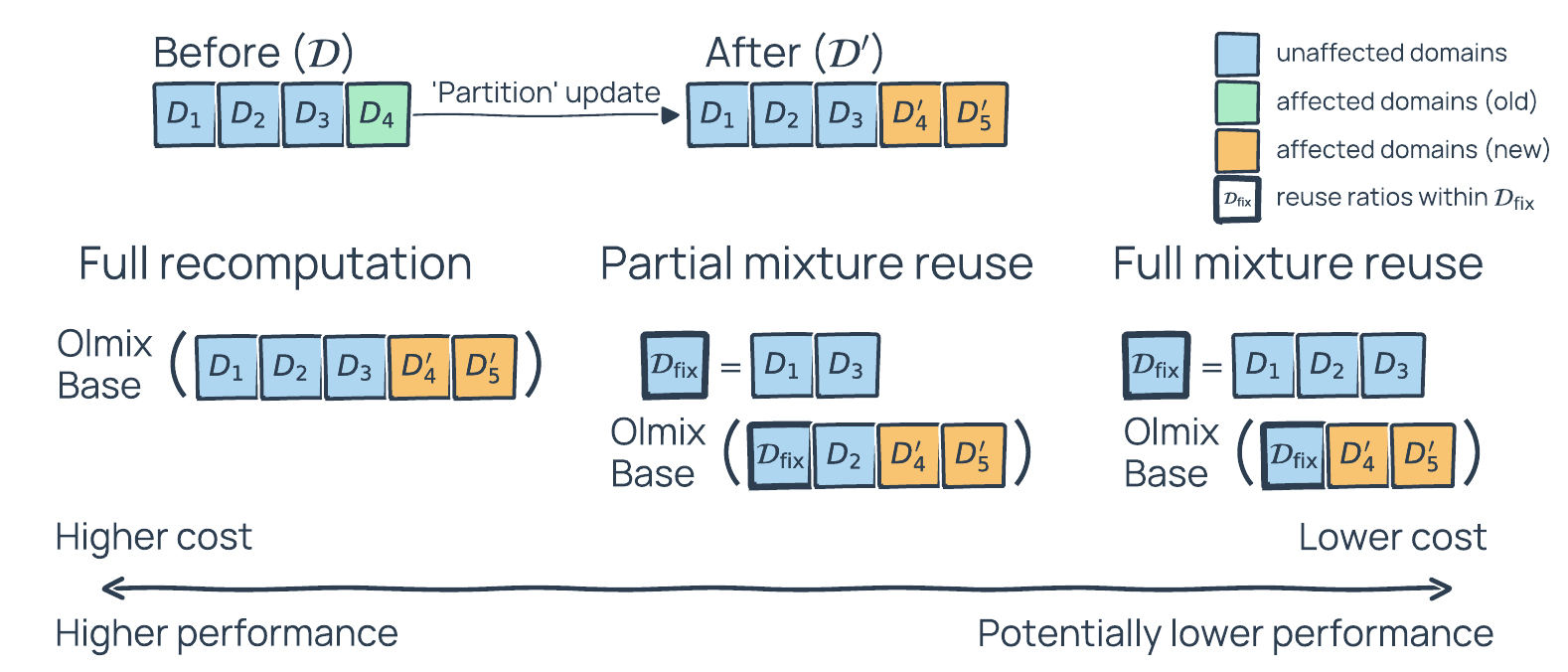}
    \caption{Overview of Mixture Recomputation Strategies. When the domain set is updated, full recomputation applies \base on the entire $\D'$, while \fullmixreuse (\S\ref{sec:mix_reuse_42}) and \partialmixreuse (\S\ref{sec:partial-reuse}) freeze the ratios of the domains in $\Dfix$, resulting in less recomputation. The strategies provide a spectrum in terms of cost (number of swam runs) and downstream performance.}
    \label{fig:mixreuse}
\end{figure}

\subsubsection{Formalizing the Mixture Reuse Problem}

The example above provides the intuition behind mixture reuse, which we formalize in the mixture reuse problem. First, we introduce notation to handle domains we will reuse versus recompute. Define $\Dfix := \D_1$ as the unaffected domains whose relative ratios we will keep fixed from $\tilde{p}$, and $\Dcomp := \D_2'$ as the affected domains whose ratios we recompute. We remap these domain sets because, as seen later in \S\ref{sec:partial-reuse}, we may choose to recompute unaffected domains.

We partition any mix $q$ on $\D'$ as $q = [\rho q_{\Dfix}, (1 - \rho) q_{\Dcomp}]$, where $\rho \in [0, 1]$ is the total weight on the unaffected domains, and $q_{\Dfix} \in \triangle^{|\Dfix| - 1}$ and $q_{\Dcomp} \in \triangle^{|\Dcomp| - 1}$ are the relative mixes within $\Dfix$ and $\Dcomp$, respectively. 
Similarly, we define $\tilde{p}_{\Dfix} \in \triangle^{|\Dfix| - 1}$ as the normalized ratios from $\tilde{p}$ restricted to $\Dfix$.

With this notation, the mixture reuse problem is the following approximation of~\eqref{eq:mixing_problem}:
\begin{align}
    &\text{minimize}_{q \in \triangle^{m'-1}}  \frac{1}{n} \sum_{i = 1}^n f_i(\LM(S, R, q))
    \label{eq:conditional_mixing_problem} \\
    &\text{s.t.} \quad  q_j \le \frac{k N_j'}{R} \quad \forall j \in [m'] \nonumber \\
    &  \qquad \; \textcolor{olmoBlue}{q_{\Dfix} = \tilde{p}_{\Dfix}} \nonumber 
\end{align}

The new constraint $\textcolor{olmoBlue}{q_{\Dfix} = \tilde{p}_{\Dfix}}$ enforces that the relative ratios among $\Dfix$ are the same as in $\tilde{p}$; we optimize over a simpler $q$ of the form $[\rho \textcolor{olmoBlue}{\tilde{p}_{\Dfix}}, (1 - \rho) q_{\Dcomp}]$.

\subsubsection{\fullmixreuse}

We now describe \fullmixreuse, our approach for solving the mixture reuse problem~\eqref{eq:conditional_mixing_problem}.

We use a change of variables that absorbs the constraint $q_{\Dfix} = \tilde{p}_{\Dfix}$, transforming~\eqref{eq:conditional_mixing_problem} into an unconstrained optimization problem.
For any feasible $q = [\rho \tilde{p}_{\Dfix}, (1 - \rho) q_{\Dcomp}]$, we aggregate all domains in $\Dfix$ into a single virtual domain and define the collapsed mixture $r = [\rho, (1-\rho)q_{\Dcomp}] \in \triangle^{|\Dcomp|}$. 
We index the elements of $r$ by $\{v\} \cup \Dcomp$, where $r_v = \rho$ denotes the virtual domain weight (the total weight on $\Dfix$) and $r_j$ for $j \in \Dcomp$ are the weights on affected domains.
Given $\tilde{p}_{\Dfix}$, any collapsed mixture $r$ can be expanded into a feasible $q$ over $\D'$ via the expansion function $\Phi_{\tilde{p}_{\Dfix}}(r)$, where $q_j = r_v \cdot \tilde{p}_j$ for $j \in \Dfix$, and $q_j = r_j$ for $j \in \Dcomp$.

From our earlier example, $r = [0.4, 0.6]$ (i.e., $\rho = 0.4$ and $q_{\Dcomp} = [1]$) with $\tilde{p}_{\Dfix} = [0.25, 0.25, 0.5]$ expands into $q = \Phi_{\tilde{p}_{\Dfix}}(r) = [0.1, 0.1, 0.2, 0.6]$.

This change of variables transforms the mixture reuse problem into standard mixing on $r$ over $1 + |\Dcomp|$ domains (Lemma~\ref{lem:collapsed_space}). We can now apply \base (Algorithm~\ref{alg:base_method}) in the collapsed space, which requires only $\mathcal{O}(|\Dcomp|)$ proxy runs instead of $\mathcal{O}(m')$. 
The result is a proposed collapsed mixture $r^\star(\tilde{p}_{\Dfix})$, which we expand back to the full domain space as $q^\star(\tilde{p}_{\Dfix}) = \Phi_{\tilde{p}_{\Dfix}}(r^\star(\tilde{p}_{\Dfix}))$. This expanded mixture is our final proposed mix over $\D'$. The full procedure is in Algorithm~\ref{alg:conditional_mixing_method}, with \textcolor{olmoBlue}{blue} text indicating changes from \base in Algorithm~\ref{alg:base_method}.

\begin{figure}[htb]
\centering
\begin{tikzpicture}
\node[inner sep=0pt] (alg) {
\begin{minipage}{0.9\textwidth}
\begin{algorithm}[H]
   \caption{\fullmixreuse}
   \label{alg:conditional_mixing_method}
\begin{algorithmic}[1]
   \STATE {\bfseries Input:} \textcolor{olmoBlue}{Domain set $\D' = \Dfix \cup \Dcomp$ of sizes $\{N_1', \dots, N_{m'}' \}$, existing mix $\tilde{p} \in \triangle^{m - 1}$}, swarm size $K = \mathcal{O}(m)$, repetition factor $k$, requested tokens $R$, KL penalty $\lambda$, natural distribution $r_0 \in \triangle^{|\Dcomp|}$.
    \STATE Sample mixes $r^1, \dots, r^K \in \triangle^{|\Dcomp|}$ and \textcolor{olmoBlue}{expand each collapsed mix $r^j$ into $q^j := \Phi_{\tilde{p}_{\Dfix}}(r^j)$}.
    \STATE Train proxy models on the expanded mixes and evaluate on downstream tasks to get a dataset of mixes and performance, $\{(r^j, \{y_{ij}\}_{i = 1}^n \}_{j =1}^K$, where $y_{ij}:= f_i(\LM(\Ssmall, \Rsmall, q^j))$.
    \FOR{$i \in [n]$}
        \STATE Use $\{(r^j, y_{ij})\}_{j = 1}^K$ to fit the log-linear model $\hat{g}_i(r) = d_i + \exp(B_i^\top r)$, where $d_i \in \R^+$ and $B_i \in \R^{|\Dcomp|}$.
    \ENDFOR
    \STATE Solve the following optimization problem to get $r^\star(\tilde{p}_{\Dfix})$:
    \begin{align}
    &\text{minimize}_{r \in \triangle^{|\Dcomp|}} \frac{1}{n} \sum_{i=1}^n \hat{g}_i(r) + \lambda \kl(r || r_0) \label{eq:alg_mixing_problem_2} \nonumber \\ & \text{subject to} \quad r_v \leq \min_{j \in \Dfix}\bigg\{\frac{k N_j'}{R \tilde{p}_{j}} \bigg\}, \;\; r_j \le  \frac{k N_j'}{R} \quad \forall j \in \Dcomp \nonumber 
    \end{align}
   \STATE \textbf{Return} \textcolor{olmoBlue}{$q^\star(\tilde{p}_{\Dfix}) := \Phi_{\tilde{p}_{\Dfix}}(r^\star(\tilde{p}_{\Dfix}))$, the expanded form of $r^\star(\tilde{p}_{\Dfix})$}.
\end{algorithmic}
\end{algorithm}
\end{minipage}
};

\draw [decorate,decoration={brace,amplitude=5pt}]
    ([yshift=-2.5cm,xshift=5pt]alg.north east) -- ([yshift=-3.8cm,xshift=5pt]alg.north east)
    node[midway,right=10pt,rotate=-90,anchor=south] {\footnotesize Swarm};
    
\draw [decorate,decoration={brace,amplitude=5pt}]
    ([yshift=-3.9cm,xshift=5pt]alg.north east) -- ([yshift=-5.4cm,xshift=5pt]alg.north east)
    node[midway,right=8pt,rotate=-90,anchor=south] {\footnotesize Regression};
    
\draw [decorate,decoration={brace,amplitude=5pt}]
    ([yshift=-5.5cm,xshift=5pt]alg.north east) -- ([yshift=-8.5cm,xshift=5pt]alg.north east)
    node[midway,right=8pt,rotate=-90,anchor=south] {\footnotesize Optimization};
\end{tikzpicture}
\end{figure}

\subsection{Theoretical Analysis}\label{sec:analysis}

\fullmixreuse reduces costs by reusing existing ratios, but when does it match versus degrade performance compared to full recomputation? We theoretically analyze this, finding that performance depends on (1) how much the optimal mix changes due to the domain update, and (2) a coupling effect between $\Dfix$ and $\Dcomp$.
Importantly, we empirically validate our theory in \S\ref{sec:validation}: \textbf{we measure the terms in our bounds and show they tightly track actual performance gaps across settings.}

\textbf{Assumptions.} We make the following assumptions.
\begin{enumerate}[itemsep=0pt,topsep=2pt,leftmargin=12pt]
    \item Log linear model holds: For each task $i\in[n]$, performance can be expressed as $f_i(\LM(S,R,q)) = c_i + \exp(A_i^\top q)$ for some $c_i\in\R^+,\; A_i\in\R^{m'}$.
    \item Full recomputation minimizes~\eqref{eq:mixing_problem}, and \fullmixreuse minimizes~\eqref{eq:conditional_mixing_problem}. 
\end{enumerate}

\textbf{Definitions.}
Define performance of $q$ as $F(q) := \frac{1}{n}\sum_{i=1}^n f_i(\LM(S,R,q))$. $q^\star$ is the solution to~\eqref{eq:mixing_problem}, and $q^\star(\tilde p_{\Dfix})$ is the solution to~\eqref{eq:conditional_mixing_problem}. We study the \emph{performance gap} of \fullmixreuse, $F(q^\star(\tilde p_{\Dfix})) \texttt{-} F(q^\star)$.

Let $q^\star_{\Dfix}$ denote the optimal mix on $\Dfix$ after the domain set update, such that $q^\star = [\rho^\star q^\star_{\Dfix},\; (1-\rho^\star)q^\star_{\Dcomp}]$. Let $\kappa(\Dfix, \Dcomp)$ denote a coupling term (defined in Appendix~\ref{supp:defs}) that captures task influence of $\Dfix$ versus $\Dcomp$.

\subsubsection{Performance Characterization}

Our first result characterizes the performance gap in terms of $\tilde{p}_{\Dfix}$ itself. See Appendix~\ref{supp:general_gap_proof} for the proof.

\begin{theorem}[Performance gap bound] \label{thm:general_gap} There exists a finite $C_1 > 0$ such that the performance gap is bounded by
\begin{align*} 
&F(q^\star(\tilde{p}_{\Dfix})) - F(q^\star) \le C_1 \kappa(\Dfix, \Dcomp) \|\tilde{p}_{\Dfix} - q^\star_{\Dfix}\|.
\end{align*} 
\end{theorem}

The performance gap is controlled by two factors: \begin{itemize}[itemsep=0pt,topsep=2pt,leftmargin=12pt] 
\item $\|\tilde{p}_{\Dfix} - q^\star_{\Dfix} \|$ (``reuse gap''): how close the mix we reuse is to the optimal mix after the update. 
\item $\kappa(\Dfix, \Dcomp)$: this coupling term is large when $\Dfix$ and $\Dcomp$ impact the same set of downstream tasks. It controls the rate at which performance depends on the reuse gap. 
\end{itemize}

When both terms are small, \fullmixreuse matches full recomputation. Our empirical validation (Figure~\ref{fig:theorem1_validation}) confirms that the reuse gap strongly predicts actual performance gaps across different scenarios.

\subsubsection{Reuse Gap}

Theorem~\ref{thm:general_gap} shows that the performance gap is governed by the reuse gap, but when is the reuse gap small? 
We analyze the case where $\tilde{p}$ is assumed to be optimal before the update, and the domain set is modified via an \Add update.
Proofs and results on other operators are in Appendix~\ref{supp:add_domains_proof}.

\begin{theorem}\label{thm:add_domains}
Define $\tilde{p}$ as the solution to~\eqref{eq:mixing_problem} on $\D$ and suppose that new domains are added. There exists a finite $C_2 > 0$ such that the reuse gap is bounded by
    \begin{align*}
        \|\tilde{p}_{\Dfix} - q^\star_{\Dfix}\| \le C_2 \kappa(\Dfix, \Dcomp) (1 - \rho^\star).
    \end{align*}
\end{theorem}

The reuse gap is controlled by two factors:
\begin{itemize}[itemsep=0pt,topsep=2pt,leftmargin=12pt]
    \item $\kappa(\Dfix, \Dcomp)$: the coupling term, also in Theorem~\ref{thm:general_gap}.
    \item $1 - \rho^\star$: this term captures how much newly added domains move the optimum, since $1$ is effectively $\rho$ induced by $\tilde{p}$ before domains are added. This term can be small when: 1) the added domains are low utility, or 2) the added domains have little data, so the repetition constraint caps their maximum weight.
\end{itemize}    

Our empirical validation (Figures \ref{fig:theorem2_validation}- \ref{fig:theorem2_validation_coupling}) shows that the success of \fullmixreuse when adding domains correlates with small $1-\rho^\star$ and low coupling.

\subsection{\partialmixreuse}
\label{sec:partial-reuse}

Our analysis (\S\ref{sec:analysis}) shows that \fullmixreuse may underperform full recomputation when coupling between $\Dfix$ and $\Dcomp$ is high. We propose \partialmixreuse (Figure~\ref{fig:mixreuse} center) as a middle ground: rather than reusing all unaffected domains, it selectively recomputes some, reducing coupling while keeping costs below full recomputation.

\textbf{Method} Let $\Dpartial \subset \D_1$ be a subset of unaffected domains we reuse. We redefine the domains we reuse versus recompute as $\Dfix := \Dpartial$ and $\Dcomp := (\D_1 \setminus \Dpartial) \cup \D_2'$, and then apply \fullmixreuse (Algorithm~\ref{alg:conditional_mixing_method}) with this new partition. The mixing problem has dimension $1 + |\D_2'| + |\D_1 \setminus \Dpartial|$, interpolating between \fullmixreuse's $1 + |\D_2'|$ and full recomputation's $m'$. Since swarm cost scales linearly with dimension (RQ2), this directly interpolates computational cost.

\textbf{Interpretation and Empirical Validation} Carefully choosing $\Dpartial$ can reduce the coupling $\kappa(\Dfix, \Dcomp)$ from Theorem~\ref{thm:general_gap}. 
An intuitive example: when adding code data to web topics, the software development web topic should be recomputed alongside it since both strongly influence code evaluation tasks. Figure~\ref{fig:theorem2_validation_coupling} confirms that this reduces coupling and improves performance compared to reusing all web topics.

\section{Experimental Results}\label{sec:results}

In \S\ref{sec:superswarm}, we evaluate \olmix---using \fullmixreuse and \partialmixreuse, with recomputation done via \base---in an LM development setting where the domain set evolves through several updates. In \S\ref{sec:validation}, we empirically validate our results from \S\ref{sec:analysis}, confirming that our theoretical analysis accurately characterizes mixture reuse performance in practice.

\subsection{Real-World LM Development Scenario}\label{sec:superswarm}

We evaluate \fullmixreuse and \partialmixreuse across a sequence of five domain updates mirroring real-world LM development. Our results show that these methods maintain strong performance while substantially reducing computational costs compared to full recomputation.

\subsubsection{Setup}

See Appendix~\ref{supp:exp_details} for full experiment details on evaluation tasks and the $1$B target models we train.

\textbf{Domain Updates} We simulate a real-world LM development workflow with an initial web corpus that undergoes 5 updates (Table~\ref{tab:domain-updates-full}). 
The final domain set contains 64 domains (token counts in Table~\ref{tab:domain_sizes}).

\begin{table}[!h]
\centering
\caption{Datasets used for simulated domain updates.}
\small
\begin{tabular}{l p{0.6\linewidth} r}
\toprule
\textbf{Operation} & \textbf{Dataset(s)} & \textbf{$\Delta m$} \\
\midrule
Initial &
DCLM~\citep{dclm} partitioned into topical domains using WebOrganizer~\citep{wettig2025organizewebconstructingdomains}) &
24 \\
\midrule
\Add &
Stack-Edu~\citep{allal2025smollm2smolgoesbig}, partitioned by programming language &
+15 \\
\Add &
ArXiv~\citep{azerbayev2023llemma}, FineMath~3+~\citep{allal2025smollm2smolgoesbig}, olmOCR Science PDFs~\citep{olmo2025olmo3}, Dolma 1 Wikipedia~\citep{soldaini2024dolma}, AlgebraicStack~\citep{azerbayev2023llemma}, pes2o~\citep{peS2o} &
+6 \\
\Revise &
olmOCR Science PDFs~\citep{olmo2025olmo3} (reformatted tables and references) &
0 \\
\Remove &
AlgebraicStack~\citep{azerbayev2023llemma} &
$-1$ \\
\Partition &
olmOCR Science PDFs~\citep{olmo2025olmo3}, partitioned into topical domains using WebOrganizer~\citep{wettig2025organizewebconstructingdomains} &
+20 \\
\bottomrule
\end{tabular}
\label{tab:domain-updates-full}
\end{table}

\textbf{Methods compared.} We consider:
\begin{itemize}[itemsep=0pt,topsep=0pt,leftmargin=12pt]
    \item Natural: mix proportional to domain sizes.
    \item Full recomputation: apply \base after each update (high performance, high cost).
    \item Swarm reuse: reuse all accumulated proxy runs that can represented on $\D'$ and apply \base on the combined swarm (see Algorithm~\ref{alg:swarm_reuse_method}). 
    \item \fullmixreuse (our method): reuse ratios for unaffected domains, recompute affected ones. 
    \item \partialmixreuse (our method): reuse ratios within DCLM topics and within Stack-Edu languages while recomputing at the source level; recompute DCLM:software development when adding Stack-Edu due to high coupling (see \S~\ref{sec:validation_results} for justification).
\end{itemize}

\textbf{Experimental protocol.} 
For full recomputation, we use swarm size $K \approx c(m + 1)$ and vary $c \in \{1, 2, 3\}$ to showcase different compute regimes. For swarm reuse and mixture reuse, we set swarm sizes to roughly match that of full recomputation with $c=1$.
We report results over $3$ random seeds of swarms with $k=4$, $R=1$T. Full details are in Appendix~\ref{supp:imp_details}.

\begin{figure}[H]
    \centering
    \includegraphics[width=0.6\linewidth]{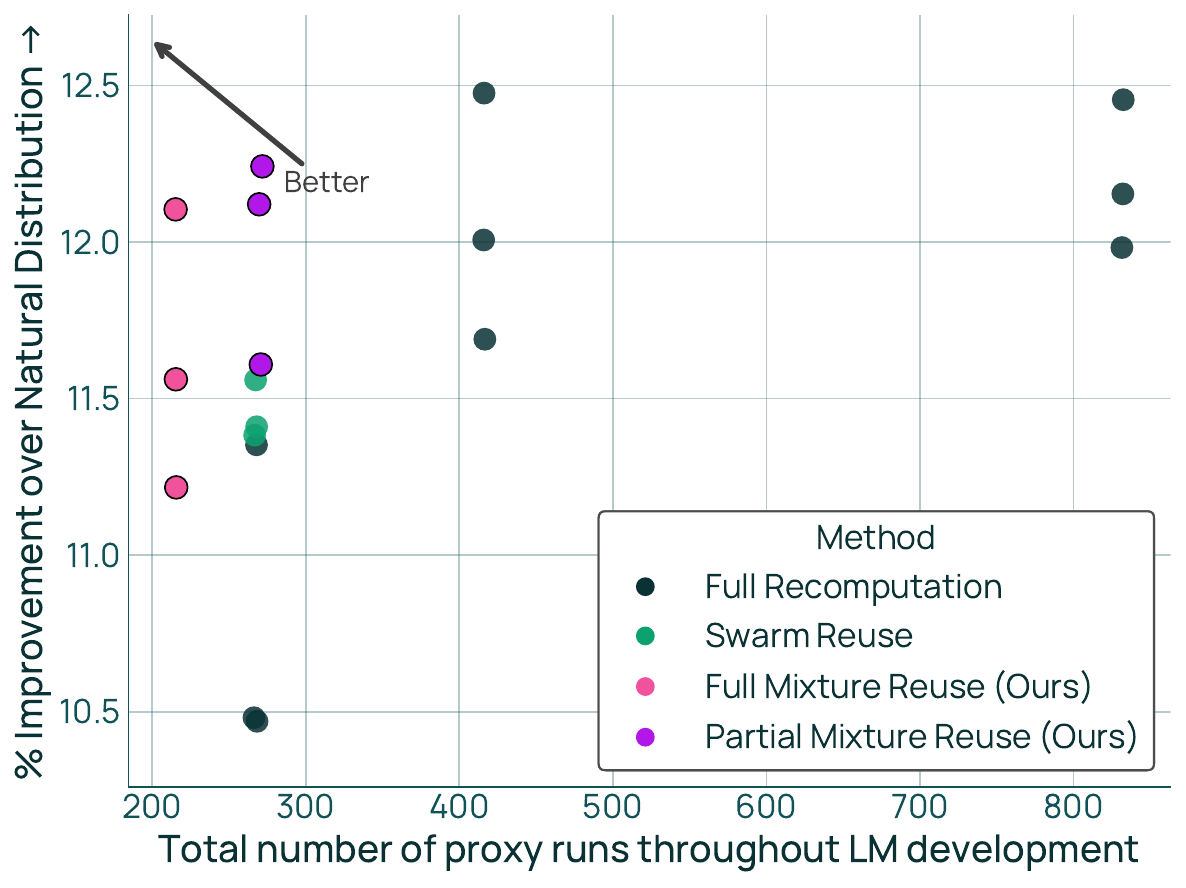}
    \caption{Performance improvement versus cost of mixing under evolving domains. \fullmixreuse and \partialmixreuse achieve $>95\%$ of the improvement of full recomputation while using at least $67\%$ fewer proxy runs.}
    \label{fig:main_results}
\end{figure}

\subsubsection{Results}\label{sec:results_1}

See Appendix~\ref{sec:additional_exp_results} for additional results on: 1) $R=6$T, 2) smaller proxy run budgets, and 3) individual domain update operators.

\textbf{\fullmixreuse roughly matches full recomputation at substantially lower cost.} Figure~\ref{fig:main_results} shows relative improvement in downstream performance (average BPB) over the natural distribution versus the total number of proxy runs across all 5 updates. \fullmixreuse achieves 95\% of full recomputation's ($c\texttt{=}3$) performance improvement (+11.6\% versus +12.2\%) while using 74\% fewer total proxy runs (216 versus 832). 

\textbf{\partialmixreuse closes the remaining gap between \fullmixreuse and full recomputation.}
By selectively recomputing some unaffected domains, \partialmixreuse achieves 98\% of full recomputation's performance (+12.0\%) with 272 total runs---still 67\% fewer than full recomputation.

\textbf{Mixture reuse outperforms swarm reuse.}
Swarm reuse achieves +11.4\% with 268 runs, underperforming \fullmixreuse (+11.6\% with 216 runs) despite using more runs. This is likely because (1) representing old swarms on updated domain sets over-explores biased subspaces, and (2) swarms cannot be reused when domains are removed or revised.

\textbf{At matched proxy run budgets, mixture reuse outperforms alternatives.} 
At a budget of 216-272 proxy runs, \fullmixreuse and \partialmixreuse achieve +11.6\% and +12.0\% improvement over the natural distribution, respectively. At this same budget, swarm reuse achieves +11.4\% and full recomputation achieves +10.8\% ($c=1$).

\textbf{Our best mixture is 3.05× more data-efficient than the natural distribution.}
Beyond final performance, we measure data efficiency: how many training steps does our best learned mixture need to match the natural distribution's final performance? Figure~\ref{fig:checkpoints} shows that \partialmixreuse reaches the natural distribution's final BPB in approximately 20,000 steps versus 61,000 steps---a 3.05× speedup.

\begin{figure}
    \centering
    \includegraphics[width=0.6\linewidth]{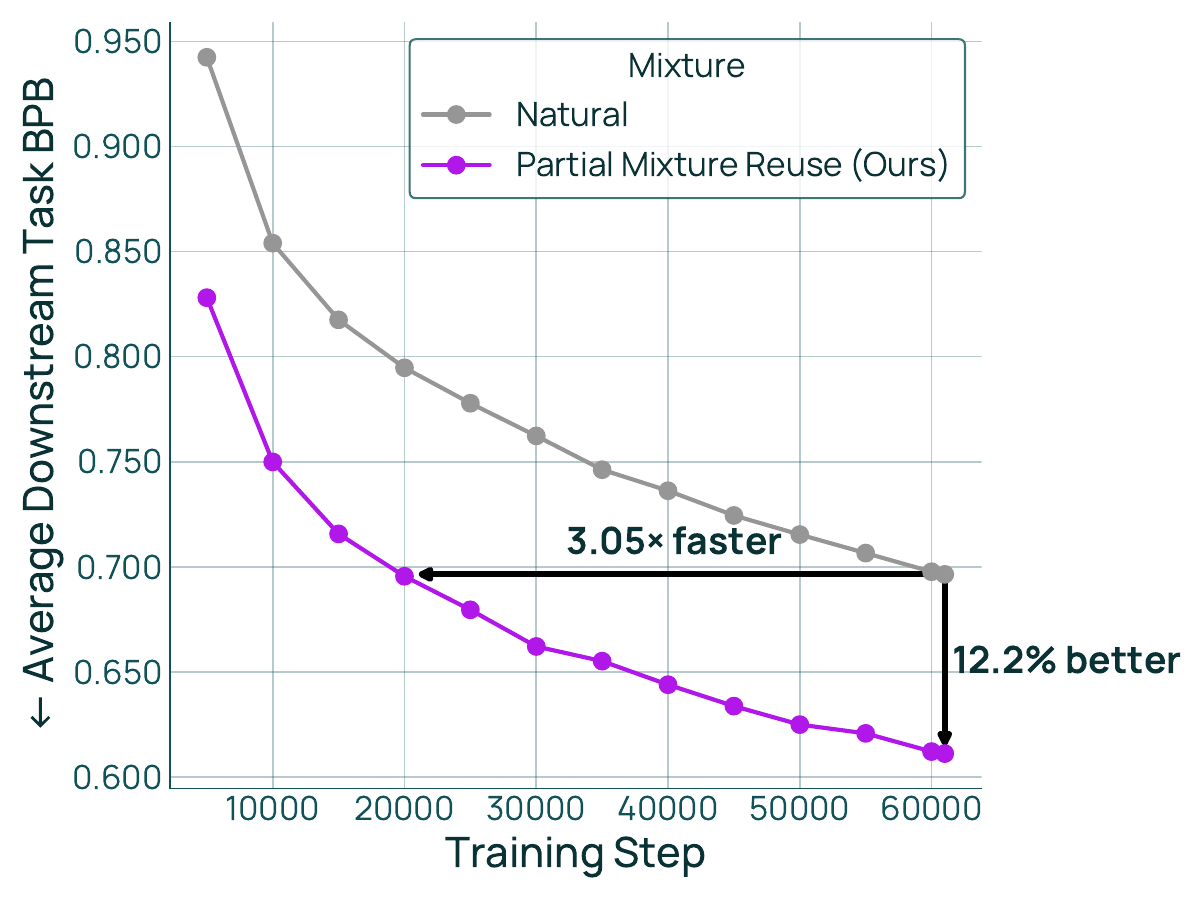}
    \caption{Downstream performance across training for \partialmixreuse (best seed) versus the natural distribution. \partialmixreuse reaches the natural distribution's final performance in 3.05$\times$ fewer steps.}
    \label{fig:checkpoints}
\end{figure}

\textbf{Qualitative analysis of learned mixtures.}
Figure~\ref{fig:superswarm_proposed_mixes} shows the final proposed mixes over a subset of domains (full mix in Table~\ref{tab:full_domain_mixtures}). \partialmixreuse is more similar to full recomputation (total variation distance of 0.067) than to the natural distribution (distance of 0.127), aligning with our downstream performance findings.

\begin{figure}
    \centering
    \includegraphics[width=0.8\linewidth]{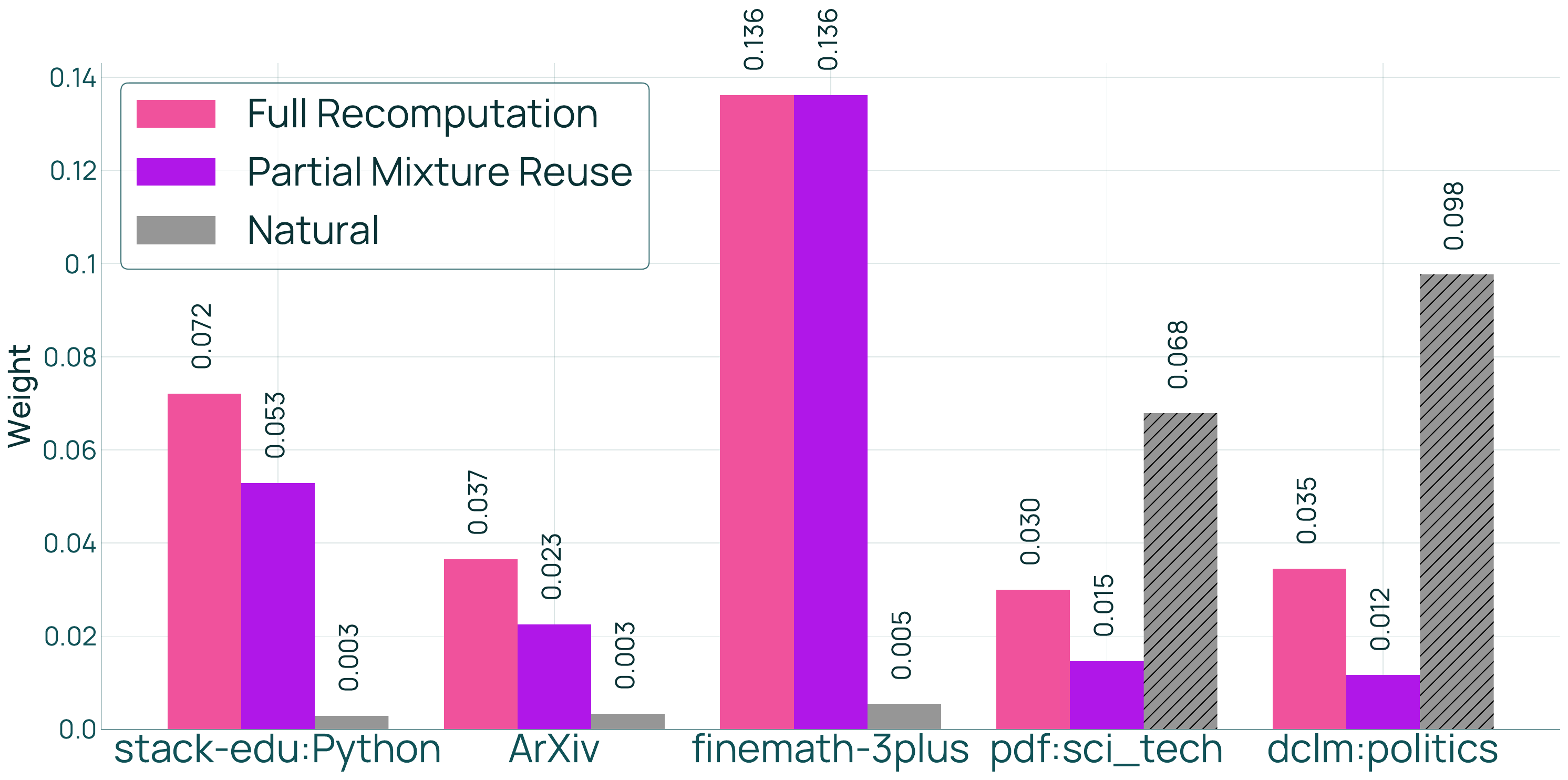}
    \caption{Proposed mixes for different mixing strategies. The mixes produced by full recomputation and \partialmixreuse are more similar to each other than they are to the natural distribution, confirming the performance results in Figure~\ref{fig:main_results}. Domains shown have the greatest difference in weights between full recomputation and natural.}
    \label{fig:superswarm_proposed_mixes}
\end{figure}

\subsection{Empirical Validation of Theorems~\ref{thm:general_gap} and~\ref{thm:add_domains}}\label{sec:validation}

We empirically measure the terms in Theorems~\ref{thm:general_gap} and~\ref{thm:add_domains} to assess whether they track the performance of mixture reuse.

\subsubsection{Setup}

We consider two examples of an \Add update:
\begin{itemize}
    \item $\D_1=$DCLM~\citep{dclm} partitioned into 24 topic-based domains using WebOrganizer~\citep{wettig2025organizewebconstructingdomains}, $\D_2'$ = Stack-Edu~\citep{allal2025smollm2smolgoesbig} partitioned into $15$ programming languages.
    \item $\D_1=$DCLM partitioned into 24 topic-based domains using WebOrganizer, $\D_2'$ = olmOCR Science PDFs~\citep{olmo2025olmo3} partitioned into $21$ topic-based domains using WebOrganizer. 
\end{itemize}

\subsubsection{Results}\label{sec:validation_results}

\textbf{Performance is correlated with the reuse gap.} 
Mixture reuse yields $q^\star(\tilde{p}_{\Dfix})$ from Algorithm~\ref{alg:conditional_mixing_method} given reused ratios $\tilde{p}_{\Dfix}$. Full recomputation yields $q^\star$ from directly solving~\eqref{eq:mixing_problem} for the optimal mix on $\D'$.
Theorem~\ref{thm:general_gap} shows that the performance gap of mixture reuse, $F(q^\star(\tilde{p}_{\Dfix})) - F(q^\star)$, where $F$ is the average downstream task BPB, is bounded in terms of the reuse gap, $\|\tilde{p}_{\Dfix} - q_{\Dfix}^\star \|$, where $q_{\Dfix}^\star$ is $q^\star$ normalized over $\Dfix$. We validate that the reuse gap predicts the performance gap.

To do this, we construct different values of $\tilde{p}_{\Dfix}$ with varying reuse gaps and measure the resulting performance gaps.
We first approximate $q_{\Dfix}^\star$ (and $F(q^\star)$) by running full recomputation on $\D'$ using \base and normalizing the resulting mix over $\Dfix$. 
We then construct two $\tilde{p}_{\Dfix}$:
\begin{itemize}
    \item Weak mix: to construct a $\tilde{p}_{\Dfix}$ far from $q^\star_{\Dfix}$, we run \base on $\D'$ with a modified objective that maximizes BPB and then normalize the resulting mix over $\Dfix$.
    \item Intermediate mix: we average the weak mix with $q^\star_{\Dfix}$.
\end{itemize}

Figure~\ref{fig:theorem1_validation} confirms the findings of Theorem~\ref{thm:general_gap}: as the reuse gap increases, the performance gap also increases across both settings.

\begin{figure}
    \centering
    \includegraphics[width=0.4\linewidth]{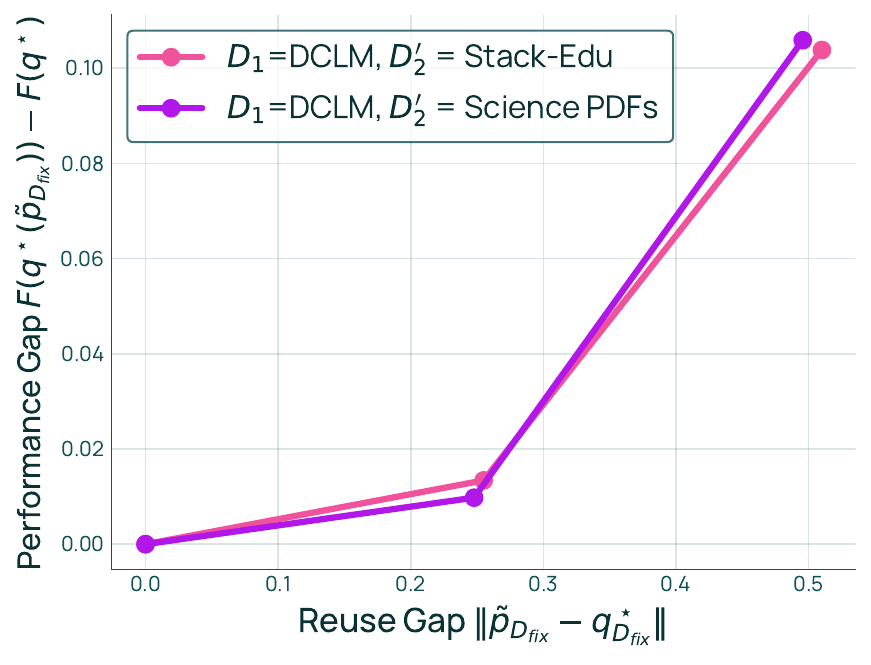}
    \caption{Performance vs Reuse Gap. The optimality of $\tilde{p}_{\Dfix}$ after the domain update is correlated with the success of mixture reuse.}
    \label{fig:theorem1_validation}
\end{figure}

\textbf{The reuse gap is controlled by how much the domain update shifts the optimal mixture.}
When adding domains, Theorem~\ref{thm:add_domains} shows that the reuse gap (and consequently the performance gap) depends on how much total weight shifts to the affected domains. Specifically, the reuse gap depends on $1 - \rho^\star$, where $\rho^\star$ is the total weight on the reused domains $\D_1$ in the optimal mix $q^\star$.
We validate that $1 - \rho^\star$ predicts both the reuse gap and performance gap.

To do this, we construct settings with different $\rho^\star$ by varying the repetition constraint. A tight constraint (small $k$ or large $R$ in~\eqref{eq:mixing_problem}) forces the mix to stay close to the natural distribution, yielding $\rho^\star \approx 1$ when $\D_1$ is much larger than the added domains. A relaxed constraint allows more weight to shift to new domains, potentially yielding smaller $\rho^\star$.
We test two settings: $R = 1$T (relaxed) and $R = 6$T (tight). For each $R$, we 1) compute $\tilde{p}_{\Dfix}$ by running \base on $\D$; 2) apply \fullmixreuse on $\D'$ using $\tilde{p}_{\Dfix}$; and 3) use full recomputation (run \base on $\D'$) to get an approximate $\rho^\star$, the reuse gap, and the performance gap.

Figure~\ref{fig:theorem2_validation} shows that larger $1 - \rho^\star$ correlates with both a larger reuse gap and a larger performance gap across both settings, validating that the extent to which the domain update shifts the optimal mixture directly impacts mixture reuse performance.

\begin{figure}
    \centering
    \includegraphics[width=\linewidth]{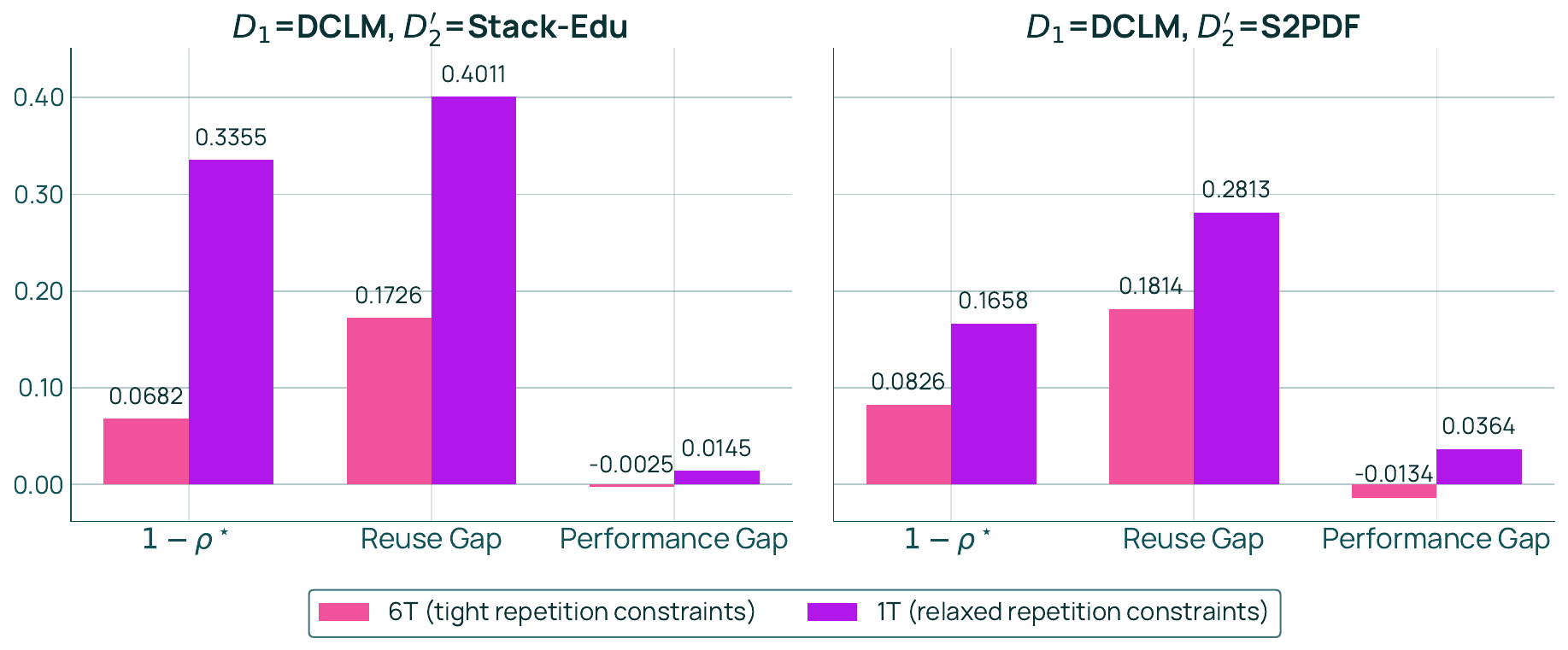}
    \caption{Performance gap vs reuse gap vs $1 - \rho^\star$ when adding Stack-Edu to DCLM (left) and olmOCR Science PDFs to DCLM (right). Varying $R$ changes $\rho^\star$, and we observe that $1 - \rho^\star$ propagates to both the reuse gap and performance gap, validating Theorem~\ref{thm:add_domains}.}
    \label{fig:theorem2_validation}
\end{figure}

\textbf{\partialmixreuse reduces coupling and improves performance.}
Theorems~\ref{thm:general_gap} and~\ref{thm:add_domains} show that the coupling term $\kappa$ controls the \emph{rate} at which changes in optimality translate to the performance gap. The coupling term is large when both reused domains $\Dfix$ and recomputed domains $\Dcomp$ impact similar downstream tasks.
We validate that reducing coupling---by adjusting $\Dfix$ vs $\Dcomp$ via \partialmixreuse---improves both the reuse gap and performance gap.

To do this, we analyze a scenario where coupling is high: adding Stack-Edu to DCLM topics when $R=1$T (from Figure~\ref{fig:theorem2_validation} left, purple). 
Figure~\ref{fig:theorem2_validation_dclm_mix} compares the reused ratios $\tilde{p}_{\Dfix}$ against the optimal ratios $q^\star_{\Dfix}$ over DCLM topics. There is a stark difference in the weight on DCLM's software development topic, suggesting high coupling between software development and Stack-Edu.
Based on this, we test whether \partialmixreuse with $\Dcomp = \text{Stack-Edu} \cup \{\text{DCLM:software\_development}\}$ reduces coupling and improves performance compared to \fullmixreuse (which uses $\Dcomp = \text{Stack-Edu}$).

\begin{figure}
    \centering
    \includegraphics[width=0.6\linewidth]{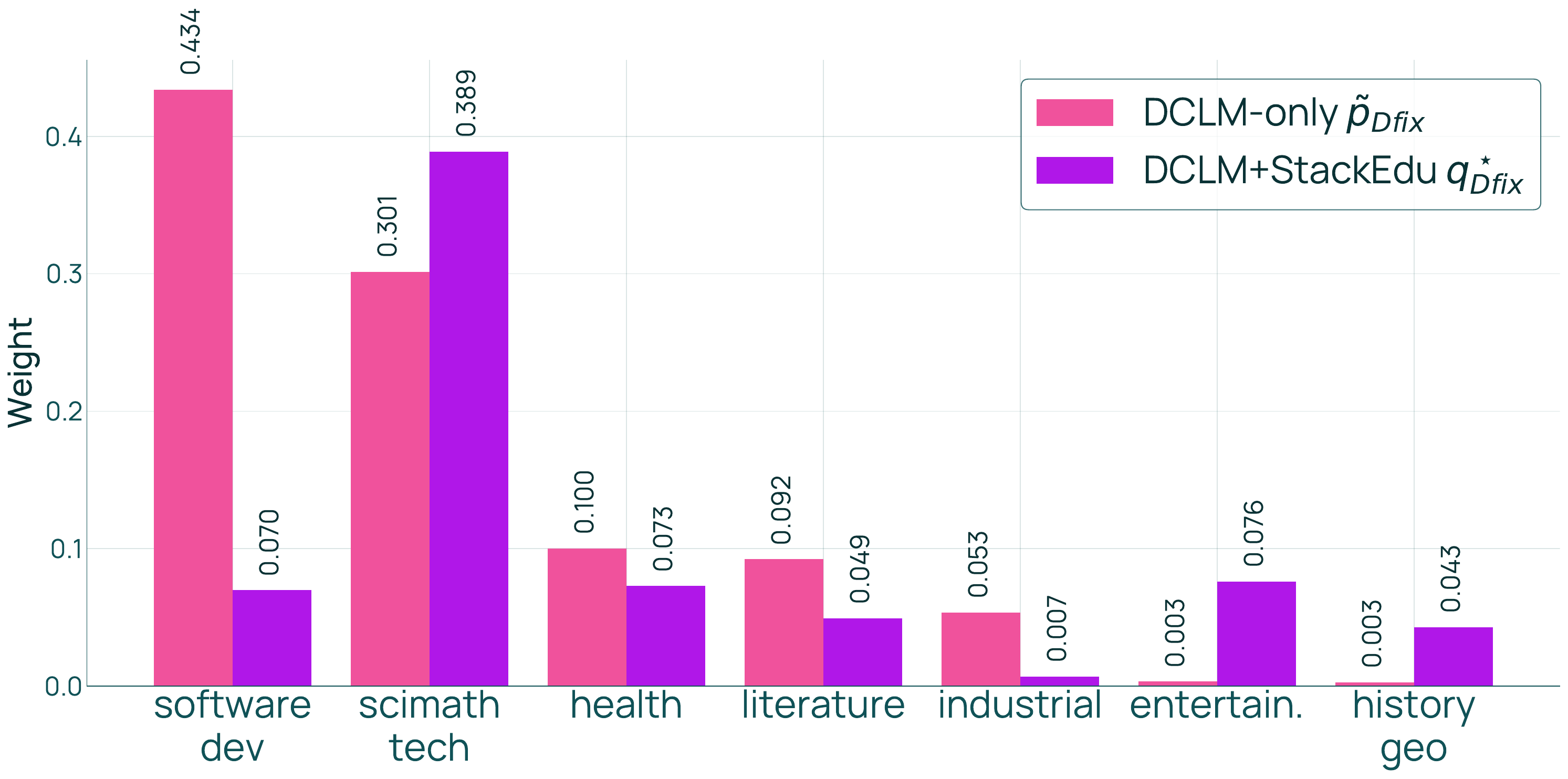}
    \caption{Comparison of $\tilde{p}_{\Dfix}$ versus $q^\star_{\Dfix}$ (top-7 domains according to $\tilde{p}_{\Dfix}$) when $\D_1=$DCLM, $\D_2'$=Stack-Edu with $R=1$T. The only domain that significantly differs in mixture weight is software development, suggesting that \partialmixreuse can reduce coupling if we recompute DCLM:software\_development.}
    \label{fig:theorem2_validation_dclm_mix}
\end{figure}

Figure~\ref{fig:theorem2_validation_coupling} shows that recomputing DCLM:software\_development along with Stack-Edu reduces both the reuse gap and performance gap, despite $1 - \rho^\star$ being larger for \partialmixreuse. This breaks the pattern seen in Figure~\ref{fig:theorem2_validation}: larger $1 - \rho^\star$ no longer translates to a larger performance gap, suggesting that the reduction in the coupling term $\kappa$ dominates.
To confirm this directly, we compute $\kappa$ (see Appendix~\ref{supp:defs} for definition) for \fullmixreuse and $\kappa$ for \partialmixreuse with $\Dcomp = \text{Stack-Edu} \cup \{\text{DCLM:topic} \}$ for each DCLM topic.
Figure~\ref{fig:theorem2_validation_kappa_reduction} shows that recomputing DCLM:software\_development reduces $\kappa$ far more than recomputing any other DCLM topic, validating that \partialmixreuse reduces coupling.

\begin{figure}
    \centering
    \includegraphics[width=0.6\linewidth]{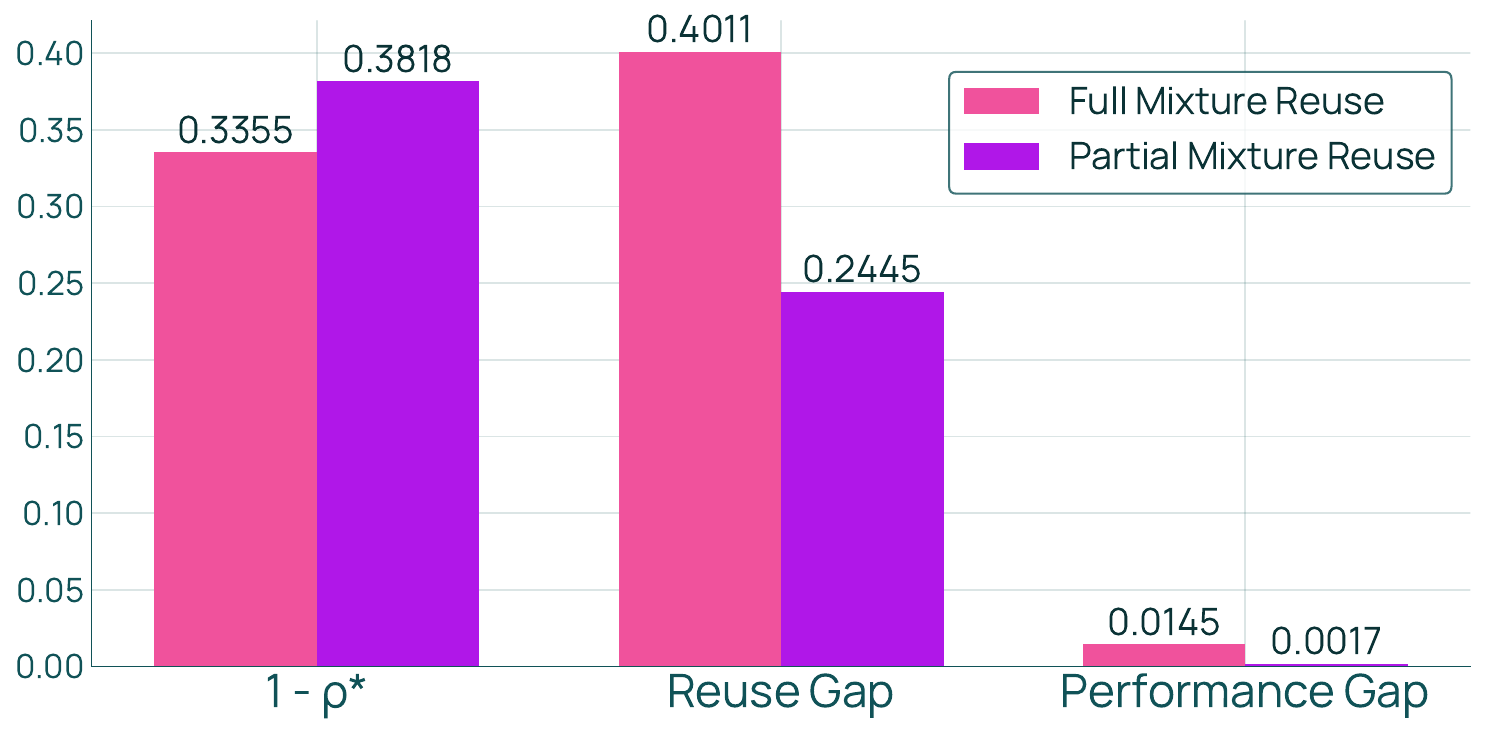}
    \caption{Performance of \partialmixreuse (recompute DCLM:software\_development) for $\D_1$=DCLM, $\D_2'$=Stack-Edu with $R=1$T. Recomputing DCLM:software development along with Stack-Edu reduces both the reuse gap and the performance gap, even though its $1 - \rho^\star$ is higher.}
    \label{fig:theorem2_validation_coupling}
\end{figure}

\begin{figure}
    \centering
    \includegraphics[width=0.6\linewidth]{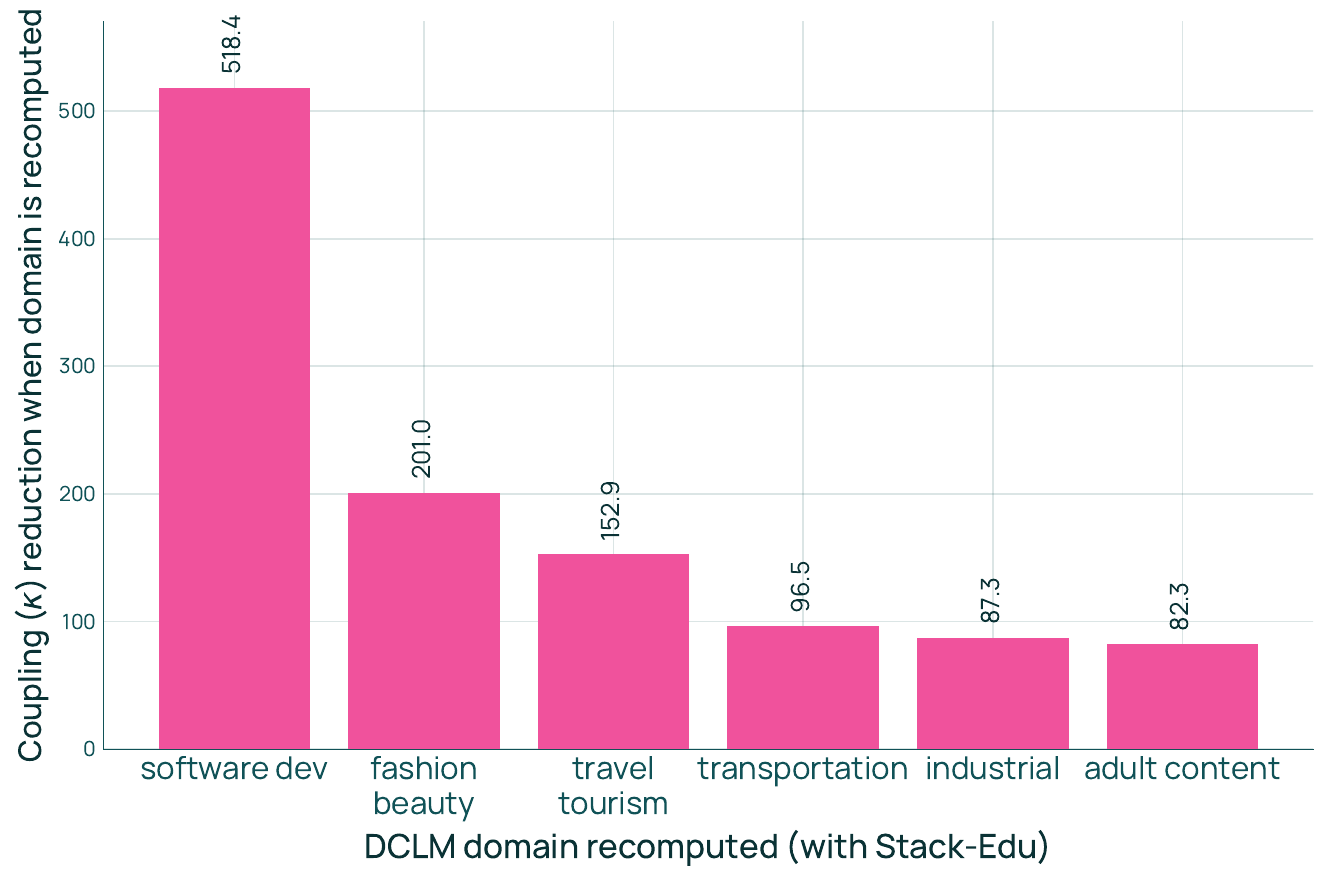}
    \caption{The top-6 DCLM domains for which recomputation results in the greatest reduction in $\kappa$ for $\D$=DCLM, $\D'$=DCLM+Stack-Edu. The reduction from recomputing DCLM:software development via \partialmixreuse is substantially higher than recomputing other DCLM topics.}
    \label{fig:theorem2_validation_kappa_reduction}
\end{figure}
\section{Discussion}

We presented \olmix, addressing two key challenges in data mixing for language models: configuring effective mixing methods and efficiently updating mixtures as domain sets evolve. We conduct a large empirical study of design choices and introduce mixture reuse. 

\textbf{Limitations.} First, our work focuses exclusively on the offline mixing schema. While this schema is widely adopted, other approaches, such as those that use a dynamic mix to explore the mixture weight space~\citep{xie2023doremioptimizingdatamixtures, fan2024dogedomainreweightinggeneralization}, may benefit from different design choices. Second, our theoretical analysis of mixture reuse assumes log-linear regression models. The analysis should extend naturally to other parametric models like AutoScale~\citep{kang2025autoscalescaleawaredatamixing} and BiMix~\citep{ge2025bimixbivariatedatamixing}, but it is less clear how to extend it to non-parametric models that yield non-convex and non-differentiable mixing objectives.
Third, \partialmixreuse requires determining the subset of unaffected domains to recompute. Some cases are quite intuitive; for instance, recomputing DCLM:software\_development when Stack-Edu code data is added. However, we do not provide a general automated approach for determining what to recompute, limiting \partialmixreuse's applicability without domain expertise. 

\textbf{Future work.} First, we aim to extend mixture reuse and our design choice study to online mixing methods that adjust mixtures during training. This is in contrast with this work, which focuses on the LM development that occurs \textit{before} the final model starts training. Second, we envision co-design of data mixing with other LM development workflows, such as using mixture performance as feedback for data quality filtering or domain discovery. Third, validating our findings at larger model scales would improve confidence in the reliability of our design choices and reuse mechanisms.

\section*{Author Contributions}

Mayee F. Chen led the project during an internship at Ai2: she developed the theoretical framework, implemented all methods, designed and ran all experiments, validated the approach at scale for OLMo 3, and wrote the paper. Kyle Lo and Luca Soldaini were the primary mentors and worked closely with Mayee on developing the methodology, designing experiments, integrating into Olmo 3, and writing the paper. Tyler Murray built and maintained the training infrastructure that enabled large-scale experimentation. David Heineman designed small-scale evaluation metrics for data mixing. Matt Jordan contributed to mixing methodology. Hannaneh Hajishirzi and Christopher R\'{e} provided research guidance and helped frame the work.

\section*{Acknowledgments}

We thank Neel Guha, Junmiao Hu, Ronny Junkins, Pang Wei Koh, Jerry Liu, Roberto Garcia Torres, Alex Wettig, Steven Euijong Whang, and Michael Zhang for helpful feedback and discussions. 
We thank the Olmo 3 team, particularly Allyson Ettinger, for iterative feedback and experimental validation of Olmix.
MC was supported in part by the Office of Naval Research (ONR) under No. N000142312633 (Deep Signal Processing). Any opinions, findings, and conclusions or recommendations expressed in this material are those of the authors and do not necessarily reflect the views, policies, or endorsements, either expressed or implied, of ONR or the U.S. Government.

\clearpage
\bibliographystyle{abbrvnat}
\bibliography{references}

\clearpage

\appendix

\section{Offline Schema Study Details}\label{supp:study_details}

All experiments for the design choice study, unless specified, use the same configuration as \base. For mixing over the DCLM topics, we constructed a sparse swarm of size $K = 128$.
We also mix at the source level over DCLM, Stack-Edu, ArXiv, FineMath 3+, olmOCR Science PDFs, Wikipedia, and Pes2o (see Table~\ref{tab:domain_sizes}). For the sources, we constructed a dense swarm of size $K = 64$. We use log-linear regression per task, an exact solver with $\lambda = 0.05$, and enforced repetition constraints in the optimization problem with $R =$ 6T and $k = 4$.

\subsection{RQ1: Proxy model size} 

\textbf{Details.} Table~\ref{tab:model_architectures} contains details about the architectures of the 1M, 15M, 30M, and 60M proxy models we train. 

\begin{table}[htb]
\centering
\caption{Proxy Model Architecture Configurations.}
\label{tab:model_architectures}
\begin{tabular}{lcccc}
\toprule
\textbf{Parameter} & \textbf{1M} & \textbf{15M} & \textbf{30M} & \textbf{60M} \\
\midrule
Vocab size & 100,352 & 100,352 & 100,352 & 100,352 \\
\texttt{n\_layers} & 4 & 8 & 4 & 8 \\
\texttt{n\_heads} & 4 & 4 & 8 & 8 \\
\texttt{d\_model} & 16 & 128 & 256 & 384 \\
\texttt{head\_dim} & 4 & 32 & 32 & 48 \\
\bottomrule
\end{tabular}
\end{table}

\subsection{RQ2: Swarm size} 

\textbf{Additional results.} In Figure~\ref{fig:swarm_size_1B}, we vary $K = c(m+1)$ for $c = 1, 2, 3, 4$ and $m=6, 24$ and study its effect on downstream BPB of 1B models. Despite having fewer $m$, the error curves when scaled according to $c$ collapse together, suggesting that the $\mathcal{O}(m)$ sample complexity holds at both 30M and 1B model scales. 

\textbf{Details.} For the results in Figure~\ref{fig:swarm_size} and Figure~\ref{fig:swarm_size_1B}, we constructed proposed mixes using a search-based solver rather than an exact solver. We also did not enforce repetition constraints and had originally included 5 additional metrics in our evaluation suite: Qasper Yes/No, Sciriff Yes/No, LabBench DBQA, LabBench ProtocolQA, and MedQA EN. 

For $m=24$, we used all WebOrganizer DCLM topics. For $m < 24$, we selected the $m$ largest domains. In particular, for $m=18$, we used: Art and Design, Crime and Law, Education and Jobs, Electronics and Hardware, Entertainment, Finance and Business, Games, Health, Literature, Politics, Religion, Science Math and Technology, Social Life, Software, Software Development, Sports and Fitness, Transportation, Travel and Tourism. 
For $m=12$, we used: Crime and Law, Education and Jobs, Entertainment, Finance and Business, Games, Health, Literature, Politics, Religion, Science Math and Technology, Software, and Software Development.
For $m=6$, we used: Entertainment, Finance and Business, Games, Health, Politics, and Science Math and Technology.

\begin{figure}
    \centering
    \includegraphics[width=0.48\linewidth]{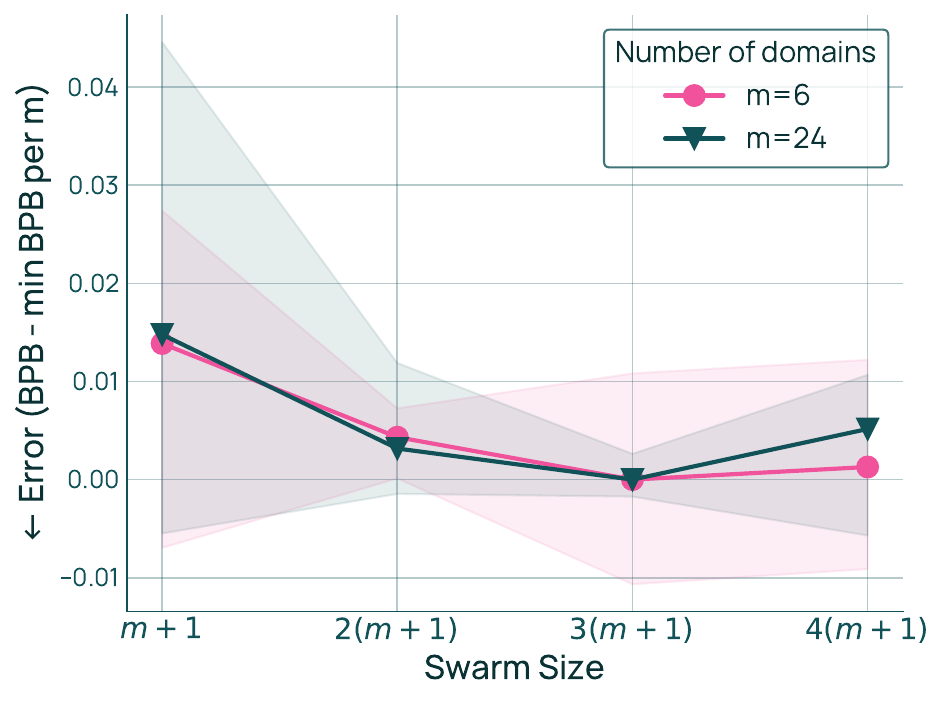}
    \caption{Error on 1B models versus swarm size. Similar to Figure~\ref{fig:swarm_size}, we see that the error curves collapse across both $m$, supporting our finding that $K$ needs $\mathcal{O}(m)$ runs to obtain strong downstream performance. Results are averaged across 3 random seeds, and the intervals indicate min and max results.}
    \label{fig:swarm_size_1B}
\end{figure}

\subsection{RQ3: Swarm distribution}\label{supp:swarm_distribution}

\textbf{Details.} To construct the sparse and dense swarms, we first generated mixes sampled from a Dirichlet distribution.
For the dense swarm, we discarded any mixes that have domains with 0 weight, and for the sparse swarm, we enforced that if the weight of a domain is less than 0.05, then it is clipped to 0. The entire mixture is normalized afterwards. 

For the centering experiments, we compared a swarm with a natural Dirichlet prior to a strong and weak prior. The strong prior was the proposed mix obtained using the natural swarm. The weak prior was a mix that aimed to maximize BPB using the the natural swarm using a simulation-based solver. Regardless of the choice of Dirichlet prior, we used the natural distribution in the KL regularization term of the objective for fair comparison.

For both sets of experiments, we constructed a held-out test set containing of 30 samples, where 15 were from a Dirichlet distribution centered around the sparse swarm's optimal mix and 15 were from a Dirichlet distribution centered around the dense swarm's optimal mix. This measures regression fit in high-performance regions of the mixture space, which matters in the mixture optimization step. 

\subsection{RQ4: Regression model family}\label{supp:regression}

\textbf{Additional results.} Figure~\ref{fig:regression_models_olmo3} compares the downstream performance and regression fit (measured on three random train-test splits of the swarm, each with 10 test mixes) across several regression model families at the source level. Across both metrics, we see that the log-linear regression model outperforms other approaches. 

In Figure~\ref{fig:k_vs_m_dclm}, we plot the relationship between swarm size, number of domains, and regression fit across all regression model families over DCLM topics. These figures reveal that different regression model families have different sample complexity and regimes in which they do well. This offers a potential explanation for why existing methods lack consensus. 

\begin{figure}
    \centering
    \includegraphics[width=0.5\linewidth]{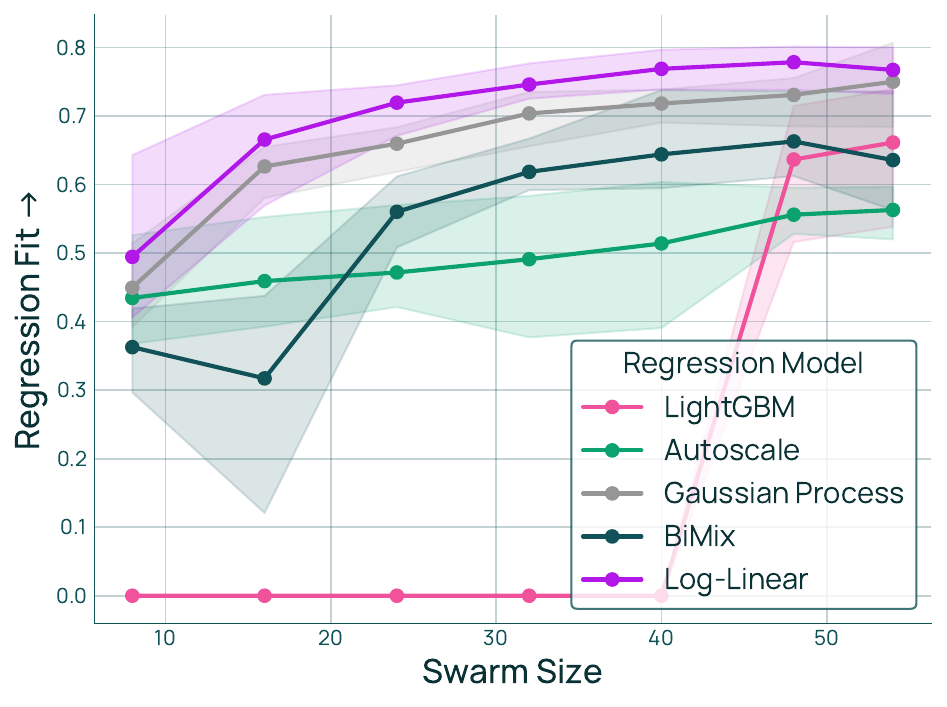}
    \includegraphics[width=0.48\linewidth]{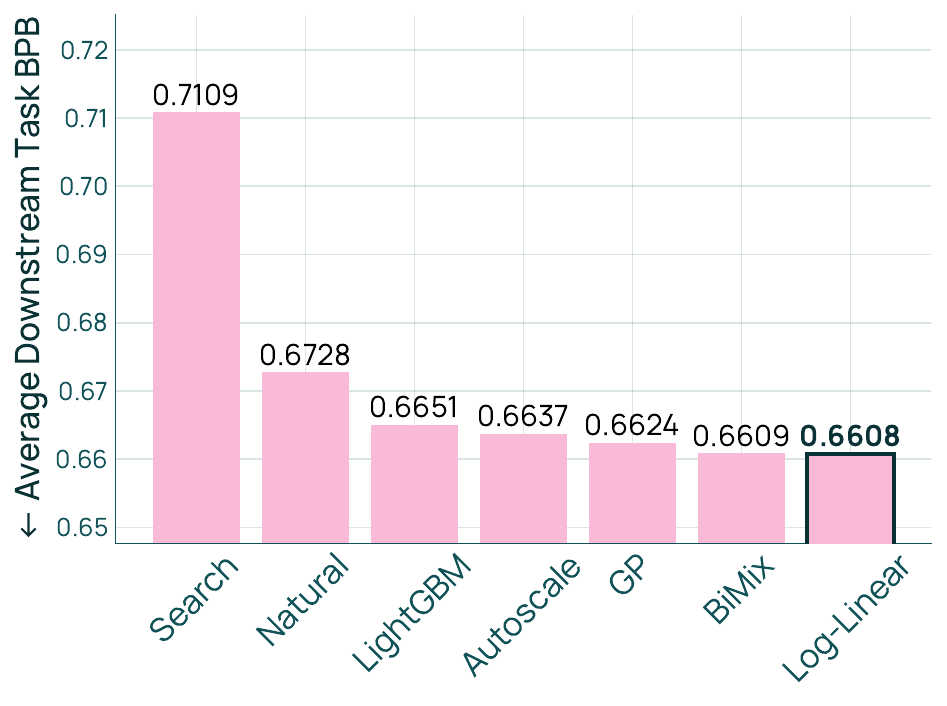}
    \caption{Left: Regression fit versus swarm size across regression models when mixing across sources. Results are averaged across 3 random seeds, and the intervals indicate min and max results. Right: downstream performance of regression models over sources.}
    \label{fig:regression_models_olmo3}
\end{figure}

\begin{figure}
    \centering
    \includegraphics[width=\linewidth]{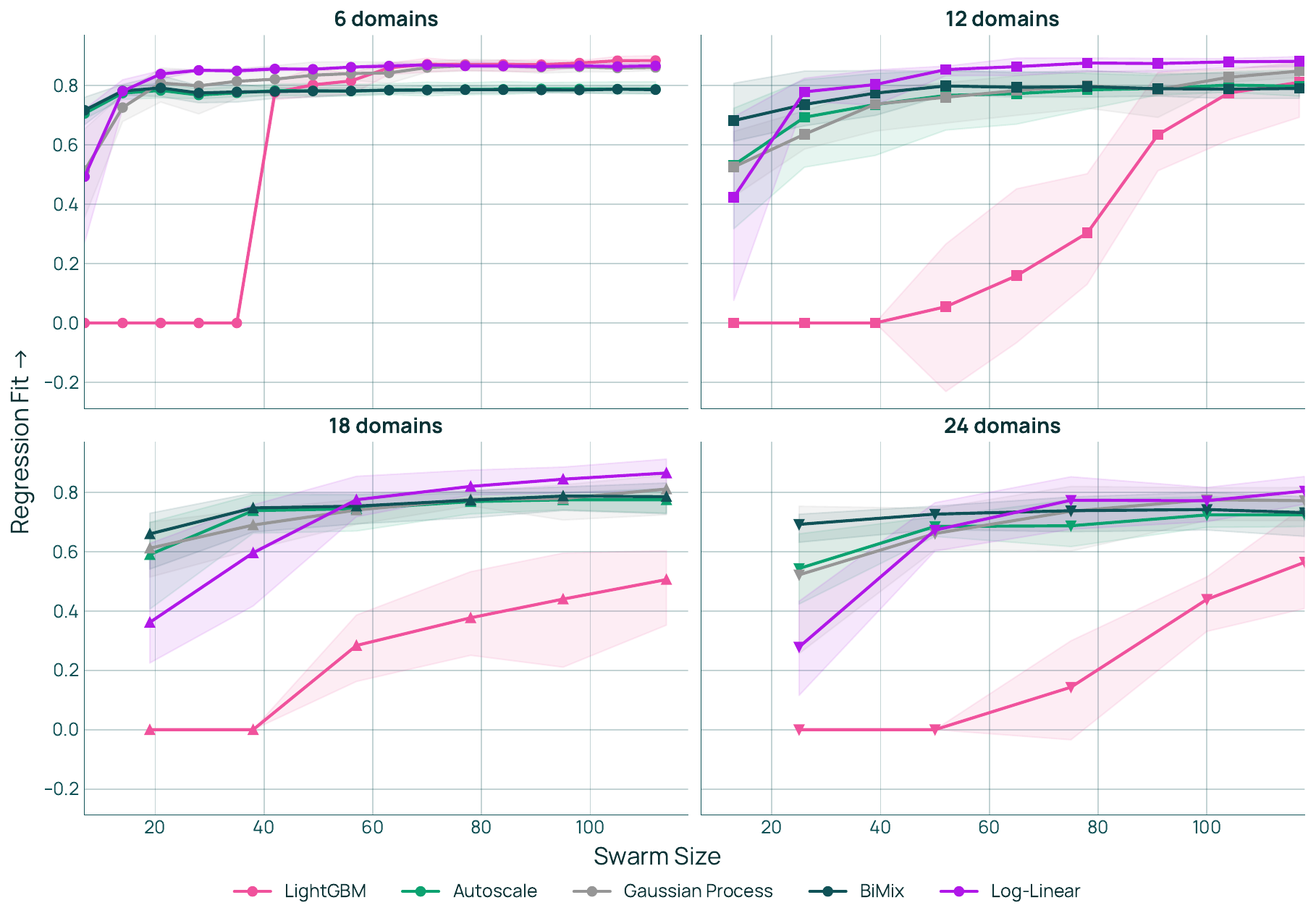}
    \caption{Regression fit versus swarm size and number of domains across regression models when mixing on DCLM topics. Across $m=6, 12, 18, 24$ domains, each regression model requires different swarm sizes to obtain good regression fit. Results are averaged across 3 random seeds, and the intervals indicate min and max results.}
    \label{fig:k_vs_m_dclm}
\end{figure}

\textbf{Details.} We describe how we adapt the regression models from BiMix~\citep{ge2025bimixbivariatedatamixing}, Autoscale~\citep{kang2025autoscalescaleawaredatamixing}, and Data Mixing Laws~\citep{ye2025datamixinglawsoptimizing} to our setting where performance is measured on a set of downstream tasks.

\begin{itemize}
    \item BiMix: The BiMix mixing law models the validation loss on domain $i$ as a function of the mixture ratio on domain $i$ and the number of training steps $s$:
    \begin{align}
        \hat{f}_i(p, s) = \frac{A_i}{p_i^{\alpha_i}} \bigg( \frac{B_i}{s^{\beta_i}} + C_i\bigg) \; \forall i \in [m].
    \end{align}

    Since we do not model the number of steps and instead directly transfer the proposed mix from the proxy scale to the target scale, this mixing law simplifies to $\hat{f}_i(p) = A_i p_i^{-\alpha_i}$, where the constants are absorbed into $A_i$. To extend this mixing law to downstream tasks rather than validation loss on training domains, we assume that task BPB can be modeled as a weighted average of training domain-specific mixing laws. Therefore, for each downstream task $i$, the BiMix mixing law for our setting is:
    \begin{align}
        \hat{f}_i(p) = \sum_{j = 1}^m A_{ij} p_j^{-\alpha_{ij}}.
    \end{align}

    \item AutoScale: The AutoScale mixing law models validation loss on a held-out dataset $i$ as a function of the mixture and the number of tokens $R$:
    \begin{align}
        \hat{f}(p, R) = \sum_{j = 1}^m \bigg((N^{j}_0 + p_j R)^{-\gamma_{j}} + l_j\bigg),
     \end{align}

     where $N^j_0$ is the ``effective'' number of tokens of all other domains excluding $j$.
     The original paper fits separate parameters for each domain using perturbed runs; we instead fit all parameters jointly using our larger swarm. We make two modifications. First, we replace the per-domain constants $l_j$ with a single constant term $c_i$ corresponding to a given task $i$. Second, instead of directly learning $N^j_0$, which can be quite large given the order of magnitude of $R$, we rewrite $N^j_0 + p_j R$ as $R(A_{ij} + p_j)$ and learn $A_{ij}$. Therefore, for each downstream task $i$, the AutoScale mixing law for our setting is:
     \begin{align}
         \hat{f}_i(p) = c_i + \sum_{j = 1}^m (R(   A_{ij} + p_j))^{-\alpha_{ij}}.
     \end{align}

     \item Data Mixing Laws: The mixing law models the validation loss on domain $i$ as:
     \begin{align}
         \hat{f}_i(p) = c_i + k_i \exp\bigg(\sum_{j = 1}^m t_{ij} p_j\bigg).\label{eq:dml_original}
     \end{align}

     The paper also extends this to out-of-domain validation sets by expressing them as a linear combination of training domains, namely $\hat{f}(p) = \sum_{i = 1}^m s_i \Big(c_i + k_i \exp \big(\sum_{j = 1}^m t_{ij} p_j\big) \Big)$. We use the first, simpler form~\eqref{eq:dml_original} directly for downstream tasks, treating each task's performance as a function of the mixture on training domains. Moreover, we observe that $k_i$ can be absorbed into the linear term; rewriting \eqref{eq:dml_original} as $c_i + \exp(\sum_{j = 1}^m t_{ij} p_j + \log k_i)$ shows that $k_i$ acts as an intercept term.  Since $p$ sums up to $1$, this intercept is redundant and can be dropped. Therefore, our final log-linear mixing law is
     \begin{align}
         \hat{f}_i(p) = c_i + \exp\bigg(\sum_{j = 1}^m A_{ij} p_j\bigg). \label{eq:app_log_linear}
     \end{align}
     
    \end{itemize}

Next, we provide implementation details for each model:
\begin{itemize}
    \item Search: we select the swarm run that has the best average BPB and satisfies the data repetition constraints.
    \item LightGBM: we use the hyperparameters from RegMix~\citep{liu2025regmixdatamixtureregression}. However, we did not use an evaluation set for early stopping, since we noticed that RegMix used the same set of mixes for early stopping and for reporting performance. We solve the optimization problem using search, as is done in RegMix.
    \item Gaussian Process: we use scikit-learn's GaussianProcessRegressor. We solve the optimization problem using search, as is done in~\citet{chen2025admirebayesoptaccelerateddatamixture}.
    \item BiMix: we use SciPy least squares to fit the regression models. We solve the optimization problem using search, an exact solver, and exact solver with KL regularization ($\lambda = 0.05$) in Table~\ref{tab:bimix_autoscale_solvers}. Based on the results, we decided to use a search-based solver when evaluating BiMix.
    \item AutoScale: we use SciPy least squares to fit the regression models. We solve the optimization problem using search, an exact solver, and exact solver with KL regularization ($\lambda = 0.05$)in Table~\ref{tab:bimix_autoscale_solvers}. Based on the results, we decide to use an exact solver when evaluating Autoscale.
    \item Log-linear: we use the fitting code from~\citet{ye2025datamixinglawsoptimizing}. We solve the optimization problem using CVXPY, while results for other solvers are discussed in \S\ref{sec:optimization}.
\end{itemize}

\begin{table}[t]
\centering
\small
\caption{Downstream performance of BiMix and AutoScale regression models on DCLM topics using exact optimization, exact + KL regularization ($\lambda = 0.05$), and search. Based on these findings, we use an exact solver with Autoscale and a search-based solver for BiMix for the rest of our study.}
\begin{tabular}{lccc}
\toprule
\textbf{Method} & \textbf{Exact (0.0)} & \textbf{Exact + KL (0.05)} & \textbf{Search} \\
\midrule
BiMix      & 0.776733 & 0.790188 & \textbf{0.768153} \\
Autoscale  & \textbf{0.798084} & 0.804814 & 0.808775 \\
\bottomrule
\end{tabular}
\label{tab:bimix_autoscale_solvers}
\end{table}

\subsection{RQ5: Regression granularity}

\textbf{Additional results.} Figure~\ref{fig:regression_granularity_30m} shows the performance of various regression granularities at the 30M scale. While the downstream BPB results at the 1B scale (Table~\ref{tab:granularity}) demonstrates that per-family performs the worst, the 30M performance results are in line with the regression fits of each granularity; per-task performs the best, followed by per-family and aggregated.

\begin{figure}
    \centering
    \includegraphics[width=0.5\linewidth]{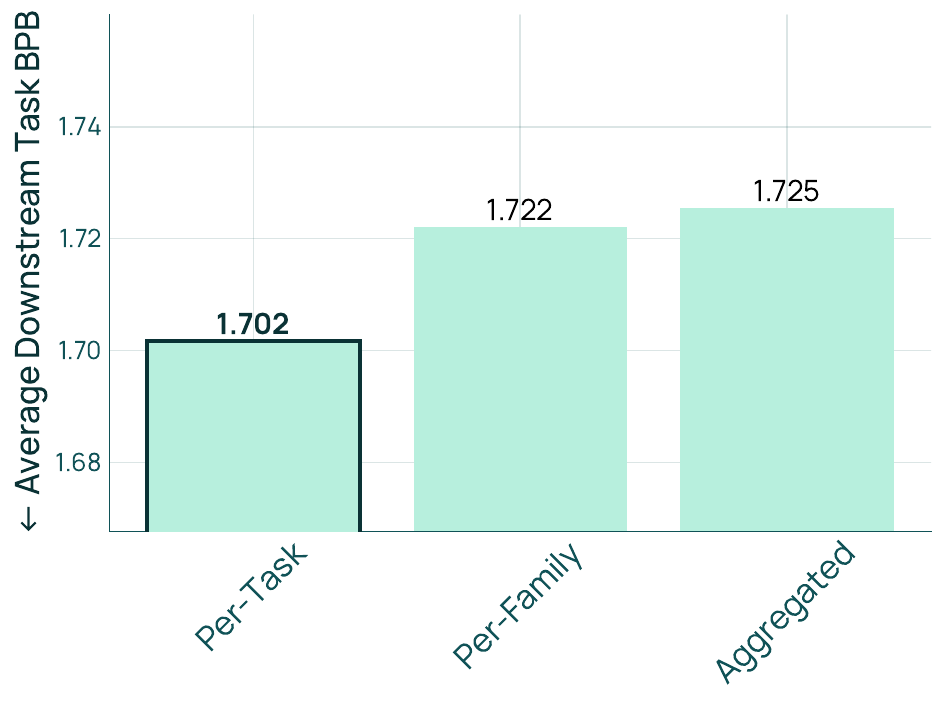}
    \caption{30M performance across regression granularities. The performance is monotonic in the task granularity, matching with the regression fit findings from Table~\ref{tab:granularity}.}
    \label{fig:regression_granularity_30m}
\end{figure}

\textbf{Details.} The task families we used are defined in Table~\ref{tab:eval_suite}: Math, Code, and QA. We constructed a held-out test set containing of 45 samples, consisting of 15 samples from a Dirichlet distribution centered around each regression granularity's respective proposed mix.
This measures regression fit in high-performance regions of the mixture space, which matters in the mixture optimization step.

\subsection{RQ6: Data repetition constraints}

\textbf{Additional results.} Figure~\ref{fig:vary_repetition_factor} is the full version of Figure~\ref{fig:vary_rep_top_7}, showing the weights over all $24$ DCLM topics when we vary the repetition factor $k$.

\begin{figure}
    \centering
    \includegraphics[width=0.6\linewidth]{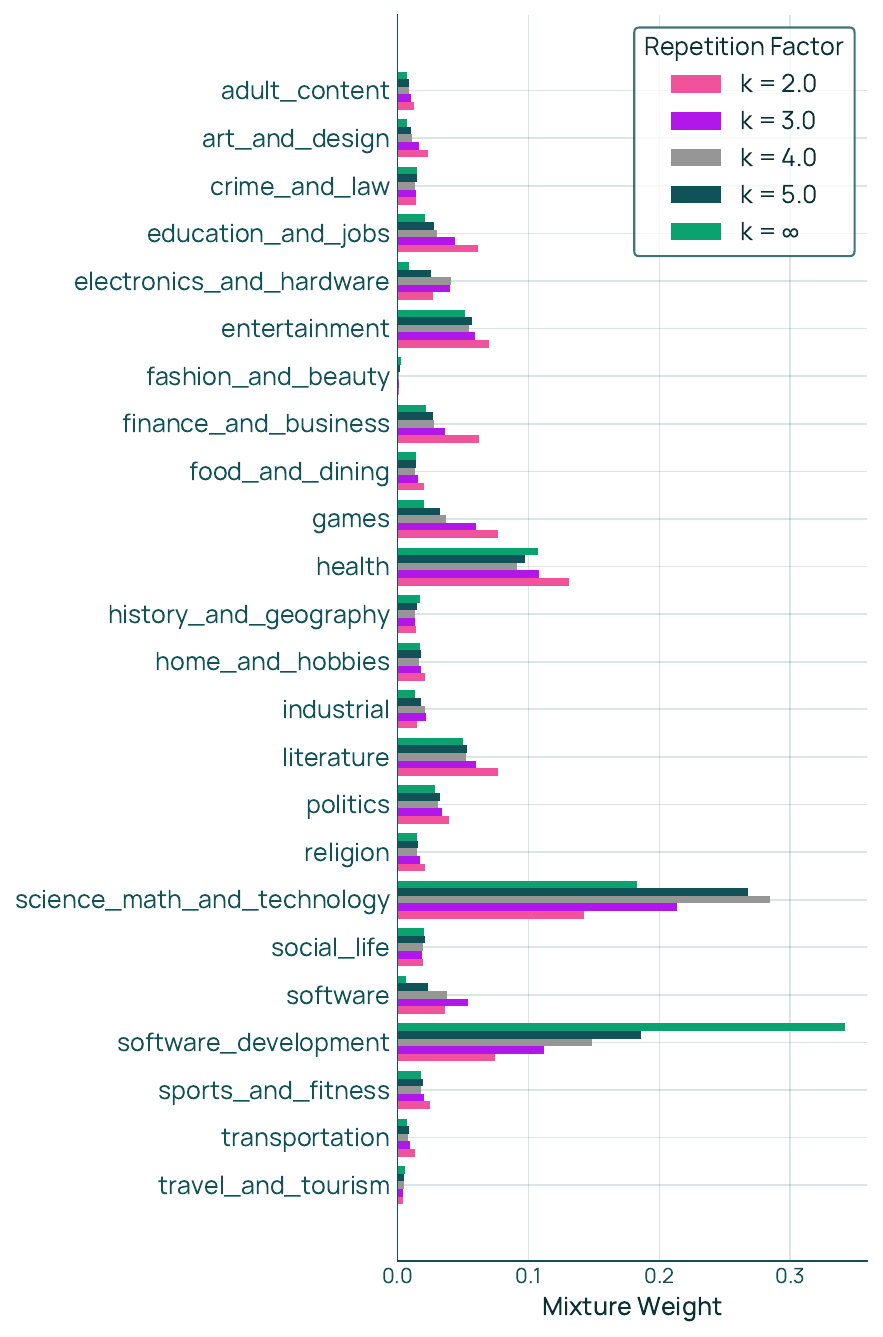}
    \caption{Proposed mixture weights versus repetition factor (all domains), full version of Figure~\ref{fig:vary_rep_top_7}.}
    \label{fig:vary_repetition_factor}
\end{figure}

\textbf{Details.} To construct the constrained swarm, we used $K = 128$ proxy runs. We set $R = 6$T requested tokens and a repetition factor of $k = 4$. The constrained swarm is constructed via rejection sampling from the Dirichlet prior.
For constrained optimization, we specified the repetition constraint in CVXPY.

\subsection{RQ7: Optimization solver}\label{supp:solver}

\textbf{Additional results.} Figure~\ref{fig:solvers_olmo3} shows the performance and predicted performance across optimization solvers when mixing at the source level. Similar to Figure~\ref{fig:solvers}, we see that Exact + KL (0.05) outperforms the other solvers in terms of downstream BPB. We also sanity check our optimizers by examining predicted performance, confirming that the exact optimizer obtains the lowest predicted performance while adding a KL term degrades performance. 

\begin{figure}
    \centering
    \includegraphics[width=0.49\linewidth]{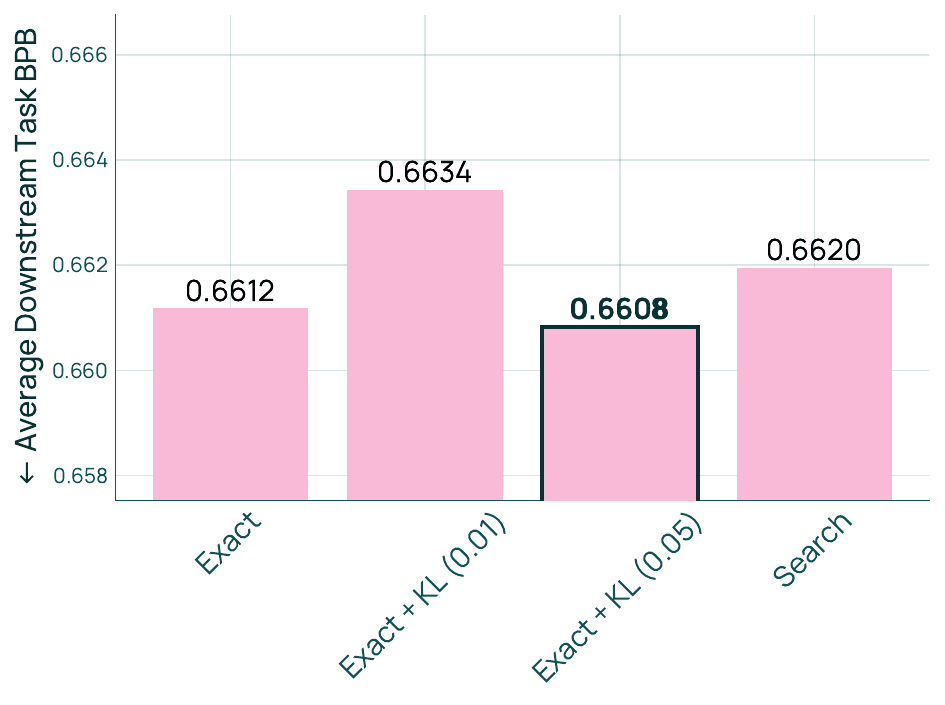}
    \includegraphics[width=0.5\linewidth]{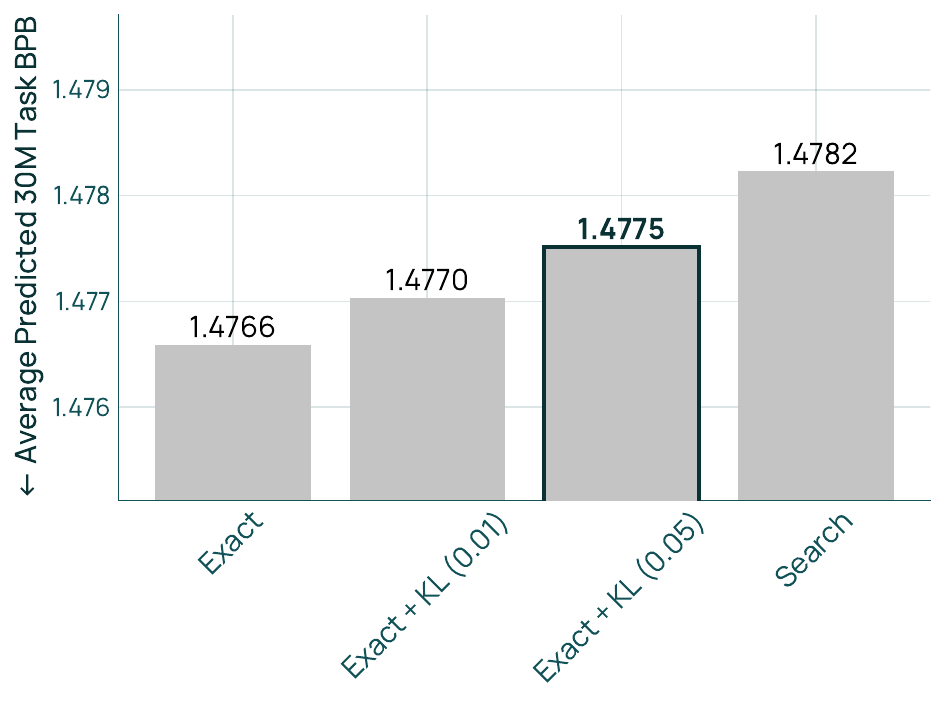}
    \caption{Performance (left) and regression fit (right) across optimization solvers when mixing across sources. The predicted performances of various solvers confirms that the exact solver minimizes the predicted BPB, as expected, followed by KL penalties of $0.01$ and $0.05$. However, the exact solver does not obtain the best downstream performance, and instead some KL regularization helps.}
    \label{fig:solvers_olmo3}
\end{figure}

\textbf{Details.} For the search approach, we used the adaptive search proposed in~\citet{wettig2025organizewebconstructingdomains}, which proceeds in multiple rounds such that the best mix from the current round is used as a Dirichlet prior to generate candidate mixes for the next round.

\section{Additional Mixture Reuse Details}

\subsection{Examples}\label{supp:examples}

Suppose our initial $\mathcal{D}$ contains $m=3$ domains with mixture $\tilde{p} = [0.25,\,0.25,\,0.5]$. We provide examples of mixture reuse for each domain update operator.

\textbf{Remove.} We remove the first domain of $\tilde{p}$. Instead of recomputing a mix over the two remaining domains, we reuse their relative ratios $[0.25,\,0.5]$, automatically producing the mix $q=[0.33,\,0.67]$ without any additional recomputation.  

\textbf{Partition.} We partition the first domain into 2 additional domains, making the total number of domains $m'=4$. Instead of recomputing over all $4$ domains, we learn a mix over a virtual domain consisting of the $2$ unaffected domains, for which $\tilde{p}_{\Dfix} = [0.33, 0.67]$, and the $2$ new domains formed from the partition. This means we only recompute over $3$ domains. If this recomputed mix is $r = [0.6, 0.1, 0.3]$, then the expanded mix is $q = 0.6 \cdot [0.33, 0.67] \cup [0.1, 0.3] = [0.198, 0.402, 0.1, 0.3]$.

\textbf{Revise.} Suppose the first domain is modified. Instead of recomputing over all $3$ domains, we learn a mix over a virtual domain consisting of the $2$ unaffected domains, for which $\tilde{p}_{\Dfix} = [0.33, 0.67]$ and the revised domain. This means we only recompute over $2$ domains. If the recomputed mix is $r=[0.4, 0.6]$, then the expanded mix is $q = 0.4 \cdot [0.33, 0.67] \cup [0.6] = [0.132, 0.268, 0.6]$.

\subsection{Solving the Mixture Reuse Problem}\label{supp:collapsed_space}

We explain how to write~\eqref{eq:conditional_mixing_problem} as a standard mixing problem (without the explicit $q_{\Dfix} = \tilde{p}_{\Dfix}$ constraint) over a \textit{collapsed} space. 

For any feasible $q = [\rho \tilde{p}_{\Dfix}, (1 - \rho) q_{\Dcomp}]$, recall that the corresponding mix in collapsed space is $r = [\rho, (1 - \rho) q_{\Dcomp}] \in \triangle^{|\Dcomp|}$. $r$ is indexed over $\{v\} \cup \Dcomp$, and $\Psi_{\tilde{p}_{\Dfix}}(r)$ is an expansion function that maps from $r$ back to $q$ on $\D'$.

\begin{lemma}[Reparametrized mixture reuse problem]\label{lem:collapsed_space}
The mixture reuse problem~\eqref{eq:conditional_mixing_problem} is equivalent to the following optimization problem over the collapsed space:
\begin{align}
    &\text{minimize}_{r \in \triangle^{|\Dcomp|}}  \frac{1}{n} \sum_{i = 1}^n g_i(r) \\
    &\text{s.t.} \quad r_v \le \min_{j \in \Dfix}\bigg\{\frac{k N_j'}{R \tilde{p}_{j}} \bigg\}, \quad r_j \le \frac{k N_j'}{R} \quad \forall j \in \Dcomp \nonumber
\end{align}
where $g_i(r) := f_i(\LM(S, R, \Phi_{\tilde{p}_{\Dfix}}(r)))$.
\end{lemma}
\begin{proof}
The constraint $q_{\Dfix} = \tilde{p}_{\Dfix}$ is absorbed by the change of variables, so any feasible $q$ corresponds to a unique $r$. The repetition constraints in expanded space, $\rho \tilde{p}_{j} \le \frac{k N_j'}{R}$ for $j \in \Dfix$ and $(1 - \rho) q_j \le \frac{k N_j'}{R}$ for $j \in \Dcomp$ become $\rho \le \frac{k N_j'}{R \tilde{p}_j}$ for all $j \in \Dfix$ and $r_j \le \frac{kN_j'}{R}$ for all $j \in \Dcomp$ in collapsed space.
\end{proof}

\section{Proofs for Section~\ref{sec:analysis}}\label{supp:proofs}

\subsection{Definitions and Notation}\label{supp:defs}
Define performance of $q$ as $F(q) := \frac{1}{n}\sum_{i=1}^n f_i(\LM(S,R,q))$. We also write $F(q)$ as a bivariate function, $F(r,p_1) \;:=\; F(\Phi_{p_1}(r))$. 

Recall $q^\star$ is the optimal solution to~\eqref{eq:mixing_problem} on $\D'$, and $q^\star(\tilde p_{\Dfix})$ is solution returned by \fullmixreuse when reusing $\tilde p_{\Dfix}$ to solve~\eqref{eq:conditional_mixing_problem}. We define $q^\star_{\Dfix}$ as the optimal mix on $\Dfix$ after the domain set update; that is, $q^\star = [\rho^\star q^\star_{\Dfix},\; (1-\rho^\star)q^\star_{\Dcomp}]$. We write $q^\star(\tilde p_{\Dfix}) = [\rho^\star(\tilde p_{\Dfix}) \tilde p_{\Dfix}, (1 - \rho^\star(\tilde p_{\Dfix})) q_{\Dcomp}^\star(\tilde p_{\Dfix})]$.

We can also write $\tilde{p}$ in terms of $\Dfix$ and $\D \setminus \Dfix$: any $p$ can be expressed as $p = [\pi p_{\Dfix}, (1 - \pi) p_{\D \setminus \Dfix}]$, where $\pi \in [0, 1]$ (similar to $\rho$).

The log-linear model is $f_i(q) = c_i + \sum_{i = 1}^n \exp(A_i^\top q)$, where $A_{ij}$ intuitively captures a notion of how much domain $j$ impacts task $i$. We split each $A_i \in \R^{m'}$ into $A_{i, \fix} \in \R^{|\Dfix|}$ and $A_{i,\comp} \in \R^{|\Dcomp|}$ corresponding to $\Dfix$ and $\Dcomp$. Define $\alphafix, \alphacomp \in \R^n$ where $\alpha_{i, \fix} = \|A_{i, \fix}\|, \alpha_{i, \comp} = \|A_{i, \comp} \|$. 

Define the coupling term $\kappa(\Dfix, \Dcomp) = \kappa(\alphafix, \alphacomp)$ as
\begin{align}
    \kappa(\alphafix, \alphacomp) := \|(1 + \alphafix + \alphacomp) \odot \alphafix \|,
\end{align}

where $\odot$ is the Hadamard product.

Define $\bar{F}(\cdot)$ to be the objective without the constant term of the log-linear regression model, $F(\cdot) - \frac{1}{n}\sum_{i =1 }^n c_i$.

Denote $p_1 = \tilde{p}_{\Dfix}$ and $p_1' = q^\star_{\Dfix}$ for simplicity. Let $p_2 = q_{\Dcomp}$. Recall that $r^\star(p_1)$ (and the expanded $q^\star(p_1)$) is the solution to~\eqref{eq:conditional_mixing_problem} when $q_{\Dfix}$ is constrained to be $p_1$. 
Any $r$ can be written as $r = [\rho, (1 - \rho), p_2]$, so we similarly define $\rho^\star(p_1)$ and $p_2^\star(p_1)$.

\subsection{Proof for Theorem~\ref{thm:general_gap}}\label{supp:general_gap_proof}

\begin{assumption}\label{ass:1}
    We make the following assumptions:
    \begin{enumerate}
    
    \item The log-linear model accurately captures the relationship between the mixture ratios and target model performance on each task; that is, $f_i(\LM(S, R, q)) = c_i + \exp(A_i^\top q)$ for some $c_i \in \R^+, A_i \in \R^{m'}$.
    \item Applying \base (Algorithm~\ref{alg:base_method}) exactly solves~\eqref{eq:mixing_problem}, and applying \fullmixreuse exactly solves~\eqref{eq:conditional_mixing_problem}. That is, we ignore (i) proxy-to-target transfer error and (ii) estimation error in learning $c_i$ and $A_i$ from the swarm. However, extending the analysis to include these errors is straightforward by adding corresponding approximation terms.
    \item We assume that $F(r, p_1)$ is $\mu$-strongly convex in $r$.
    \item Let $\mathcal{S}(q_{\Dfix})$ be the feasible set of~\eqref{eq:conditional_mixing_problem} defined by repetition constraints:
    \begin{align*}
        r_v \le \min_{j \in \Dfix} \bigg\{\frac{k N_j'}{R q_j}\bigg\}, \quad r_j \le \frac{k N_j'}{R} \; \forall j \in \D_2',
    \end{align*}
    together with the simplex constraints on $r$. We assume \textbf{mutual feasibility} of $\tilde{p}_{\Dfix}$ and $q^\star_{\Dfix}$: that $r^\star(\tilde{p}_{\Dfix}) \in \mathcal{S}(q^\star_{\Dfix})$ and $r^\star(q^\star_{\Dfix}) \in \mathcal{S}(\tilde{p}_{\Dfix})$. This holds in the following two cases:
    \begin{itemize}
        \item Both $\tilde{p}_{\Dfix}$ and $q^\star_{\Dfix}$ have an active repetition constraint on at least one domain. Intuitively, this means that among the unaffected domains, there exist some high-utility domains that are data-constrained.
        \item $r_v^\star(\tilde{p}_{\Dfix}) \le \min_{j \in \Dfix} \{\frac{k N_j'}{ R q_j} \}$ and $r_v^\star(\tilde{p}_{\Dfix}) \le \min_{j \in \Dfix} \{\frac{k N_j'}{ R \tilde{p}_j} \}$. Intuitively, the contents of the virtual domain are sufficiently low utility such that the proposed mix does not ``push against'' the tighter of the two $r_v$ constraints. 
    \end{itemize}
    This assumption is reasonable in our setting: when the repetition constraint is active, a domain is valuable enough that we would allocate more weight if we could. When it is slack, the domain does not warrant allocating weight up to its availability limit. Our two cases simply rule out the marginal situation where one mix is availability-limited while the other is not.
    \end{enumerate}
\end{assumption}

First, we show that the mixture reuse problem and the standard mixing problem are equivalent when the existing mix on the unaffected domains is equal to the optimal mix on the unaffected domains after the domain update. This is the ``equality'' condition of Theorem~\ref{thm:general_gap}: when $\tilde{p}_{\Dfix} = q^\star_{\Dfix}$, the performance gap is $0$.

\begin{lemma}\label{lemma:conditional_standard_equivalence}
The solution of~\eqref{eq:conditional_mixing_problem}, $r^\star(\tilde{p}_{\Dfix})$, is equivalent to the solution of~\eqref{eq:mixing_problem}, $q^\star = [\rho^\star q^\star_{\Dfix}, (1 - \rho^\star) q^\star_{\D_2'}]$, when the existing mix satisfies $\tilde{p}_{\Dfix} = q^\star_{\Dfix}$. 
\end{lemma}

\begin{proof}
For~\eqref{eq:mixing_problem}, restricting the feasible set to the subspace $q_{\Dfix} = q^\star_{\Dfix}$ does not change the optimal value of this problem, since this subspace still contains the optimal solution to the original problem. Therefore,~\eqref{eq:mixing_problem} is equivalent to:
\begin{align*}
    &\text{minimize}_{q \in \triangle^{m'-1}}  \frac{1}{n} \sum_{i = 1}^m f_i(\LM(S, R, q)) \\
    &\text{s.t.} \quad  q_j \le \frac{k N_j'}{R} \quad \forall j \in [m'] \nonumber \\
    &  \qquad \; q_{\Dfix} = q^\star_{\Dfix} \nonumber 
\end{align*}

Comparing this formulation of~\eqref{eq:mixing_problem} to~\eqref{eq:conditional_mixing_problem}, we can see that the problems are equivalent when $\tilde{p}_{\Dfix} = q^\star_{\Dfix}$.
\end{proof}

Based on this result, we can write $q^\star = q^\star(q_{\Dfix}^\star) = q^\star(p_1')$. Similarly, $\rho^\star = \rho^\star(p_1')$, $r^\star = r^\star(p_1')$, $p_2^\star = p_2^\star(p_1')$. Furthermore, the problem of bounding the performance gap now becomes one of bounding $F(q^\star(\tilde{p}_{\Dfix})) - F(q^\star(q^\star_{\Dfix}))$ in terms of $\|q^\star(\tilde{p}_{\Dfix}) - q^\star_{\Dfix} \|$. We can apply standard tools from optimization and convexity over $F(r, p_1)$ to address this.

\subsubsection{Performance Gap Lemmas}

\begin{lemma}[FOC Inequality]\label{lem:mutual_feas_foc}
Recall that $r^\star(p_1)= \arg\min_{r \in\mathcal S(p_1)} F(r, p_1)$ and $r^\star(p_1')= \arg\min_{r \in\mathcal S(p_1')} F(r, p_1')$. 
Then,
\begin{align*}
\left\langle \triangledown_r F(r^\star(p_1),p_1)-\triangledown_r F(r^\star(p_1'),p_1'),\; r^\star(p_1')-r^\star(p_1)\right\rangle \ge 0.
\end{align*}
\end{lemma}

\begin{proof}
Since $r^\star(p_1)$ minimizes $F(\cdot,p_1)$ over the convex set $\mathcal S(p_1)$, we have the first-order condition
\begin{align*}
\langle \triangledown_r F(r^\star(p_1),p_1), r-r^\star(p_1)\rangle \ge 0 \quad \forall r\in\mathcal S(p_1).
\end{align*}

Using the mutual feasibility assumption, $r^\star(p_1')\in\mathcal S(p_1)$, so plugging $r=r^\star(p_1')$ yields
\begin{align*}
\langle \triangledown_r F(r^\star(p_1),p_1), r^\star(p_1')-r^\star(p_1)\rangle \ge 0.
\end{align*}

Similarly, since $r^\star(p_1')$ minimizes $F(\cdot,p_1')$ over $\mathcal S(p_1')$ and $r^\star(p_1)\in\mathcal S(p_1')$,
\begin{align*}
\langle \triangledown_r F(r^\star(p_1'),p_1'), r^\star(p_1)-r^\star(p_1')\rangle \ge 0.
\end{align*}
Adding the two inequalities completes the proof.
\end{proof}

\begin{lemma}[Mean value bound for $\triangledown_r F$ along a segment]\label{lem:mvt_grad_r}
Fix $r$ and consider the map
\begin{align*}
G(\cdot) := \triangledown_r F(r,\cdot).
\end{align*}
Assume $G$ is continuously differentiable on the line segment
\begin{align*}
p_1^t := t p_1' + (1-t)p_1,\qquad t\in[0,1].
\end{align*}
Then
\begin{align*}
\left\|\triangledown_r F(r,p_1')-\triangledown_r F(r,p_1)\right\|
\le
\sup_{t\in[0,1]}\left\|\triangledown^2_{r,p_1}F(r,p_1^t)\right\|\cdot \|p_1'-p_1\|.
\end{align*}
\end{lemma}

\begin{proof}
Define the scalar function
\begin{align*}
h(t) := G(p_1^t),\qquad t\in[0,1].
\end{align*}
Since $G$ is continuously differentiable on the segment, $h$ is differentiable and
\begin{align*}
h'(t) = \triangledown_{p_1}G(p_1^t)\,(p_1'-p_1),
\end{align*}
where $\triangledown_{p_1}G(p_1^t)$ denotes the Jacobian of $G$ at $p_1^t$.
By the fundamental theorem of calculus,
\begin{align*}
G(p_1')-G(p_1) = h(1)-h(0) = \int_0^1 h'(t)\,dt
= \int_0^1 \triangledown_{p_1}G(p_1^t)\,(p_1'-p_1)\,dt.
\end{align*}
Taking norms and applying the triangle inequality,
\begin{align*}
\|G(p_1')-G(p_1)\|
&\le \int_0^1 \|\triangledown_{p_1}G(p_1^t)\,(p_1'-p_1)\|\,dt \\
&\le \int_0^1 \|\triangledown_{p_1}G(p_1^t)\|\,\|p_1'-p_1\|\,dt \\
&\le \left(\sup_{t\in[0,1]}\|\triangledown_{p_1}G(p_1^t)\|\right)\cdot \|p_1'-p_1\|.
\end{align*}
Substituting $G(\cdot)=\triangledown_r F(r,\cdot)$ completes the proof.
\end{proof}

\begin{lemma}[Behavior of $r^\star(\cdot)$ under changes in $p_1$]\label{lem:stability}
Let $\Delta := p_1'-p_1$ and $\Delta r := r^\star(p_1')-r^\star(p_1)$. 

Then
\begin{align*}
\|\Delta r\| \;\le\; \frac{M}{\mu}\,\|\Delta\|,
\end{align*}
where
\begin{align*}
M \;:=\; \sup_{t\in[0,1]} \left\| \triangledown^2_{r,p_1} F\big(r^\star(p_1'),\, p_1^t\big)\right\|,\quad
p_1^t := t p_1' + (1-t)p_1.
\end{align*}
\end{lemma}

\begin{proof}
By $\mu$-strong convexity of $F(\cdot,p_1)$ in $r$, for $r:=r^\star(p_1)$ and $r':=r^\star(p_1')$,
\begin{align*}
F(r',p_1)&\ge F(r,p_1)+\langle \triangledown_r F(r,p_1), r'-r\rangle +\frac{\mu}{2}\|r'-r\|^2,\\
F(r,p_1)&\ge F(r',p_1)+\langle \triangledown_r F(r',p_1), r-r'\rangle +\frac{\mu}{2}\|r-r'\|^2.
\end{align*}
Adding yields
\begin{align*}
\mu\|\Delta r\|^2 \le \left\langle \triangledown_r F(r^\star(p_1'),p_1)-\triangledown_r F(r^\star(p_1),p_1),\; \Delta r\right\rangle.
\end{align*}
Next, we add and subtract $\triangledown_r F(r^\star(p_1'),p_1')$:
\begin{align*}
\mu\|\Delta r\|^2
&\le
\left\langle \triangledown_r F(r^\star(p_1'),p_1)-\triangledown_r F(r^\star(p_1'),p_1'),\; \Delta r\right\rangle \\
&\quad+
\left\langle \triangledown_r F(r^\star(p_1'),p_1')-\triangledown_r F(r^\star(p_1),p_1),\; \Delta r\right\rangle.
\end{align*}
By Lemma~\ref{lem:mutual_feas_foc}, the second inner product is less than $0$. Thus
\begin{align*}
\mu\|\Delta r\|^2
\le
\left\|\triangledown_r F(r^\star(p_1'),p_1)-\triangledown_r F(r^\star(p_1'),p_1')\right\|\cdot \|\Delta r\|.
\end{align*}

Finally, applying Lemma~\ref{lem:mvt_grad_r}, we have
\begin{align*}
    \mu\|\Delta r\|^2 \le \sup_{t\in[0,1]}\left\|\triangledown^2_{r,p_1}F(r^\star(p_1'),p_1^t)\right\| \cdot \|\Delta\| \cdot \| \Delta r \|
\end{align*}

We divide both sides by $\|\Delta r\|$ to complete the proof.
\end{proof}

\begin{lemma}[Decomposing the performance gap]\label{lem:gap_decomp}
Define $\Delta:=p_1'-p_1$, $\Delta r:=r^\star(p_1')-r^\star(p_1)$. Then
\begin{align*}
F(r^\star(p_1),p_1)-F(r^\star(p_1'),p_1')
&\le
\|\triangledown_r F(r^\star(p_1),p_1)\|\cdot \|\Delta r\|
+
\|\Delta\|\cdot \sup_{t\in[0,1]}\|\triangledown_{p_1}F(r^\star(p_1'),p_1^t)\|.
\end{align*}
\end{lemma}

\begin{proof}
Add and subtract $F(r^\star(p_1'),p_1)$ from the LHS:
\begin{align*}
F(r^\star(p_1),p_1)-F(r^\star(p_1'),p_1')
=
\big(F(r^\star(p_1),p_1)-F(r^\star(p_1'),p_1)\big)
+
\big(F(r^\star(p_1'),p_1)-F(r^\star(p_1'),p_1')\big).
\end{align*}
For the first difference, convexity of $F(\cdot,p_1)$ implies
\begin{align*}
F(r^\star(p_1'),p_1)\ge F(r^\star(p_1),p_1)+\langle \triangledown_r F(r^\star(p_1),p_1), r^\star(p_1')-r^\star(p_1)\rangle,
\end{align*}
so
\begin{align*}
F(r^\star(p_1),p_1)-F(r^\star(p_1'),p_1)
\le
\langle \triangledown_r F(r^\star(p_1),p_1), r^\star(p_1)-r^\star(p_1')\rangle
\le
\|\triangledown_r F(r^\star(p_1),p_1)\|\cdot \|\Delta r\|.
\end{align*}
For the second difference, we apply the mean value theorem along $p_1^t$:
\begin{align*}
|F(r^\star(p_1'),p_1)-F(r^\star(p_1'),p_1')|
\le
\|\Delta\|\cdot \sup_{t\in[0,1]}\|\triangledown_{p_1}F(r^\star(p_1'),p_1^t)\|.
\end{align*}
Combining the two bounds completes our proof.
\end{proof}

\subsubsection{Gradient Lemmas}

\begin{lemma} \label{lemma:grad_r}
    The gradient norm $\|\triangledown_r F(r^\star(p_1),p_1)\|$ can be bounded by 
    \begin{align*}
        \triangledown_r F(r^\star(p_1), p_1) \le \bar{F}(q^\star(p_1))  \big\| | A_1 p_1 | + \alphacomp \big\|.
    \end{align*}
\end{lemma}

\begin{proof}

First, we write $r$ as $[\rho, b_2]$, where $b_2 = (1 - \rho)p_2$. Define $s_i = A_{i, 1}^\top \rho p_1 + A_{i, 2}^\top b_2$. We compute the gradient of $F(r, p_1)$ with respect to $\rho$ and $b_2$:
\begin{align*}
    \triangledown_{\rho} F(r, p_1) &= \frac{1}{n} \sum_{i = 1}^n \exp(s_i) A_{i, 1}^\top p_1 \\
    \triangledown_{b_2} F(r, p_1) &= \frac{1}{n} \sum_{i = 1}^n \exp(s_i) A_{i, 2}
\end{align*}

We can consolidate this into a single expression if we define $a_i = \SmallMatrix{ \langle A_{i, 1}, p_1 \rangle \\ A_{i, 2}} \in \R^{|\D_2'| + 1}$:
\begin{align*}
    \triangledown_{r} F(r, p_1) &= \frac{1}{n} \sum_{i = 1}^n \exp(s_i) a_i.
\end{align*}

Let us write $r^\star(p_1)$ as $[\rho^\star(p_1), (1 - \rho^\star(p_1)) p_2^\star(p_1)]$. The expression for the gradient at $r^\star(p_1)$ is
\begin{align*}
    \triangledown_r F(r^\star(p_1), p_1) &= \frac{1}{n}\sum_{i = 1}^n \exp(A_{i, 1}^\top \rho^\star(p_1) p_1 + A_{i, 2}^\top (1 - \rho^\star(p_1)) p_2^\star(p_1)) a_i.
\end{align*}

We now bound the norm. First, applying the Cauchy-Schwarz inequality twice, we have:
\begin{align*}
    \|\triangledown_r F(r^\star(p_1), p_1)\| &\le \frac{1}{n} \sum_{i = 1}^n \exp(A_{i, 1}^\top \rho^\star(p_1) p_1 + A_{i, 2}^\top (1 - \rho^\star(p_1)) p_2^\star(p_1)) \|a_i\| \\
    &\le \frac{1}{n} \sqrt{\sum_{i = 1}^n \exp(2(A_{i, 1}^\top \rho^\star(p_1) p_1 + A_{i, 2}^\top (1 - \rho^\star(p_1)) p_2^\star(p_1)))} \sqrt{\sum_{i = 1}^n \| a_i\|^2} \\
    &\le \frac{1}{n} \sum_{i = 1}^n \exp(A_{i, 1}^\top \rho^\star(p_1) p_1 + A_{i, 2}^\top (1 - \rho^\star(p_1)) p_2^\star(p_1)) \sqrt{\sum_{i = 1}^n \| a_i\|^2}.
\end{align*}

Observe that the first summation is equivalent to $\bar{F}(q^\star(p_1))$. Then, to bound the remaining term $\sqrt{\sum_{i = 1}^n \| a_i\|^2}$, we have
\begin{align*}
\sqrt{\sum_{i = 1}^n \| a_i\|^2} &= \sqrt{\sum_{i = 1}^n \langle A_{i, 1}, p_1 \rangle^2 + \| A_{i, 2}\|^2 } \\
&\le \sqrt{\sum_{i = 1}^n (|\langle A_{i, 1}, p_1 \rangle| + \| A_{i, 2}\|)^2 } \\
&= \big\| | A_1 p_1 | + \alphacomp \big\|.
\end{align*}

where $| \cdot |$ is an elementwise absolute value, and $\alphacomp \in \R^n$ is constructed where $\alpha_{i, 2} = \|A_{i, 2}\|$. 

Therefore,
\begin{align*}
    \triangledown_r F(r^\star(p_1), p_1) \le \bar{F}(q^\star(p_1))  \big\| | A_1 p_1 | + \alphacomp \big\|.
\end{align*}

\end{proof}

\begin{lemma}\label{lemma:grad_p_1}
    The gradient norm $\triangledown_{p_1} F(r^\star(p_1'), p_1^t)$ can be bounded by 
    \begin{align*}
        &\|\triangledown_{p_1} F(r^\star(p_1'), p_1^t)\| \le  \bar{F}(q^\star(p_1'))  \big\|  \exp( \alphafix \|\triangle \|) \odot \alphafix   \big\|,
    \end{align*}
    where $\odot$ is the Hadamard product, and $\exp$ is elementwise.
\end{lemma}

\begin{proof}

Define $s_i = A_{i, 1}^\top \rho p_1 + A_{i, 2}^\top (1 - \rho) p_2$. We compute the gradient of $F(r, p_1)$ with respect to $p_1$:
\begin{align*}
    \triangledown_{p_1} F(r, p_1) = \frac{1}{n}\sum_{i = 1}^n \exp(s_i) \rho A_{i, 1}.
\end{align*}

Let us write $r^\star(p_1')$ as $[\rho^\star(p_1'), (1 - \rho^\star(p_1')) p_2^\star(p_1')]$. Then, the gradient of $F(r, p_1)$ at $r^\star(p_1'), p_1^t$ is
\begin{align*}
    \triangledown_{p_1} F(r^\star(p_1'), p_1^t) &= \frac{1}{n} \sum_{i = 1}^n \exp(A_{i, 1}^\top \rho^\star(p_1') p_1^t + A_{i, 2}^\top (1 - \rho^\star(p_1')) p_2^\star(p_1')) \rho^\star(p_1') A_{i, 1}.
\end{align*}

We now bound the norm of the gradient over $t \in [0, 1]$. First, we add and subtract $p_1'$ to create an $\exp$ term defined over $p_1'$ and a difference term with $p_1^t - p_1'$:
\begin{align*}
    &\|\triangledown_{p_1} F(r^\star(p_1'), p_1^t)\| \le \frac{1}{n} \sum_{i = 1}^n \exp(A_{i, 1}^\top \rho^\star(p_1') (p_1' + p_1^t - p_1') + A_{i, 2}^\top (1 - \rho^\star(p_1')) p_2^\star(p_1')) \rho^\star(p_1')  \| A_{i, 1} \| \\
    &\le \frac{1}{n} \sum_{i = 1}^n \exp(A_{i, 1}^\top \rho^\star(p_1') p_1' + A_{i, 2}^\top (1 - \rho^\star(p_1')) p_2^\star(p_1')) \exp(A_{i, 1}^\top \rho^\star(p_1') (p_1^t - p_1'))  \| A_{i, 1} \|.
\end{align*}

Then, we apply Cauchy-Schwarz inequality to get
\begin{align*}
    &\|\triangledown_{p_1} F(r^\star(p_1'), p_1^t)\| \\
    &\le \frac{1}{n} \sqrt{\sum_{i = 1}^n \exp(2 A_{i, 1}^\top \rho^\star(p_1') (p_1^t - p_1'))  \| A_{i, 1} \|^2} \sqrt{\sum_{i = 1}^n \exp(2(A_{i, 1}^\top \rho^\star(p_1') p_1' + A_{i, 2}^\top (1 - \rho^\star(p_1')) p_2^\star(p_1')))} \\
    &\le \frac{1}{n} \sqrt{\sum_{i = 1}^n \exp(2 A_{i, 1}^\top \rho^\star(p_1') (p_1^t - p_1'))  \| A_{i, 1} \|^2} \cdot \sum_{i = 1}^n \exp(A_{i, 1}^\top \rho^\star(p_1') p_1' + A_{i, 2}^\top (1 - \rho^\star(p_1')) p_2^\star(p_1')).
\end{align*}

We observe that we can replace the second summation with $\bar{F}(q^\star(p_1'))$, so we only need to further simplify $\sqrt{\sum_{i = 1}^n \exp(2 A_{i, 1}^\top \rho^\star(p_1') (p_1^t - p_1'))  \| A_{i, 1} \|^2}$. We observe that $\|p_1^t - p_1'\|$ is always bounded by $\|\Delta\|$, and use the fact that $\rho^\star(p_1') \le 1$:
\begin{align*}
    & \sqrt{\sum_{i = 1}^n \exp(2 A_{i, 1}^\top \rho^\star(p_1') (p_1^t - p_1'))  \| A_{i, 1} \|^2} \le \sqrt{\sum_{i = 1}^n \exp(2 \|A_{i, 1} \| \|p_1^t - p_1' \|)  \| A_{i, 1} \|^2} \\
    &\le \sqrt{\sum_{i = 1}^n \exp(2 \|A_{i, 1} \| \|\triangle \|)  \| A_{i, 1} \|^2} \le \big\|  \exp( \alphafix \|\triangle \|) \odot \alphafix   \big\|.
\end{align*}

Therefore,
\begin{align*}
    &\|\triangledown_{p_1} F(r^\star(p_1'), p_1^t)\| \le  \bar{F}(q^\star(p_1'))  \big\|  \exp( \alphafix \|\triangle \|) \odot \alphafix  \big\|.
\end{align*}
\end{proof}

\begin{lemma}\label{lemma:grad_r_p_1}
Let $a_{\max} = \max_i \{ \|A_{i, 1}\|\}$, the largest norm of $A_{i, 1}$ across all tasks. Define $b_{\max} \in \R^n$ where $b_{i, \max} = \max\{ |\langle A_{i, 1}, p_1 \rangle|, |\langle A_{i, 1}, p_1' \rangle| \}$ as the larger of the two dot products per task.  The gradient norm $\|\triangledown^2_{r,p_1}F(r^\star(p_1'),p_1^t)\|$ can be bounded by 
    \begin{align*}
\|\triangledown^2_{r,p_1}F(r^\star(p_1'),p_1^t)\| &\le\exp( a_{\max} \| \Delta \| ) \cdot  \bar{F}(q^\star(p_1')) \cdot \big\| (\mathbf{1} + b_{\max} + \alphacomp) \odot \alphafix \big\|,
    \end{align*}
    
where $\odot$ is the Hadamard product.
\end{lemma}

\begin{proof}
We derive an expression for $\triangledown_{r, p_1}^2 F(r, p_1)$. 
We write $r$ as $[\rho, b_2]$, where $b_2 := (1 - \rho) p_2$, and define $s_i = A_{i, 1}^\top \rho p_1 + A_{i, 2}^\top b_2$ to be the expression inside the $\exp$. First, we have that $\triangledown_{p_1} F(r, p_1)$ is
\begin{align*}
    \triangledown_{p_1} F(r, p_1) = \frac{1}{n} \sum_{i = 1}^n \exp(s_i) \rho A_{i, 1}.  \nonumber
\end{align*}

Then, we can compute $\triangledown_{\rho, p_1} F(r, p_1)$ and $\triangledown_{b_2, p_1} F(r, p_1)$:
\begin{align*}
    \triangledown_{\rho, p_1} F(r, p_1) &= \triangledown_{\rho} \bigg(\frac{1}{n}\sum_{i = 1}^n \exp(s_i) \rho A_{i, 1} \bigg) = \frac{1}{n} \sum_{i= 1}^n (\rho \exp(s_i) \langle A_{i, 1}, p_1 \rangle + \exp(s_i)) A_{i, 1}^\top \nonumber \\
    \triangledown_{b_2, p_1} F(r, p_1) &= \triangledown_{b_2} \bigg(\frac{1}{n}\sum_{i = 1}^n \exp(s_i) \rho A_{i, 1} \bigg) = \frac{1}{n} \sum_{i = 1}^n \rho \exp(s_i) A_{i, 2} A_{i, 1}^\top \nonumber
\end{align*}

We can consolidate these into
\begin{align*}
    \triangledown_{r, p_1}^2 F(r, p_1) = \frac{1}{n}\sum_{i = 1}^n \exp(s_i) (\rho a_i + e_1) A_{i, 1}^\top, \nonumber
\end{align*}

where $a_i = \SmallMatrix{\langle A_{i, 1}, p_1 \rangle \\ A_{i, 2}}  \in \R^{|\D_2'| + 1}$ and $e_1$ is the standard basis vector $\SmallMatrix{1 \\ 0 \\ \vdots \\ 0} \in \R^{|\D_2'| + 1}$. Therefore, the gradient $\triangledown_{r, p_1}^2 F(r, p_1)$ at $r^\star(p_1'), p_1^t$ is
\begin{align*}
    \triangledown_{r, p_1}^2 F(r^\star(p_1'), p_1^t) = \frac{1}{n}\sum_{i = 1}^n \exp(A_{i, 1}^\top \rho^\star(p_1') p_1^t + A_{i, 2}^\top (1 - \rho^\star(p_1')) p_2^\star(p_1')) (\rho^\star(p_1') a_i + e_1) A_{i, 1}^\top,
\end{align*}

where $a_i$ is now $\SmallMatrix{\langle A_{i, 1}, p_1^t \rangle \\ A_{i, 2}}$. Now, we would like to upper bound $ \sup _{t \in [0, 1]} \|\triangledown_{r, p_1}^2 F(r^\star(p_1'), p_1^t) \|$. Using Cauchy-Schwarz inequality, we have
\begin{align}
    &\|\triangledown_{r, p_1}^2 F(r^\star(p_1'), p_1^t) \| \le \frac{1}{n}\sum_{i = 1}^n \exp(A_{i, 1}^\top \rho^\star(p_1') p_1^t + A_{i, 2}^\top (1 - \rho^\star(p_1')) p_2^\star(p_1')) \|(\rho^\star(p_1') a_i + e_1) A_{i, 1}^\top \| \nonumber \\
    &\le \frac{1}{n} \sqrt{\sum_{i = 1}^n \exp\big(2 (A_{i, 1}^\top \rho^\star(p_1') p_1^t + A_{i, 2}^\top (1 - \rho^\star(p_1')) p_2^\star(p_1'))\big)} \sqrt{\sum_{i=1}^n \|(\rho^\star(p_1') a_i + e_1) A_{i, 1}^\top \|^2}. \label{eq:m_bound}
\end{align}

We simplify the first summation in~\eqref{eq:m_bound} by adding and subtracting $p_1'$ to $p_1^t$:
\begin{align*}
    \sum_{i = 1}^n &\exp\big(2 (A_{i, 1}^\top \rho^\star(p_1') p_1^t + A_{i, 2}^\top (1 - \rho^\star(p_1')) p_2^\star(p_1'))\big) \\
    &= \sum_{i = 1}^n \exp\big(2 A_{i, 1}^\top \rho^\star(p_1') (p_1' + p_1^t - p_1') +  2 A_{i, 2}^\top (1 - \rho^\star(p_1')) p_2^\star(p_1')\big) \\
    &= \sum_{i = 1}^n \exp\big(2 (A_{i, 1}^\top \rho^\star(p_1') p_1' + A_{i, 2}^\top (1 - \rho^\star(p_1')) p_2^\star(p_1')) + 2 A_{i,1}^\top \rho^\star(p_1') ( p_1^t - p_1')\big) \\
    &\le \max_i \{\exp(2 A_{i,1}^\top \rho^\star(p_1') ( p_1^t - p_1')) \} \bigg(\sum_{i = 1}^n \exp\big(A_{i, 1}^\top \rho^\star(p_1') p_1' + A_{i, 2}^\top (1 - \rho^\star(p_1')) p_2^\star(p_1')\big) \bigg)^2.
\end{align*}

We note that $\sum_{i = 1}^n \exp\big(A_{i, 1}^\top \rho^\star(p_1') p_1' + A_{i, 2}^\top (1 - \rho^\star(p_1')) p_2^\star(p_1')\big)$ is $n \cdot \bar{F}(q^\star(p_1'))$.  Moreover, using the fact that $\rho^\star(p_1') \le 1$ and $\|p_1^t - p_1' \|\le \| \Delta\|$, we can bound $\max_i \{\exp(2 A_{i,1}^\top \rho^\star(p_1') ( p_1^t - p_1')) \} \le \max_i \{\exp( 2 \|A_{i, 1}\| \| p_1^t - p_1' \|) \} \le \exp(2  \| \Delta \| \max_i \{ \|A_{i, 1}\| \})$. Let $a_{\max} = \max_i \{ \|A_{i, 1}\|\}$, the largest norm of $A_{i, 1}$ across all tasks. Therefore, we have:
\begin{align*}
    &\frac{1}{n}\sqrt{\sum_{i = 1}^n \exp\big(2 (A_{i, 1}^\top \rho^\star(p_1') p_1^t + A_{i, 2}^\top (1 - \rho^\star(p_1')) p_2^\star(p_1'))\big)} \le \frac{1}{n} \sqrt{\exp(2a_{\max}  \| \Delta \| ) \cdot  n^2 \cdot F(q^\star(p_1'))^2} \\
    &= \exp( a_{\max} \| \Delta \| ) \cdot  F(q^\star(p_1')).
\end{align*}

Next, we simplify the second summation in~\eqref{eq:m_bound}, $\sum_{i=1}^n \|(\rho^\star(p_1') a_i + e_1) A_{i, 1}^\top \|^2$, using Cauchy-Schwarz inequality, the definition of the norm, and the fact that $\rho^\star(p_1') \le 1$:
\begin{align*}
    &\sum_{i=1}^n \|(\rho^\star(p_1') a_i + e_1) A_{i, 1}^\top \|^2 \le \sum_{i=1}^n \|\rho^\star(p_1') a_i + e_1 \|^2 \| A_{i, 1} \|^2 \\
    &= \sum_{i=1}^n ((1 + \rho^\star(p_1') \langle A_{i, 1}, p_1^t \rangle)^2 + \rho^\star(p_1')^2 \| A_{i, 2}\|^2)  \| A_{i, 1} \|^2 \\
    &\le \sum_{i=1}^n ((1 + \rho^\star(p_1') \langle A_{i, 1}, p_1^t \rangle)^2 + \| A_{i, 2}\|^2)  \| A_{i, 1} \|^2.
\end{align*}

Note that $\langle A_{i, 1}, p_1^t \rangle$ is bound by $\langle A_{i, 1}, p_1' \rangle$ and $\langle A_{i, 1}, p_1 \rangle$. Therefore, it holds that $(1 + \rho^\star(p_1') \langle A_{i, 1}, p_1^t \rangle)^2 \le (1 + \max\{ |\langle A_{i, 1}, p_1 \rangle|, |\langle A_{i, 1}, p_1' \rangle| \})^2$. We denote $b_{i, \max} = \max\{ |\langle A_{i, 1}, p_1 \rangle|, |\langle A_{i, 1}, p_1' \rangle| \}$ as the larger of the two dot products per task, and then we can write:
\begin{align*}
    &\sqrt{\sum_{i=1}^n \|(\rho^\star(p_1') a_i + e_1) A_{i, 1}^\top \|^2} \le \sqrt{\sum_{i = 1}^n ((1 + b_{i, \max})^2 + \|A_{i, 2} \|^2) \| A_{i, 1}\|^2} \\
    &\le \sqrt{\sum_{i = 1}^n (1 + b_{i, \max} + \|A_{i, 2} \|)^2 \| A_{i, 1}\|^2}  \\
    &= \big\| (\mathbf{1} + b_{\max} + \alphacomp) \odot \alphafix \big\|.
\end{align*}

where addition is elementwise, $\odot$ is the Hadamard product, and $\alphafix, \alphacomp \in \R^n$ are constructed where $\alpha_{i, 1} = \|A_{i, 1}\|$ and $\alpha_{i, 2} = \|A_{i, 2}\|$. Putting everything together, our gradient norm bound in~\eqref{eq:m_bound} becomes
\begin{align*}
    \|\triangledown^2_{r,p_1}F(r^\star(p_1'),p_1^t)\| &\le\exp( a_{\max} \| \Delta \| ) \cdot  \bar{F}(q^\star(p_1')) \cdot \big\| (\mathbf{1} + b_{\max} + \alphacomp) \odot \alphafix \big\|.
\end{align*}
\end{proof}

\subsubsection{Theorem~\ref{thm:general_gap} (Formal)}

\begin{theorem}\label{thm:gap_lipschitz}
Define $a_{\max} = \max_i \{\|A_{i, 1} \|\}$. The performance gap between using $q^\star$ and $q^\star(\tilde{p}_{\Dfix})$ is bounded by:
\begin{align*}
F(q^\star(\tilde p_{\Dfix}))-F(q^\star) \le C \|\tilde{p}_{\Dfix}-q^\star_{\Dfix}\|,
\end{align*}
where 
\begin{align*}
    C = \bar{F}(q^\star) \exp(a_{\max} \|\tilde{p}_{\Dfix}-q^\star_{\Dfix}\|) \left( \frac{\bar{F}(q^\star(\tilde{p}_{\Dfix}))}{\mu} \cdot \|\alphafix + \alphacomp \| \cdot \kappa(\alphafix, \alphacomp)   + \|\alphafix \| \right).
\end{align*}
\end{theorem}

\begin{proof}

Recall our simplified notation $p_1 = \tilde{p}_{\Dfix}$, $p_1' = q^\star_{\Dfix}$.
Applying Lemma~\ref{lem:stability} to Lemma~\ref{lem:gap_decomp}, we can bound the performance gap to be linear in $\| p_1 - p_1' \|$:
\begin{align*}
    F(r^\star(p_1),p_1)-F(r^\star(p_1'),p_1') \le C \|\Delta\|,
\end{align*}

where $C$ is a term that consists of three gradient norms:
\begin{align*}
        C = \frac{1}{\mu} \sup_{t\in[0,1]} \left( \|\triangledown^2_{r,p_1}F(r^\star(p_1'),p_1^t)\| \right) \,\|\triangledown_r F(r^\star(p_1),p_1)\| + \sup_{t\in[0,1]}\|\triangledown_{p_1}F(r^\star(p_1'),p_1^t)\|,
\end{align*}
and $p_1^t=t p_1'+(1-t)p_1$. We can apply the gradient norm bounds from Lemmas~\ref{lemma:grad_r}, \ref{lemma:grad_p_1}, \ref{lemma:grad_r_p_1}:

\begin{align*}
    C &\le \frac{1}{\mu} \exp( a_{\max} \| \Delta \| ) \cdot  \bar{F}(q^\star(p_1')) \cdot \big\| (\mathbf{1} + b_{\max} + \alphacomp) \odot \alphafix \big\| \cdot \bar{F}(q^\star(p_1))  \big\| | A_1 p_1 | + \alphacomp \big\| \\
    &+ \bar{F}(q^\star(p_1'))  \big\|  \exp( \alphafix \|\triangle \|) \odot \alphafix   \big\|.
\end{align*}

We simplify this bound a bit further. First, note that the elements of $|A_1 p_1|$ are $| \langle A_{i, 1}, p_1 \rangle| \le \|A_{i, 1}\| \| p_1\| \le \|A_{i, 1}\|$. Therefore, $|A_1 p_1| \preceq \alphafix$.

Second, note that $b_{i, \max} = \max\{ |\langle A_{i, 1}, p_1 \rangle|, |\langle A_{i, 1}, p_1' \rangle| \} \le \|A_{i, 1}\|$, so $b_{\max} \preceq \alphafix$ as well.

Lastly, $\big\|  \exp( \alphafix \|\triangle \|) \odot \alphafix   \big\| $ can be bounded by $\exp(a_{\max} \|\Delta\|) \cdot \big\|  \alphafix \big\| $.

Then, we can bound $C$ as
\begin{align*}
    C \le \bar{F}(q^\star) \exp(a_{\max} \| \Delta \|) \left( \frac{\bar{F}(q^\star(\tilde{p}_{\Dfix}))}{\mu} \|(\mathbf{1} + \alphafix + \alphacomp) \odot \alphafix \| \cdot \|\alphafix + \alphacomp \|  + \|\alphafix \| \right).
\end{align*}

Replacing $\|(\mathbf{1} + \alphafix + \alphacomp) \odot \alphafix \|$ with $\kappa(\alphafix, \alphacomp)$ completes our proof.

\end{proof}

\subsection{Proof for Theorem~\ref{thm:add_domains}}\label{supp:add_domains_proof}

When domains are being added and $\tilde{p}$ is the optimal mix on $\D$, $\tilde{p}$ is equal to $\tilde{p}_{\Dfix}$ and is the solution to the following optimization problem:
\begin{align}
    &\text{minimize}_{p \in \triangle^{m-1}}  \frac{1}{n} \sum_{i = 1}^m f_i(\LM(S, R, p))
    \label{eq:p_mixing_problem} \\
    &\text{s.t.} \quad  p_j \le \frac{k N_j}{R} \quad \forall j \in [m] \nonumber
\end{align}

\begin{assumption}\label{ass:2}
We make the following assumptions beyond Assumption~\ref{ass:1}.
\begin{enumerate}
    \item We assume that $F(r, p_1)$ is $\mu_1$-strongly convex in $p_1$.
    \item Let $\mathcal{S}(r)$ be the feasible set of~\eqref{eq:mixing_problem} defined by the repetition constraints on $p_1$:
    \begin{align*}
        p_j \le \frac{k N_j}{R \rho} \; \forall j \in \Dfix
    \end{align*}
    We assume \textbf{mutual feasibility} of $\tilde{p}_{\Dfix}$ and $q^\star_{\Dfix}$:
    that $\tilde{p}_{\Dfix} \in \mathcal{S}(r^\star)$, and $q^\star_{\Dfix} \in \mathcal{S}(r_0)$, where $r_0 = [1, \mathbf{0}]$ corresponds to the mix before new domains are added. That is, we assume that $\tilde{p}_j \le \frac{k N_j}{R \rho^\star}$ and $q^\star_j \le \frac{k N_j}{R}$ for $j \in \Dfix$, meaning that the existing mix cannot have active repetition constraints.
\end{enumerate}
    
\end{assumption}

First, we present a general lemma that allows us to bound any $\|p_1 - p_1'\|$.

\begin{lemma}[Argmin stability with mutual feasibility]\label{lem:argmin_stability_mutual}
Let $\mathcal{S}, \mathcal{S}' \subset \triangle^{\|D_1\| - 1}$ be convex sets. Let $h: \mathcal{S} \rightarrow \R$ be differentiable and $\mu_1$-strongly convex on $\mathcal{S}$. Let $h': \mathcal{S}' \rightarrow \R$ be a differentiable function. Define
\begin{align*}
    p_1 = \arg\min_{p_1 \in \mathcal{S}} h(p_1), \qquad p_1' = \arg\min_{p_1 \in \mathcal{S}'} h'(p_1)
\end{align*}

We assume mutual feasibility; that is, $p_1 \in \mathcal{S}'$ and $p_1' \in \mathcal{S}$. Then,
\begin{align*}
    \|p_1 - p_1'\| \le \frac{1}{\mu_1} \|\triangledown h(p_1') - \triangledown h'(p_1') \|.
\end{align*}
\end{lemma}

\begin{proof}
Applying our mutual feasibility assumption, the first order condition at $p_1$ for $h$ and $p_1'$ for $h$ yields
\begin{align*}
    &\langle \triangledown h(p_1), p_1' - p_1) \ge 0 \\
    &\langle \triangledown h'(p_1'), p_1 - p_1') \ge 0
\end{align*}

Adding these together and adding and subtracting $\triangledown h(p_1')$, we have
\begin{align*}
    &\langle \triangledown h(p_1) - \triangledown h'(p_1'), p_1' - p_1 \rangle \ge 0 \\
    \Rightarrow &\langle \triangledown h(p_1) - \triangledown h(p_1') + \triangledown h(p_1') - \triangledown h'(p_1'), p_1' - p_1 \rangle \ge 0 \\
    \Rightarrow &\langle \triangledown h(p_1) - \triangledown h(p_1'), p_1 - p_1' \rangle \le \langle \triangledown h(p_1') - \triangledown h'(p_1'), p_1' - p_1 \rangle
\end{align*}

Next, due to $\mu_1$ strong convexity of $h$, we have
\begin{align*}
    h(p_1) &\ge h(p_1') + \langle \triangledown h(p_1'), p_1 - p_1' \rangle + \frac{\mu_1}{2} \| p_1 - p_1'\|^2 \\
    h(p_1') &\ge h(p_1) + \langle \triangledown h(p_1), p_1' - p_1 \rangle + \frac{\mu_1}{2} \| p_1 - p_1'\|^2 \\
\end{align*}

Adding these together and applying the bound from the first order conditions, we get
\begin{align*}
    \mu \| p_1 - p_1'\|^2 &\le \langle \triangledown h(p_1) - \triangledown h(p_1'), p_1 - p_1' \rangle  \\
    &\le \langle \triangledown h(p_1') - \triangledown h'(p_1'), p_1' - p_1 \rangle \\
    &\le \|\triangledown h(p_1') - \triangledown h'(p_1') \| \cdot \| p_1' - p_1\|
\end{align*}

Dividing both sides by $\mu \|p_1' - p_1\|$ completes our proof.
\end{proof}

\begin{lemma}[When $\tilde{p}_{\Dfix} = q^\star_{\Dfix}$ for adding]\label{lem:add_exactness}
Consider the \Add update, where $\D_2=\emptyset$ and $\D'=[\D_1\D_2']$ with $\D_2'\neq\emptyset$.
Let $\tilde p\in\triangle^{m-1}$ be an optimal solution to the mixing problem on the pre-update domain set $\D=\D_1$ (equivalently, $\tilde p_{\D_1}=\tilde p$), and let $q^\star$ be the optimal solution to the post-update mixing problem on $\D'$.
If $\rho^\star=1$, then
\[
\tilde p_{\Dfix} = q^\star_{\Dfix}.
\]
\end{lemma}

\begin{proof}
Since $\D_2=\emptyset$, the pre-update problem over $\D=\D_1$ is equivalent to the post-update problem restricted to mixtures that assign zero mass to all added domains:
\begin{align*}
    &\text{minimize}_{q\in\triangle^{m'-1}} \; \frac{1}{n} \sum_{i = 1}^n f_i(\LM(S, R, q)) \\
    &\text{s.t.}\quad q_j \le \frac{kN_j'}{R}\quad \forall j\in\D_1\cup\D_2' \\
    &\qquad\;\; q_j = 0 \quad \forall j\in\D_2'.
\end{align*}
If $\rho^\star=1$, then $q_j^\star=0$ for all $j\in\D_2'$. This means that the additional constraints $q_j=0$ on the post-update problem do not change the optimal value, since the restricted feasible set still contains the optimal solution of $q_j^\star = 0$. Therefore, if $\rho^\star = 1$, the restricted problem above and the unrestricted post-update problem have the same optimum, and their optimizers coincide on $\D_1$; in particular, $q^\star_{\D_1}=\tilde p_{\D_1}$. Replacing $\D_1$ with $\Dfix$ completes our proof.
\end{proof}

\begin{lemma}\label{lemma:add_bound}
Assume mutual feasibility between $p_1$ and $p_1'$. Then, when the domain update involves adding domains,
\begin{align*}
     \| p_1 - p_1' \| \le \frac{2(1 - \rho^\star)}{\mu_1} \| \sup_{t \in [0, 1]}  \| \triangledown_{p_1, r}^2 F(r^t, p_1') \| \qquad r^t = t r_0+(1 - t)r^\star
\end{align*}

\end{lemma}

\begin{proof}
    Define $r_0 = [1, \mathbf{0}]$ where $\rho_0 = 1$. Then, $p_1 = \arg \min_{p_1 \in \mathcal{S}(r_0)} F(r_0, p_1)$. We can write $p_1'$ similarly as $p_1' = \arg \min_{p_1 \in \mathcal{S}(r^\star)} F(r^\star, p_1)$. 

    If we assume mutual feasibility and define $h(p_1) = F(r_0, p_1), h'(p_1) = F(r^\star, p_1)$, then applying Lemma~\ref{lem:argmin_stability_mutual} gives us
    \begin{align*}
        \| p_1 - p_1' \| \le \frac{1}{\mu_1} \|\triangledown_{p_1} F(r_0, p_1') - F(r^\star, p_1') \|.
    \end{align*}

    Define $r^t = t r_0 + (1 - t) r^\star$, where $t \in [0, 1]$. By the fundamental theorem of calculus, it holds that 
    \begin{align*}
        \triangledown_{p_1} F(r_0, p_1') - \triangledown_{p_1} F(r^\star, p_1') = \int_0^1 \triangledown_{p_1, r}^2 F(r^t, p_1') (r_0 - r^\star) dt.
    \end{align*}

    Therefore, 
    \begin{align*}
        \| p_1 - p_1' \| \le \frac{1}{\mu_1} \|r_0 - r^\star \| \sup_{t \in [0, 1]}  \| \triangledown_{p_1, r}^2 F(r^t, p_1') \|.
    \end{align*}

    Lastly, we bound $\|r_0 - r^\star\|$:
    \begin{align*}
        \|r_0 - r^\star\| = \|[1, \mathbf{0}] - [\rho^\star, (1 - \rho^\star) q_{\D_2'}^\star] \| = \| (1 - \rho^\star) [1, -q_{\D_2'}^\star]\| \le 2(1 - \rho^\star).
    \end{align*}

    This completes our proof.
\end{proof}

\begin{lemma}\label{lemma:grad_r_p1_add}
    Define $r^t = t r_0 + (1 - t) r^\star(p_1')$, where $r_0 = [1, \mathbf{0}]$.
    The gradient norm $\|\triangledown^2_{r,p_1}F(r^t,p_1')\|$ can be bounded by 
    \begin{align*}
       \|\triangledown^2_{r,p_1}F(r^t,p_1')\| &\le \bar{F}(q^\star) \exp( c_{\max}(1 - \rho^\star)) \|(1 + \alphafix + \alphacomp) \odot \alphafix \|
    \end{align*}
    
    where $c_{\max} = max_i \{\|A_{i, 1}\| + \|A_{i, 2}\| \}$ and $\odot$ is the Hadamard product. 
\end{lemma}

\begin{proof}

First, we can write $r^t$ as
\begin{align*}
    r^t &= t[1, \mathbf{0}] + (1 - t) [\rho^\star(p_1'), (1 - \rho^\star(p_1')) p_2^\star(p_1')] \\
    &= [t + (1 - t) \rho^\star(p_1'), (1 - t) (1 - \rho^\star(p_1')) p_2^\star(p_1')] \\
    &= [t + (1 - t) \rho^\star(p_1'), (1 - (t + (1 - t) \rho^\star(p_1'))) p_2^\star(p_1')].
\end{align*}

Therefore, we can express $r^t = [\rho^t, (1 - \rho^t) p_2^t]$ where $\rho^t := t + (1 - t) \rho^\star(p_1')$, and $p_2^t = p_2^\star(p_1')$. 

Applying Lemma~\ref{lemma:grad_r_p_1}, we can write the gradient as
\begin{align*}
    \triangledown^2_{r,p_1} F(r^t,p_1') &= \frac{1}{n}\sum_{i = 1}^n \exp(A_{i, 1}^\top \rho^t p_1' + A_{i, 2}^\top (1 - \rho^t) p_2^t) (\rho^t a_i + e_1) A_{i, 1}^\top,
\end{align*}

where $a_i = \SmallMatrix{\langle A_{i, 1}, p_1' \rangle \\ A_{i, 2}}  \in \R^{|\D_2'| + 1}$ and $e_1$ is the standard basis vector $\SmallMatrix{1 \\ 0 \\ \vdots \\ 0} \in \R^{|\D_2'| + 1}$. Then, using Cauchy-Schwarz inequality, the gradient norm can be bounded as
\begin{align}
&\|\triangledown^2_{r,p_1}F(r^t,p_1')\| \le \frac{1}{n}\sum_{i = 1}^n \exp(A_{i, 1}^\top \rho^t p_1' + A_{i, 2}^\top (1 - \rho^t) p_2^t) \|(\rho^t a_i + e_1) A_{i, 1}^\top\| \label{eq:add_grad_bound} \\
&\le \frac{1}{n}\sqrt{\sum_{i = 1}^n \exp(2(A_{i, 1}^\top \rho^t p_1' + A_{i, 2}^\top (1 - \rho^t) p_2^t))} \sqrt{\sum_{i = 1}^n \|(\rho^t a_i + e_1) A_{i, 1}^\top\|^2}. \nonumber 
\end{align}

We bound the first summation of~\eqref{eq:add_grad_bound} by adding and subtracting $\rho^\star(p_1')$:
\begin{align*}
\sum_{i = 1}^n &\exp(2(A_{i, 1}^\top \rho^t p_1' + A_{i, 2}^\top (1 - \rho^t) p_2^t)) \\
&= \sum_{i = 1}^n \exp(2(A_{i, 1}^\top (\rho^t + \rho^\star(p_1') - \rho^\star(p_1')) p_1' + A_{i, 2}^\top ((1 - \rho^t) + (1 - \rho^\star(p_1')) - (1 - \rho^\star(p_1'))) p_2^t)) \\
&= \sum_{i = 1}^n \exp(2(A_{i, 1}^\top \rho^\star(p_1') p_1' + A_{i, 2}^\top (1 - \rho^\star(p_1')) p_2^t)) \exp(2(A_{i, 1}^\top (\rho^t - \rho^\star(p_1')) p_1' + A_{i,2}^\top (\rho^\star(p_1') - \rho^t) p_2^\star(p_1') )) \\
&= \sum_{i = 1}^n \exp(2(A_{i, 1}^\top \rho^\star(p_1') p_1' + A_{i, 2}^\top (1 - \rho^\star(p_1')) p_2^t)) \exp(2t (1 - \rho^\star(p_1'))(A_{i, 1}^\top p_1' - A_{i, 2}^\top p_2^\star(p_1'))).
\end{align*}

We can then replace some terms in~\eqref{eq:add_grad_bound} with $\bar{F}(r^\star(p_1'), p_1')$ 
\begin{align*}
    \frac{1}{n}&\sqrt{\sum_{i = 1}^n \exp(2(A_{i, 1}^\top \rho^t p_1' + A_{i, 2}^\top (1 - \rho^t) p_2^t))} \\
    &\le \frac{1}{n} \sum_{i = 1}^n \exp(A_{i, 1}^\top \rho^\star(p_1') p_1' + A_{i, 2}^\top (1 - \rho^\star(p_1')) p_2^t) \cdot \max_i \{\exp(t (1 - \rho^\star(p_1'))(A_{i, 1}^\top p_1' - A_{i, 2}^\top p_2^\star(p_1'))) \}\\
    &= \bar{F}(r^\star(p_1'), p_1') \max_i \{\exp(t (1 - \rho^\star(p_1'))(A_{i, 1}^\top p_1' - A_{i, 2}^\top p_2^\star(p_1'))) \}.
\end{align*}

We can bound $A_{i, 1}^\top p_1' - A_{i, 2}^\top p_2^\star(p_1') \le \|A_{i, 1}\| + \|A_{i, 2}\|$ using Cauchy-Schwarz. Let $c_{\max} = \max_i \{ \|A_{i, 1}\| + \|A_{i, 2}\|\}$.
Using the fact that $t\in[0, 1]$, the first part of~\eqref{eq:add_grad_bound} can be bounded by
\begin{align*}
    \frac{1}{n}&\sqrt{\sum_{i = 1}^n \exp(2(A_{i, 1}^\top \rho^t p_1' + A_{i, 2}^\top (1 - \rho^t) p_2^t))} \le \bar{F}(r^\star(p_1'), p_1') \exp( c_{\max}(1 - \rho^\star(p_1'))).
\end{align*}

The second summation of~\eqref{eq:add_grad_bound} can be bounded in the exact same way as the second summation of~\eqref{eq:m_bound} in Lemma~\ref{lemma:grad_r_p_1}. Therefore, we have that
\begin{align*}
    \sqrt{\sum_{i = 1}^n \|(\rho^t a_i + e_1) A_{i, 1}^\top\|^2} \le  \|(1 + \alphafix + \alphacomp) \odot \alphafix \|.
\end{align*}

This completes our proof.
\end{proof}

\subsubsection{Theorem~\ref{thm:add_domains} (Formal)}

\begin{theorem}
    Define $\tilde{p}$ is the solution to~\eqref{eq:mixing_problem} on $\D$ and suppose that new domains are added. Let $c_{\max} = \max_i \{\|A_{i, 1}\| + \|A_{i, 2}\|\}$. Then, the difference between $\tilde{p}_{\Dfix}$ and $q^\star_{\Dfix}$ is bounded by
    \begin{align*}
        \|\tilde{p}_{\Dfix} - q^\star_{\Dfix}\| \le \frac{2\bar{F}(q^\star) \exp( c_{\max}(1 - \rho^\star))}{\mu_1}   \cdot \kappa(\alphafix, \alphacomp) \cdot (1 - \rho^\star).
    \end{align*}
\end{theorem}

\begin{proof}
    Applying Lemma~\ref{lemma:grad_r_p1_add} and Lemma~\ref{lemma:add_bound}, we have
    \begin{align*}
        \| p_1 - p_1' \| &\le \frac{2(1 - \rho^\star)}{\mu_1} \bar{F}(q^\star) \exp( c_{\max}(1 - \rho^\star)) \|(1 + \alphafix + \alphacomp) \odot \alphafix \| \\
        &= \frac{2\bar{F}(q^\star) \exp( c_{\max}(1 - \rho^\star))}{\mu_1}   \cdot \kappa(\alphafix, \alphacomp) \cdot (1 - \rho^\star).
    \end{align*}
\end{proof}

\subsection{Additional Theoretical Results}

\subsubsection{Removing domains}

Our results for are symmetric to the addition update, but we state the full derivation for completeness.

\begin{lemma}[When $\tilde{p}_{\Dfix} = q^\star_{\Dfix}$ for removing]\label{lem:rem_exactness}
Consider the \emph{removing} update, where $\D_2\neq\emptyset$ and $\D'=\D_1$ (equivalently, $\D_2'=\emptyset$).
Let $\tilde p\in\triangle^{m-1}$ be an optimal solution to the mixing problem on the pre-update domain set $\D=[\D_1,\D_2]$, and let $q^\star$ be the optimal solution to the post-update mixing problem on $\D'=\D_1$ (equivalently, $q^\star_{\D_1}=q^\star$).
If $\pi^\star=1$, then
\[
\tilde p_{\Dfix} = q^\star_{\Dfix}.
\]
\end{lemma}

\begin{proof}
Since $\D_2'=\emptyset$, the post-update problem over $\D'=\D_1$ is equivalent to the pre-update problem restricted to mixtures that assign zero mass to all removed domains:
\begin{align*}
    &\text{minimize}_{p\in\triangle^{m-1}} \; \frac{1}{n} \sum_{i = 1}^n f_i(\LM(S, R, p)) \\
    &\text{s.t.}\quad p_j \le \frac{kN_j}{R}\quad \forall j\in\D_1\cup\D_2 \\
    &\qquad\;\; p_j = 0 \quad \forall j\in\D_2'.
\end{align*}
If $\pi^\star=1$, then $\tilde p_j=0$ for all $j\in\D_2$. This means that the additional constraints $p_j=0$ on the pre-update problem do not change the optimal value, since the restricted feasible set still contains the optimal solution of $\tilde p_j = 0$. Therefore, if $\pi^\star = 1$, the restricted problem above and the unrestricted pre-update problem have the same optimum, and their optimizers coincide on $\D_1$; in particular, $q^\star_{\D_1}=\tilde p_{\D_1}$. Replacing $\D_1$ with $\Dfix$ completes our proof.
\end{proof}

\subsubsection{Partitioning a domain}

\begin{lemma}[When $\tilde{p}_{\Dfix} = q^\star_{\Dfix}$ for partitioning]\label{lem:part_exactness}
Consider the \emph{partitioning} update operator, where $\D \setminus \D_1 =\{D\}$ consists of a single domain and
$\D_2'=\{D_1',\dots,D_\ell'\}$ are subdomains with $D=\bigcup_{j=1}^\ell D_j'$.

Define the natural distribution $p_{\mathrm{nat}}\in\triangle^{\ell-1}$ by
\[
[p_{\mathrm{nat}}]_j = \frac{N_j'}{\sum_{t=1}^\ell N_t'} \qquad \forall j\in[\ell].
\]
If $q^\star_{\D_2'} = p_{\mathrm{nat}}$, then
\[
\tilde p_{\Dfix} = q^\star_{\Dfix}.
\]
\end{lemma}

\begin{proof}
Since $\D_2=\{D\}$ is a single domain, sampling $(1-\pi)R$ tokens from $D$ induces a distribution over the subdomains
$\{D_1',\dots,D_\ell'\}$ that is proportional to token counts, i.e., $p_{\mathrm{nat}}$.
Therefore, the pre-update optimization problem can be written as the post-update optimization problem restricted to mixtures that preserve this
conditional distribution within the partition:
\begin{align*}
    &\text{minimize}_{q\in\triangle^{m'-1}} \; \frac{1}{n} \sum_{i = 1}^n f_i(\LM(S, R, q)) \\
    &\text{s.t.}\quad q_j \le \frac{kN_j'}{R}\quad \forall j\in\D_1 \cup\D_2' \\
    &\qquad\;\; q_{\D_2'} = p_{\mathrm{nat}},
\end{align*}

If $q^\star_{\D_2'}=p_{\mathrm{nat}}$, then $q^\star$ lies in the restricted feasible set above. The additional constraint
$q_{\D_2'}=p_{\mathrm{nat}}$ does not change the optimal value of the post-update problem, since the restricted feasible set still contains
the optimal solution. Therefore, the optimizer on the unrestricted post-update problem must coincide with the optimizer of the restricted problem on the
unaffected domains, implying $q^\star_{\D_1}=\tilde p_{\D_1}$. Replacing $\D_1$ with $\Dfix$ completes our proof.
\end{proof}

\subsubsection{Revising a domain}

\begin{lemma}[Exactness for revising]\label{lem:rev_exactness}
Consider the \emph{revising} update operator, where $\D_2=\{D\}$ and $\D_2'=\{D'\}$ consists of a single domain whose contents are modified.
Assume the log-linear model holds both pre- and post-update:
\[
F^{\mathrm{old}}(q) = \frac{1}{n}\sum_{i=1}^n \exp\!\left((A_{i,1})^\top q_{\Dfix} + (A_{i,2}^{\mathrm{old}})^\top q_{\D_2}\right),
\qquad
F^{\mathrm{new}}(q) = \frac{1}{n}\sum_{i=1}^n \exp\!\left((A_{i,1})^\top q_{\Dfix} + (A_{i,2}^{\mathrm{new}})^\top q_{\D_2}\right),
\]
where $\D=[\D_1,\D_2]$ and $\D'=[\D_1,\D_2']$ have the same dimensionality, and the repetition constraints are unchanged.

Let $\tilde p$ be an optimal solution to the pre-update mixing problem with objective $F^{\mathrm{old}}$, and let $q^\star$ be an optimal solution to the post-update mixing problem with objective $F^{\mathrm{new}}$.
If $A_{i,2}^{\mathrm{new}} = A_{i,2}^{\mathrm{old}}$ for all $i\in[n]$, then $\tilde p = q^\star$, and in particular $\tilde p_{\Dfix}=q^\star_{\Dfix}$.
\end{lemma}

\begin{proof}
If $A_{i,2}^{\mathrm{new}} = A_{i,2}^{\mathrm{old}}$ for all $i\in[n]$, then $F^{\mathrm{new}}(q)=F^{\mathrm{old}}(q)$ for all $q$.
Since the feasible set (simplex and repetition constraints) is unchanged under revising, the pre-update and post-update optimization problems are identical.
Therefore their optimizers coincide, so $\tilde p=q^\star$.
\end{proof}

\section{Experiment Details}\label{supp:details}

\subsection{Experimental Setup}\label{supp:exp_details}

\textbf{Domains.} We provide a breakdown of the number of tokens per domain in Table~\ref{tab:domain_sizes}.

\begin{table}[h]
\centering
\caption{Token counts for all domains in the final domain set. Sources with topic-level partitions (DCLM, Stack-Edu, olmOCR Science PDFs) are shown with indented topics. Single-source domains are shown without subdivision.}
\small
\begin{tabular}{l r | l r}
\toprule
\textbf{Domain} & \textbf{Token Count} & \textbf{Domain} & \textbf{Token Count} \\
\midrule
\multicolumn{2}{l|}{\textbf{DCLM}~\citep{dclm}} & \multicolumn{2}{l}{\textbf{Stack-Edu}~\citep{allal2025smollm2smolgoesbig}} \\
\quad Adult Content & 67,760,078,203 & \quad C & 4,735,074,247 \\
\quad Art and Design & 70,659,711,995 & \quad C\# & 7,204,525,343 \\
\quad Crime and Law & 170,130,914,779 & \quad C++ & 12,530,445,761 \\
\quad Education and Jobs & 184,690,792,861 & \quad Go & 1,398,595,118 \\
\quad Electronics and Hardware & 80,168,541,745 & \quad Java & 31,347,725,888 \\
\quad Entertainment & 441,768,061,760 & \quad JavaScript & 8,886,972,357 \\
\quad Fashion and Beauty & 37,256,539,512 & \quad Markdown & 28,916,320,218 \\
\quad Finance and Business & 310,313,927,581 & \quad PHP & 7,395,033,318 \\
\quad Food and Dining & 105,937,299,687 & \quad Python & 18,017,832,560 \\
\quad Games & 229,992,491,702 & \quad Ruby & 1,386,775,805 \\
\quad Health & 393,496,227,836 & \quad Rust & 1,418,447,132 \\
\quad History and Geography & 161,049,719,459 & \quad SQL & 7,063,472,860 \\
\quad Home and Hobbies & 126,910,777,314 & \quad Shell & 2,542,637,875 \\
\quad Industrial & 43,572,140,450 & \quad Swift & 1,510,019,025 \\
\quad Literature & 364,834,344,848 & \quad TypeScript & 2,495,753,789 \\
\quad Politics & 611,198,130,192 & \multicolumn{2}{l}{} \\
\quad Religion & 277,776,929,208 & \multicolumn{2}{l}{\textbf{olmOCR Science PDFs}~\citep{olmo2025olmo3}} \\
\quad Science, Math and Technology & 427,131,054,341 & \quad Adult & 303,073,226 \\
\quad Social Life & 218,731,841,124 & \quad Art and Design & 6,833,185,034 \\
\quad Software & 108,039,380,021 & \quad Crime and Law & 42,538,674,743 \\
\quad Software Development & 223,384,974,282 & \quad Education and Jobs & 138,127,926,093 \\
\quad Sports and Fitness & 196,759,999,355 & \quad Entertainment & 6,069,602,783 \\
\quad Transportation & 90,793,306,202 & \quad Fashion and Beauty & 557,917,820 \\
\quad Travel and Tourism & 57,642,815,530 & \quad Finance and Business & 61,150,044,703 \\
 & & \quad Food and Dining & 2,322,982,204 \\
\multicolumn{2}{l|}{\textbf{Single-Source Domains}} & \quad Games & 2,486,095,532 \\
\textbf{AlgebraicStack}~\citep{azerbayev2023llemma} & 11,818,955,329 & \quad Health & 108,215,933,374 \\
\textbf{ArXiv}~\citep{azerbayev2023llemma} & 20,773,846,846 & \quad Home and Hobbies & 3,924,579,643 \\
\textbf{FineMath-3+}~\citep{allal2025smollm2smolgoesbig} & 34,057,973,953 & \quad Industrial & 29,389,278,657 \\
\textbf{Pes2o}~\citep{peS2o} & 58,552,461,187 & \quad Literature & 31,886,391,090 \\
\textbf{Wikipedia}~\citep{soldaini2024dolma} & 10,067,758,073 & \quad Politics & 39,234,116,889 \\
 & & \quad Religion & 24,729,732,953 \\
 & & \quad Science and Technology & 424,245,385,160 \\
 & & \quad Software & 9,146,853,216 \\
 & & \quad Software Development & 41,841,278,724 \\
 & & \quad Sports and Fitness & 5,450,913,796 \\
 & & \quad Transportation & 17,149,342,957 \\
 & & \quad Travel & 2,102,425,717 \\
\bottomrule
\end{tabular}
\label{tab:domain_sizes}
\end{table}

\textbf{Model.} We train 1B parameter decoder-only transformer models using Olmo 2 architecture~\citep{olmo20242olmo2furious}. In particular, we set n\_layers=16, n\_head=16, d\_model=2048, head\_dim=128.

We use a batch size of 512, a learning rate of $0.0018$ with a cosine scheduler with warmup and linear decay, and a sequence length of 4096. We use the Dolma 2 tokenizer.

All 1B models were trained on 32 NVIDIA H100s (80 GB), while proxy models were trained on one NVIDIA H100.

\textbf{Evaluation.} Table~\ref{tab:eval_suite} lists all the downstream tasks. 

\begin{table}[t]
  \centering
  \caption{
  \textbf{Details of the BPB evaluation suite}. 
  We use the \textsc{OlmoBaseEval} easy evaluation suite \citep{olmo2025olmo3}, formatted as bits-per-byte (BPB) over gold continuations. $^\dagger$ = few-shot examples are built-in the task; $^\alpha$ = human-written few-shot examples. We treat subtasks as standalone tasks, except for MMLU, where we use averages over the four MMLU categories. 
  }
\begin{scriptsize}
\begin{tabular}{lllccc}
\toprule
& \bf{Task} & \bf{Capability} & \bf{\# ICL} & \bf{Metric} & \bf{\# Subtasks} \\
\midrule

\rowcolor{ai2offwhite} & Minerva MATH (\citeyear{lewkowycz2022solving}) & Math Gen & 4$^\alpha$ & BPB & 7 \\
\rowcolor{ai2offwhite} & \multicolumn{5}{l}{\hspace{6pt}Subtasks: \textit{Algebra, Counting and Probability, Geometry, Intermediate Algebra} } \\
\rowcolor{ai2offwhite}\rowcolor{ai2offwhite} \multirow{-3}{*}{\rotatebox[origin=c]{90}{\textit{Math}}} & \multicolumn{5}{l}{\hspace{6pt}\textit{Number Theory, Prealgebra, Precalculus} } \\
\rowcolor{lightgrey} & HumanEval (\citeyear{chen2021codex}) & Code Gen & 3 & BPB & - \\
\rowcolor{lightgrey} & MBPP (\citeyear{austin2021program}) & Code Gen & 3 & BPB & - \\
\rowcolor{lightgrey} & MT MBPP (\citeyear{cassano2022multipl, olmo2025olmo3}) & Code Gen & 3 & BPB & 17 \\
\rowcolor{lightgrey} & \multicolumn{5}{l}{\hspace{6pt}Subtasks: \textit{Bash, C, C++, C\#, Go, Haskell, Java, JavaScript, MatLab, PHP} } \\
\rowcolor{lightgrey}\multirow{-5}{*}{\rotatebox[origin=c]{90}{\textit{Code}}} & \multicolumn{5}{l}{\hspace{6pt}\textit{Python, R, Ruby, Rust, Scala, Swift, TypeScript} } \\
\rowcolor{ai2offwhite} & ARC (\citeyear{clark2018think}) & Science QA & 5 & BPB & 2 \\
\rowcolor{ai2offwhite} & \multicolumn{5}{l}{\hspace{6pt}Subtasks: \textit{ARC-Easy, ARC-Challenge} } \\
\rowcolor{ai2offwhite} & MMLU STEM (\citeyear{hendryckstest2021}) & General QA & 5 & BPB & 19 \\
\rowcolor{ai2offwhite} & MMLU Humanities (\citeyear{hendryckstest2021}) & General QA & 5 & BPB & 13 \\
\rowcolor{ai2offwhite} & MMLU Social Sci. (\citeyear{hendryckstest2021}) & General QA & 5 & BPB & 12 \\
\rowcolor{ai2offwhite} & MMLU Other (\citeyear{hendryckstest2021}) & General QA & 5 & BPB & 14 \\
\rowcolor{ai2offwhite} & CSQA (\citeyear{talmor-etal-2019-commonsenseqa}) & Commonsense QA & 5 & BPB & - \\
\rowcolor{ai2offwhite} & HellaSwag (\citeyear{zellers-etal-2019-hellaswag}) & Language Modeling & 5 & BPB & - \\
\rowcolor{ai2offwhite} & WinoGrande (\citeyear{Sakaguchi_Le_Bras_Bhagavatula_Choi_2020}) & Language Modeling & 5 & BPB & - \\
\rowcolor{ai2offwhite} & SocialIQA (\citeyear{sap-etal-2019-social}) & Social QA & 5 & BPB & - \\
\rowcolor{ai2offwhite} & PiQA (\citeyear{Bisk_Zellers_Le_bras_Gao_Choi_2020}) & Physical QA & 5 & BPB & - \\
\rowcolor{ai2offwhite} & CoQA (\citeyear{reddy-etal-2019-coqa}) & Conversation QA & 0$^\dagger$ & BPB & - \\
\rowcolor{ai2offwhite} & DROP (\citeyear{dua-etal-2019-drop}) & Passage QA & 5 & BPB & - \\
\rowcolor{ai2offwhite} & Jeopardy (\citeyear{mosaic-jeopardy}) & Trivia QA & 5 & BPB & - \\
\rowcolor{ai2offwhite} & NaturalQs (\citeyear{kwiatkowski-etal-2019-natural}) & General QA & 5 & BPB & - \\
\rowcolor{ai2offwhite} & SQuAD (\citeyear{rajpurkar-etal-2016-squad}) & General QA & 5 & BPB & - \\
\rowcolor{ai2offwhite} & SciQ (\citeyear{welbl-etal-2017-crowdsourcing}) & Science QA & 5 & BPB & - \\
\rowcolor{ai2offwhite} & Basic Skills (\citeyear{olmo2025olmo3}) & Basic QA & 5 & BPB & 6 \\
\rowcolor{ai2offwhite} & \multicolumn{5}{l}{\hspace{6pt}Subtasks: \textit{Basic Arithmetic, String Manipulation, Simple Coding} } \\
\rowcolor{ai2offwhite} & \multicolumn{5}{l}{\hspace{6pt}\textit{Elementary Logical Reasoning, Basic Common Sense, Simple Pattern Recognition} } \\
\rowcolor{ai2offwhite} & Lambada (\citeyear{paperno2016lambada}) & Language Modeling & 0 & BPB & - \\
\rowcolor{ai2offwhite} \multirow{-23}{*}{\rotatebox[origin=c]{90}{\textit{QA}}} & MedMCQA (\citeyear{pmlr-v174-pal22a}) & Medical QA & 5 & BPB & - \\
\bottomrule
\end{tabular}
\end{scriptsize}
  \vspace{6pt}
  \label{tab:eval_suite}
\end{table}

\subsection{Implementation Details}\label{supp:imp_details}

For \fullmixreuse and \partialmixreuse, the first $\tilde{p}$ is from applying \base on the initial web corpus. After each update, the newly recomputed mix becomes the $\tilde{p}$ of the next one.

Across all swarm-based methods we evaluate on (\base, Swarm Reuse, \fullmixreuse, \partialmixreuse), we use $\Ssmall=30$M parameters (architecture in Table~\ref{tab:model_architectures}) and train for 5x Chinchilla, with a batch size of 64, a learning rate of $0.007$, and a sequence length of 2048. 

For any $m$ domains we want to recompute, we set $K$ to approximate $c(m + 1)$ for linear multiplier $c = 1, 2, 3$. In particular, for $c = 3$, we set $K = 2^x$ to be the power of $2$ that is closest to $3(m+1)$. For $c = 2$, we set $K = 2^{x-1}$. For $c = 1$, we set $K = m+1$. When we recompute over $m=2$ domains, we use search instead of regression, as the search space is one-dimensional. See Tables~\ref{tab:cost_fullrecomputation}, \ref{tab:cost_swarmreuse}, \ref{tab:cost_fullreuse}, and~\ref{tab:cost_partialreuse} for exact swarm sizes.
We solve all of our mixture optimization problems using CVXPY and a KL regularization term with $\lambda=0.05$. We set $k=4$ and $R = 1$T for the repetition constraint of the experiments in \S\ref{sec:superswarm}.

\begin{table}[!h]
\centering
\caption{Swarm details per stage for \textbf{Full Recomputation}. We list the number of domains recomputed at each stage, and the columns with $c$ contain the number of proxy runs $K \approx c(m+1)$.}
\footnotesize
\setlength{\tabcolsep}{6pt}
\begin{tabular}{lcccc}
\toprule
\textbf{Stage} & \textbf{\# domains recomputed} & \textbf{$c=1$} & \textbf{$c=2$} & \textbf{$c=3$} \\
\midrule
Initial      & 24 & 25 & 32 & 64 \\
\Add (Stack-Edu languages)  & 39 & 40 & 64 & 128 \\
\Add (ArXiv, FineMath 3+, PDFs, Wiki, AlgebraicStack, pes2o) & 45 & 46 & 64 & 128 \\
\Revise (PDFs) & 45 & 46 & 64 & 128 \\
\Remove (AlgebraicStack) & 44 & 45 & 64 & 128 \\
\Partition (PDFs) & 64 & 65 & 128 & 256 \\
\midrule
\textbf{Total \# proxy runs} & & 267 & 416 & 832 \\
\bottomrule
\end{tabular}
\label{tab:cost_fullrecomputation}
\end{table}

\begin{table}[!h]
\centering
\caption{Swarm details per stage for \textbf{Swarm Reuse}. We list the number of domains recomputed at each stage, and the columns with $c$ contain the number of additional proxy runs needed, $K \approx c(m+1)$. $c=3$ corresponds to Figure~\ref{fig:main_results}.}
\footnotesize
\setlength{\tabcolsep}{6pt}
\begin{tabular}{lcccc}
\toprule
\textbf{Stage} & \textbf{\# domains recomputed} & \textbf{$c=1$} & \textbf{$c=2$} & \textbf{$c=3$} \\
\midrule
Initial      & 24 & 25 & 32 & 64 \\
\Add (Stack-Edu languages)  & 39 & 15 & 32 & 64 \\
\Add (ArXiv, FineMath 3+, PDFs, Wiki, AlgebraicStack, pes2o) & 45 & 6 & 6 & 6 \\
\Revise (PDFs) & 45 & 6 & 6 & 6 \\
\Remove (AlgebraicStack) & 44 & 5 & 5 & 5 \\
\Partition (PDFs) & 64 & 20 & 59 & 123 \\
\midrule
\textbf{Total \# proxy runs} & & 77 & 140 & 268 \\
\bottomrule
\end{tabular}
\label{tab:cost_swarmreuse}
\end{table}

\begin{table}[!h]
\centering
\caption{Swarm details per stage for \textbf{Full Mixture Reuse}. We list the number of domains recomputed at each stage, and the columns with $c$ contain the number of proxy runs $K \approx c(m+1)$. $c=3$ corresponds to Figure~\ref{fig:main_results}.}
\footnotesize
\setlength{\tabcolsep}{6pt}
\begin{tabular}{lcccc}
\toprule
\textbf{Stage} & \textbf{\# domains recomputed} & \textbf{$c=1$} & \textbf{$c=2$} & \textbf{$c=3$} \\
\midrule
Initial      & 24 & 25 & 32 & 64 \\
\Add (Stack-Edu languages)  & 16 & 17 & 32 & 64 \\
\Add (ArXiv, FineMath 3+, PDFs, Wiki, AlgebraicStack, pes2o) & 7 & 8 & 8 & 16 \\
\Revise (PDFs) & 2 & 3 & 4 & 8 \\
\Remove (AlgebraicStack) & 0 & 0 & 0 & 0 \\
\Partition (PDFs) & 22 & 23 & 32 & 64 \\
\midrule
\textbf{Total \# proxy runs} &  & 76 & 108 & 216 \\
\bottomrule
\end{tabular}
\label{tab:cost_fullreuse}
\end{table}

\begin{table}[!h]
\centering
\caption{Swarm details per stage for \textbf{Partial Mixture Reuse}. We list the number of domains recomputed at each stage, and the columns with $c$ contain the number of proxy runs $K \approx c(m+1)$. $c=3$ corresponds to Figure~\ref{fig:main_results}.}
\footnotesize
\setlength{\tabcolsep}{6pt}
\begin{tabular}{lcc}
\toprule
\textbf{Stage} & \textbf{\# domains ($m$)} & \textbf{Proxy Runs ($c=3$)} \\
\midrule
Initial      & 24 & 64 \\
\Add (Stack-Edu languages)  & 17 & 64 \\
\Add (ArXiv, FineMath 3+, PDFs, Wiki, AlgebraicStack, pes2o) & 8 & 32 \\
\Revise (PDFs) & 8 & 32 \\
\Remove (AlgebraicStack) & 7 & 16 \\
\Partition (PDFs) & 27 & 64 \\
\midrule
\textbf{Total \# proxy runs} &  & 272 \\
\bottomrule
\end{tabular}
\label{tab:cost_partialreuse}
\end{table}

Algorithm~\ref{alg:swarm_reuse_method} corresponds to the Swarm Reuse approach in \S\ref{sec:superswarm}, presenting another way to reuse prior information from mixing. 
The blue coloring indicates aspects of this algorithm that are different from \base.

\begin{algorithm}[htb]
   \caption{Swarm Reuse Mixing}
   \label{alg:swarm_reuse_method}
\begin{algorithmic}[1]
   \STATE {\bfseries Input:} \textcolor{olmoBlue}{Old domain set $\D$, new domain set $\D'$, old swarm $\{(p^j_{\text{old}}, \{y_{ij}^{\text{old}}\}_{i=1}^n)\}_{j=1}^{K_{\text{old}}}$}, new swarm size $K_{\text{new}} = \mathcal{O}(|\D'|)$, Dirichlet distribution $\mathcal{P}'$, repetition factor $k$, requested tokens $R$.
   \STATE Sample new mixes $p^1_{\text{new}}, \dots, p^{K_{\text{new}}}_{\text{new}} \sim \mathcal{P}'$ on $\D'$.
   \STATE Train proxy models on new mixes and evaluate on downstream tasks to get new swarm data $\{(p^j_{\text{new}}, \{y_{ij}^{\text{new}}\}_{i = 1}^n \}_{j =1}^{K_{\text{new}}}$.
   \STATE \textcolor{olmoBlue}{Map old swarm mixes to new domain set: for $j \in [K_{\text{old}}]$, compute $\tilde{p}^j_{\text{old}} \in \triangle^{m'-1}$ from $p^j_{\text{old}} \in \triangle^{m-1}$ where:}
   \begin{itemize}
       \item \textcolor{olmoBlue}{If \textbf{Add}: $\D = [\D_1, \emptyset]$ and $\D' = [\D_1, \D_2']$. Then $\tilde{p}^j_{\text{old}}(D_i') = p^j_{\text{old}}(D_i)$ for $D_i \in \D_1$, and $\tilde{p}^j_{\text{old}}(D_i') = 0$ for $D_i' \in \D_2'$.}
       \item \textcolor{olmoBlue}{If \textbf{Partition}: $\D = [\D_1, \{D_{\text{par}}\}]$ and $\D' = [\D_1, \{D_{\text{par}}^1, \dots, D_{\text{par}}^{\ell}\}]$ where $D_{\text{par}} = \bigcup_{k=1}^{\ell} D_{\text{par}}^k$. Then $\tilde{p}^j_{\text{old}}(D_i') = p^j_{\text{old}}(D_i)$ for $D_i \in \D_1$, and $\tilde{p}^j_{\text{old}}(D_{\text{par}}^k) = p^j_{\text{old}}(D_{\text{par}}) \cdot \frac{N_{D_{\text{par}}^k}}{\sum_{k'=1}^{\ell} N_{D_{\text{par}}^{k'}}}$ for the partitioned domains.}
   \end{itemize}
   \STATE \textcolor{olmoBlue}{Construct combined swarm: $\mathcal{S} = \{(\tilde{p}^j_{\text{old}}, \{y_{ij}^{\text{old}}\})\}_{j=1}^{K_{\text{old}}} \cup \{(p^j_{\text{new}}, \{y_{ij}^{\text{new}}\})\}_{j=1}^{K_{\text{new}}}$}.
   \FOR{$i \in [n]$}
       \STATE Use combined swarm $\mathcal{S}$ to fit the log-linear model $\hat{f}_i(p) \texttt{=} c_i + \exp(A_i^\top p)$, where $c_i \in \R^+$ and $A_i \in \R^{|\D'|}$.
   \ENDFOR
   \STATE Solve the optimization problem:
   \begin{align}
       &\text{minimize}_{p \in \triangle^{|\D'|-1}} \frac{1}{n} \sum_{i=1}^n \hat{f}_i(p) \\ 
       &\text{subject to} \quad p_j \leq \frac{k N_j'}{R} \; \forall j \in |\D'| \nonumber
   \end{align}
   \STATE \textbf{Return} $p^\star$, the solution on new domain set $\D'$.
\end{algorithmic}
\end{algorithm}

\subsection{Additional Results}\label{sec:additional_exp_results}

\textbf{6T results.} In Figure~\ref{fig:superswarm_6T}, we compare the performance versus cost of mixing strategies when $R=6$T, a more data-constrained setting. We exclude \partialmixreuse since \fullmixreuse performs well at $R=6$T and only falls short at $R=1$T (Figure~\ref{fig:theorem2_validation}).

We find that \fullmixreuse achieves 99\% of full recomputation's ($c=3$) performance improvement (+6.94\% versus +6.97\%) while using 74\% fewer total proxy runs (216 versus 832 runs). Swarm reuse achieves $93\%$ of full recomputation's performance with 52 more runs than \fullmixreuse. 

\begin{figure}
    \centering
    \includegraphics[width=0.5\linewidth]{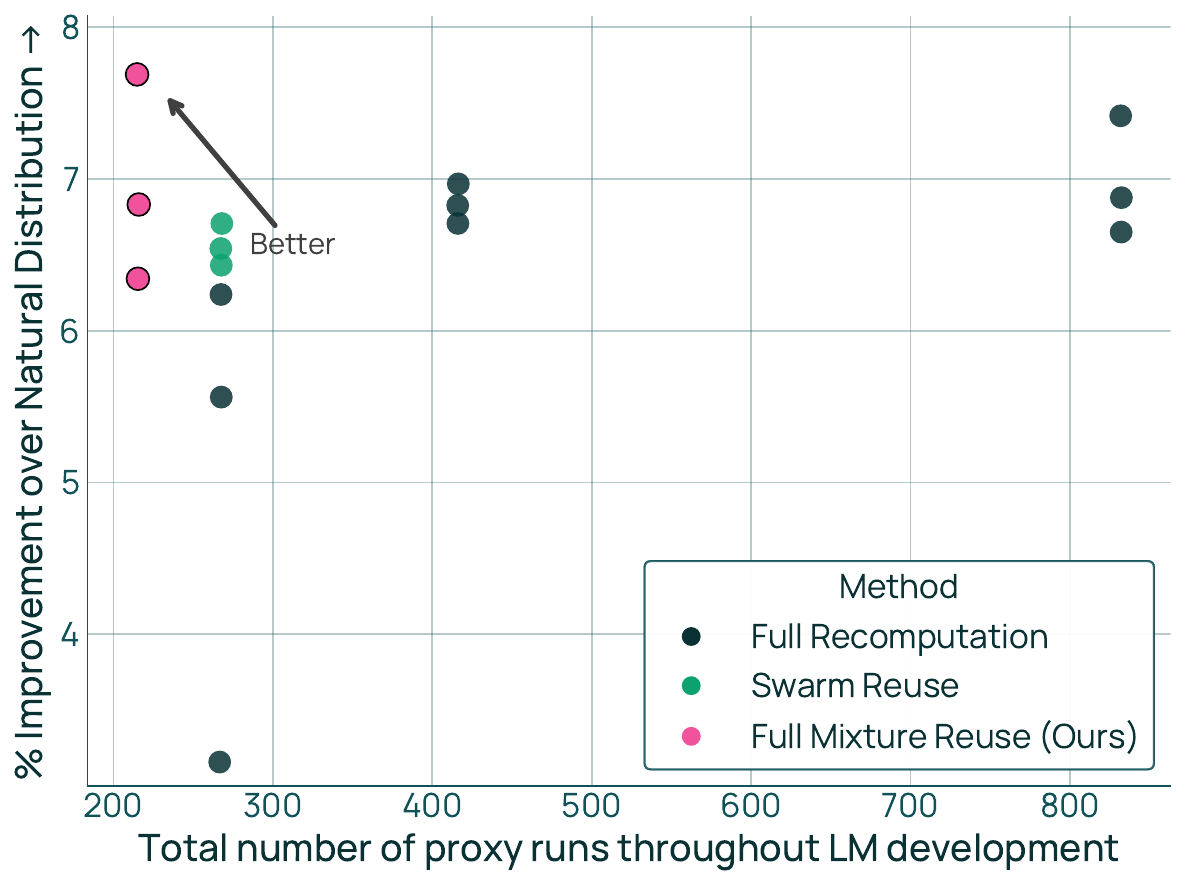}
    \caption{Performance improvement versus cost of mixing under evolving domains. \fullmixreuse achieves $99\%$ of the improvement of full recomputation while using $74\%$ fewer proxy runs.}
    \label{fig:superswarm_6T}
\end{figure}

\textbf{Low proxy run budget regime.} While Figure~\ref{fig:main_results} assesses performance versus number of proxy runs for cumulative costs ranging from 216 to 832 runs, it is unclear how recomputation strategies behave in a lower budget regime. We evaluate full recomputation with $c = 1$ (i.e., the minimum swarm size possible) and all other methods with $K \approx c(m+1)$ for $c = 1, 2, 3$, such that the range of total proxy runs considered is from 76 to 272.

Figure~\ref{fig:low_budget} shows performance versus number of proxy runs in a low budget regime. We make three observations:
\begin{itemize}
    \item The performance of swarm reuse versus \fullmixreuse at low budgets is similar.
    \item While the minimum number of proxy runs needed for full recomputation is 267, mixture reuse and swarm reuse unlock a lower budget regime, requiring as little as 76 runs. These reuse strategies can be attractive to practitioners with strict compute budgets and large domain sets that undergo many updates.
    \item With only 76 runs, mixture reuse still achieves a $+9.6\%$ and $+5\%$ improvement over the natural distribution when $R=1$T and $6$T, respectively.
\end{itemize}

\begin{figure}
    \centering
    \includegraphics[width=0.49\linewidth]{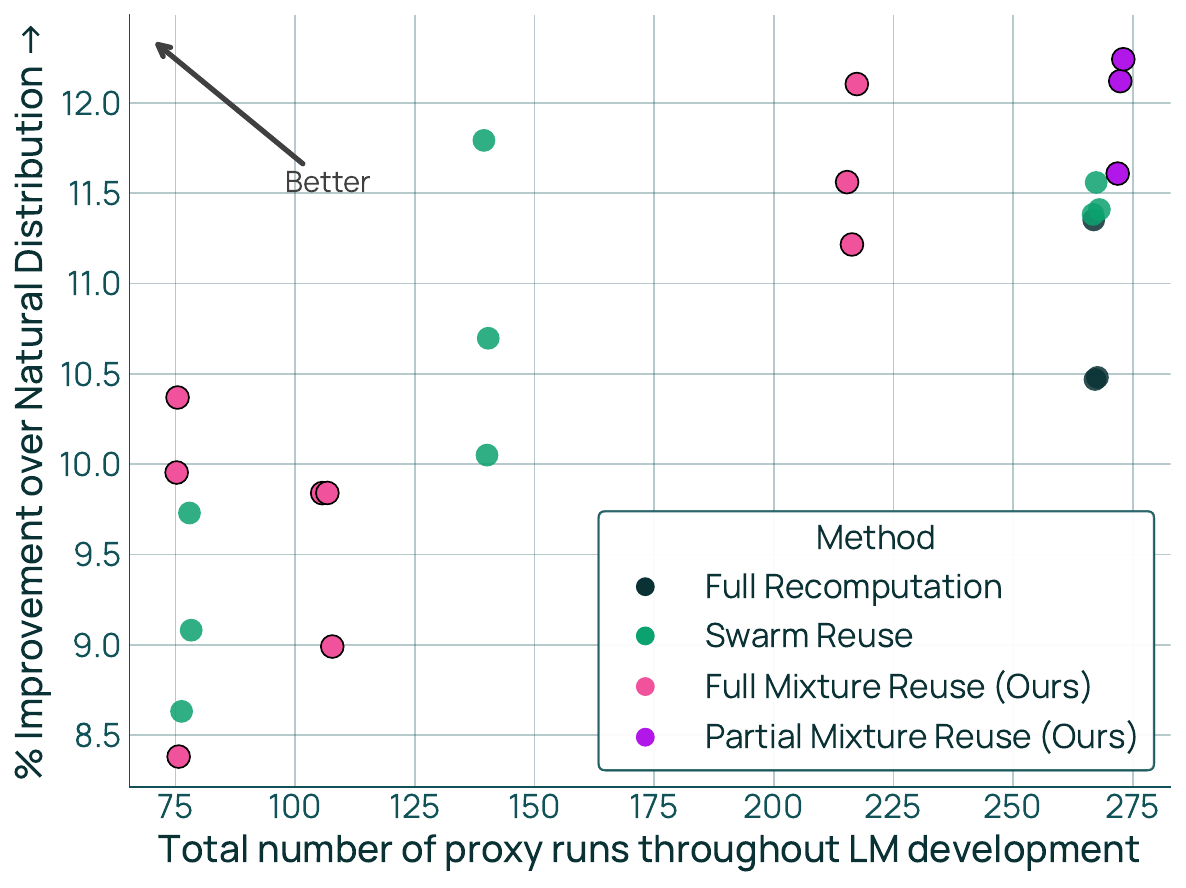}
    \includegraphics[width=0.49\linewidth]{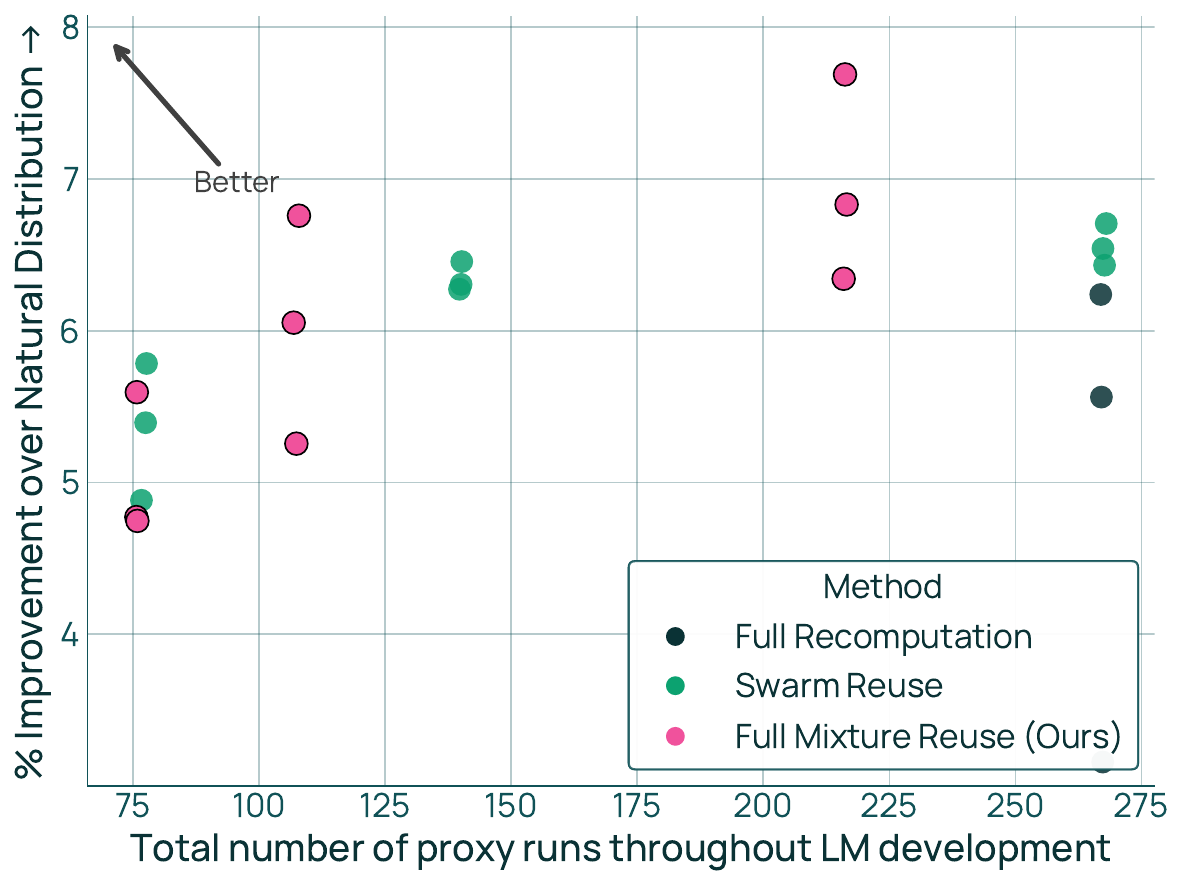}
    \caption{Performance improvement versus cost of mixing under evolving domains for $R=1$T (left) and $R=6$T (right). Total number of proxy runs vary from 76 to 272.}
    \label{fig:low_budget}
\end{figure}

\textbf{Proposed mixes.} Table~\ref{tab:full_domain_mixtures} shows the full mixtures corresponding to the natural distribution, full recomputation (best seed), and \partialmixreuse (best seed).

\begin{table}[t]
\centering
\caption{Mixture weights (full version of Figure~\ref{fig:superswarm_proposed_mixes}) for the natural distribution, full recomputation, and \partialmixreuse.}
\scriptsize
\setlength{\tabcolsep}{4pt}
\begin{tabular}{lccc | lccc}
\toprule
\textbf{Domain} & \textbf{Nat.} & \textbf{Full} & \textbf{Partial} & \textbf{Domain} & \textbf{Nat.} & \textbf{Full} & \textbf{Partial} \\
\midrule
arxiv & 0.003320 & 0.036555 & 0.022549 & pdf:fashion\_beauty & 0.000089 & 0.000003 & 0.000043 \\
dclm:adult\_content & 0.010828 & 0.006309 & 0.004050 & pdf:finance\_business & 0.009771 & 0.004230 & 0.001101 \\
dclm:art\_and\_design & 0.011291 & 0.006265 & 0.008097 & pdf:food\_dining & 0.000371 & 0.000113 & 0.009292 \\
dclm:crime\_and\_law & 0.027186 & 0.009991 & 0.005497 & pdf:games & 0.000397 & 0.000074 & 0.009944 \\
dclm:education\_and\_jobs & 0.029513 & 0.019003 & 0.008874 & pdf:health & 0.017292 & 0.013871 & 0.006251 \\
dclm:electronics\_and\_hardware & 0.012811 & 0.003289 & 0.003360 & pdf:home\_hobbies & 0.000627 & 0.000396 & 0.015698 \\
dclm:entertainment & 0.070592 & 0.055739 & 0.021609 & pdf:industrial & 0.004696 & 0.002512 & 0.000749 \\
dclm:fashion\_and\_beauty & 0.005953 & 0.003134 & 0.000769 & pdf:literature & 0.005095 & 0.005713 & 0.007648 \\
dclm:finance\_and\_business & 0.049587 & 0.022324 & 0.009896 & pdf:politics & 0.006269 & 0.001782 & 0.038680 \\
dclm:food\_and\_dining & 0.016928 & 0.015615 & 0.003785 & pdf:religion & 0.003952 & 0.001056 & 0.000088 \\
dclm:games & 0.036752 & 0.018811 & 0.013172 & pdf:science\_tech & 0.067792 & 0.029994 & 0.014598 \\
dclm:health & 0.062879 & 0.071471 & 0.050413 & pdf:software & 0.001462 & 0.001752 & 0.000000 \\
dclm:history\_and\_geography & 0.025735 & 0.028328 & 0.008994 & pdf:software\_dev & 0.006686 & 0.001929 & 0.015815 \\
dclm:home\_and\_hobbies & 0.020280 & 0.017861 & 0.005614 & pdf:sports\_fitness & 0.000871 & 0.000105 & 0.000060 \\
dclm:industrial & 0.006963 & 0.004541 & 0.007756 & pdf:transportation & 0.002740 & 0.001798 & 0.068597 \\
dclm:literature & 0.058299 & 0.048124 & 0.026746 & pdf:travel & 0.000336 & 0.000043 & 0.007881 \\
dclm:politics & 0.097667 & 0.034541 & 0.011652 & stack-edu:C & 0.000757 & 0.004088 & 0.001655 \\
dclm:religion & 0.044387 & 0.018381 & 0.006006 & stack-edu:CSharp & 0.001151 & 0.014910 & 0.015118 \\
dclm:science\_math\_and\_technology & 0.068254 & 0.077043 & 0.088744 & stack-edu:Cpp & 0.002002 & 0.012993 & 0.036761 \\
dclm:social\_life & 0.034952 & 0.018543 & 0.011095 & stack-edu:Go & 0.000223 & 0.003558 & 0.004103 \\
dclm:software & 0.017264 & 0.003887 & 0.002666 & stack-edu:Java & 0.005009 & 0.015251 & 0.025259 \\
dclm:software\_development & 0.035696 & 0.018979 & 0.025232 & stack-edu:JavaScript & 0.001420 & 0.004582 & 0.026072 \\
dclm:sports\_and\_fitness & 0.031441 & 0.017050 & 0.006516 & stack-edu:Markdown & 0.004621 & 0.028525 & 0.084832 \\
dclm:transportation & 0.014508 & 0.006398 & 0.003515 & stack-edu:PHP & 0.001182 & 0.003491 & 0.007554 \\
dclm:travel\_and\_tourism & 0.009211 & 0.005433 & 0.001804 & stack-edu:Python & 0.002879 & 0.072071 & 0.052859 \\
finemath-3plus & 0.005442 & 0.136232 & 0.136232 & stack-edu:Ruby & 0.000222 & 0.000868 & 0.004068 \\
pes2o & 0.009356 & 0.012147 & 0.006478 & stack-edu:Rust & 0.000227 & 0.000008 & 0.004161 \\
pdf:adult & 0.000048 & 0.000907 & 0.000016 & stack-edu:SQL & 0.001129 & 0.001833 & 0.002143 \\
pdf:art\_design & 0.001092 & 0.000231 & 0.004289 & stack-edu:Shell & 0.000406 & 0.008734 & 0.007459 \\
pdf:crime\_law & 0.006797 & 0.002672 & 0.000038 & stack-edu:Swift & 0.000241 & 0.006040 & 0.004430 \\
pdf:education\_jobs & 0.022072 & 0.010150 & 0.005608 & stack-edu:TypeScript & 0.000399 & 0.009983 & 0.001015 \\
pdf:entertainment & 0.000970 & 0.000087 & 0.003318 & wikipedia & 0.001609 & 0.017654 & 0.011671 \\
\bottomrule
\end{tabular}
\label{tab:full_domain_mixtures}
\end{table}

\textbf{Ablations for each domain update operator.} We demonstrate that \fullmixreuse attains performance similar to full recomputation for each type of domain update operator. We evaluate the following domain updates:
\begin{itemize}
    \item \Add: we start with 24 DCLM topic-based domains. Then, we add 15 Stack-Edu programming language domains.
    \item \Remove: we start with 39 domains consisting of 24 DCLM topics and 15 Stack-Edu programming languages. Then, we remove the 15 Stack-Edu domains.
    \item \Partition: we start with 24 DCLM topics and a single Stack-Edu source domain. Then, we partition the Stack-Edu source into 15 programming language domains. 
    \item \Revise: we start with 24 DCLM topics and a single StackV2~\citep{lozhkov2024starcoder} source domain. Then, the StackV2 domain is replaced with the Stack-Edu domain.
\end{itemize}

For each operator, we compare \fullmixreuse and the natural distribution over $\D'$ (as a control) against full recomputation. We measure the total variation distance of the proposed mixtures with respect to full recomputation's mix, defined as $\tv(q, q^\star) = \frac{1}{2} \sum_{i=1}^{m'} |q_i - q_i^\star| \in [0, 1]$. We also measure performance gap relative to full recomputation's downstream performance (average task BPB) for \Add, \Remove, and \Partition. For our setup, we used $k=4$, $R=6$T.

Figure~\ref{fig:operators} shows that \fullmixreuse achieves substantially smaller performance gaps and TV distances compared to the natural distribution across \Add, \Remove, and \Partition. For \Revise, the TV distance between \fullmixreuse and full recomputation was only 0.21\%---orders of magnitude smaller than other operators---indicating nearly identical mixes. We therefore did not evaluate downstream performance for \Revise, as variation would be dominated by noise. These results demonstrate that mixture reuse works effectively across all domain update operators.

\begin{figure}
    \centering
    \includegraphics[width=\linewidth]{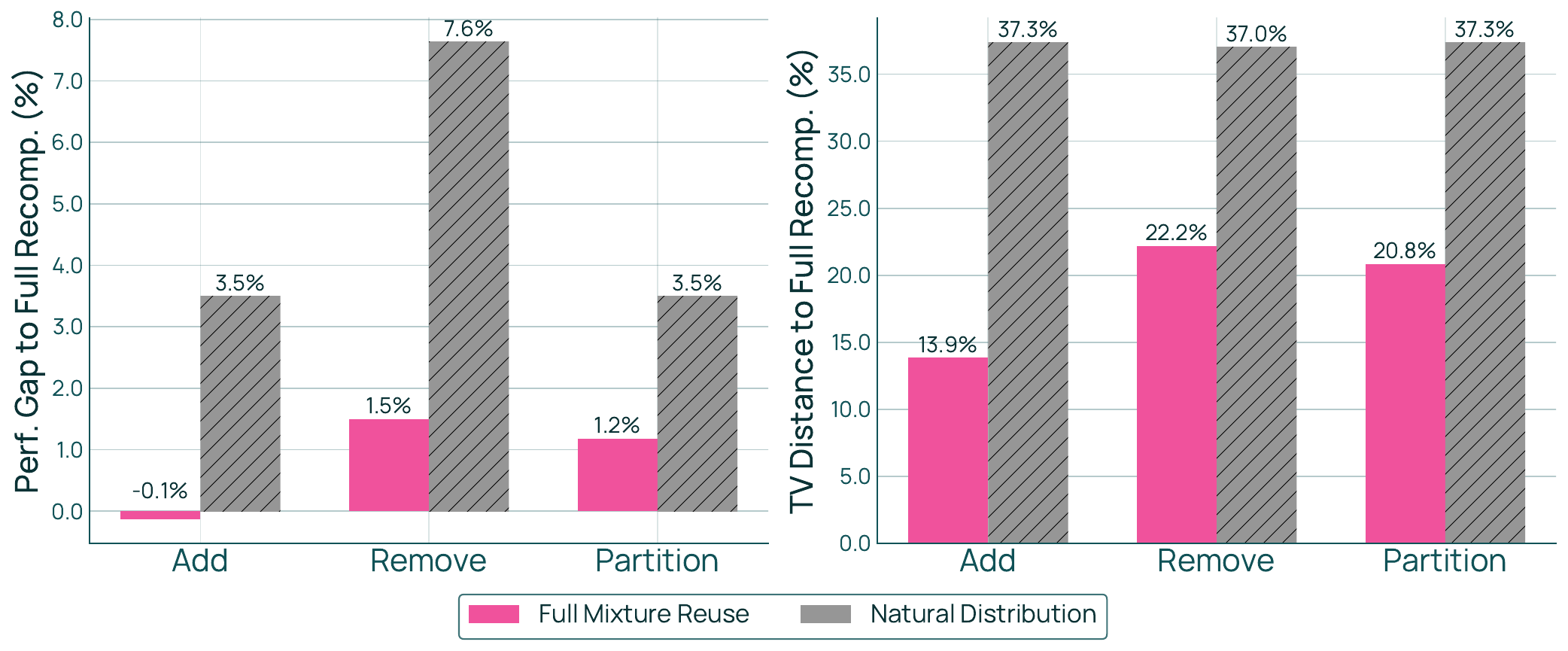}
    \caption{Relative performance gap (left) and total variation distance (right) w.r.t. full recomputation across domain update operators. \fullmixreuse is substantially closer to full recomputation than the natural distribution across all three operators, in both performance and mixes.}
    \label{fig:operators}
\end{figure}

\end{document}